\newif\ifarxiv
\arxivtrue

\newif\ificlr
\iclrfalse

\documentclass{article} 

\ificlr
\usepackage{iclr2027_conference,times}
\else
\ifarxiv
\usepackage[margin=2.5cm]{geometry}
\usepackage{palatino}
\usepackage{natbib}
\usepackage{parskip}
\fi\fi

\usepackage{preamble}

\title{Two-Timescale Fine-tuning Provably Learns New Features for Two-Layer ReLU Networks}

\author{Etienne Boursier \\
INRIA, LMO, Université Paris-Saclay, \\
Orsay, France \\
\texttt{etienne.boursier@inria.fr}
\ificlr\And\else\ifarxiv\and\fi\fi
Nicolas Flammarion \\
EPFL, Lausanne, Switzerland\\
\texttt{nicolas.flammarion@epfl.ch}
}

\ifarxiv\date{}\fi

\begin{document}

\setcounter{tocdepth}{3}

\doparttoc 
\faketableofcontents 

\maketitle

\begin{abstract}
Fine-tuning pre-trained models on specialized tasks with scarce data is central to modern deep learning. Despite its empirical success, theoretical understanding of fine-tuning remains limited. 
We introduce a Gaussian multi-index setting to study fine-tuning from pre-trained weights, where the teacher network has $m+1$ features, $m$ of which are learned during pre-training and one of which must be learned during fine-tuning. 
For two-layer ReLU networks, we show that two-timescale training, i.e., updating the outer weights infinitely faster than the hidden ones, learns the new task-specific feature while preserving the pre-trained ones in the model representation. Moreover, only $\mathcal{O}(d)$ fine-tuning samples are required for this recovery, independently of the number of pre-trained features. In contrast, with random initialization, the same number of samples is insufficient to recover the target parameters. Our results therefore demonstrate that pre-training can induce an implicit bias with a clear statistical advantage over random initialization, enabling feature learning from scarce fine-tuning data.
\end{abstract}

\section{Introduction}

With the recent rise of foundation models, deep learning architectures are increasingly trained in two stages: they are first pre-trained on large and diverse datasets, and then fine-tuned on more specialized and typically much smaller datasets. This pre-training/fine-tuning paradigm has become central to modern machine learning and consistently outperforms \textit{single-task} learning, where a model is trained from scratch directly on the target task \citep{kornblith2019better}.

Despite this empirical success, our theoretical understanding of fine-tuning remains limited. A large body of work has studied the training of neural networks in the \textit{single-task} setting, often through simplified architectures such as two-layer ReLU neural networks \citep{chizat2018global,lyu2020Gradient,boursier2022gradient}.
These analyses, however, typically consider training from an uninformative random initialization. Fine-tuning is fundamentally different: optimization starts from a highly structured initialization that already encodes information about related tasks. This raises a basic question: 
\textit{
Does pre-training merely make optimization easier, helping fine-tuning reach a solution that could also be learned from the fine-tuning data alone? Or does the pre-trained initialization provide side information that changes what can be learned from these limited data?}
We show that the latter can indeed occur. Pre-training can induce a favorable implicit bias during fine-tuning, leading to better generalization even when the target task requires learning features that are absent from the pre-trained representation. 

To understand this advantage, we need to explain how fine-tuning can learn the missing features from limited target data without relearning the entire representation. Existing theoretical analyses have largely focused on the preservation and reuse of pre-trained
features \citep{shachaf2021theoretical,kumar2022finetuning,malladi2023kernel}. Yet empirical evidence shows that fine-tuning can also modify the representation and learn new task-specific features \citep{peters2019tune}. Our work addresses this tension by showing how fine-tuning can acquire a new relevant feature while preserving the useful features already learned during pre-training. In particular, we show that such a feature can be learned from a pre-trained initialization with limited target data, even when the same feature fails to emerge from random initialization.

\paragraph{Contributions.} 

We introduce a Gaussian multi-index setting designed to capture the key aspects of fine-tuning while remaining amenable to a precise analysis. 
We show that, under two-timescale training---where the outer weights are adapted much faster than the hidden weights---and in the limit where the \textit{void} features inherited from pre-training have vanishing scale, the training dynamics converge to a limiting process in which, during an initial phase, only the void features are effectively updated. The pre-trained features therefore remain intact while new features are learned.

We then show that these limiting dynamics can recover a new neuron, corresponding to a feature unseen during pre-training, from only $\cO(d)$ fine-tuning samples. Importantly, this sample complexity is independent of the number of relevant features already learned during pre-training. In contrast, learning the same model from scratch requires a sample complexity that grows with the number of features, highlighting a statistical advantage of fine-tuning over single-task learning.

Our analysis combines a detailed characterization of the optimization dynamics (Section~\ref{sec:population}) with sharp statistical concentration arguments (Section~\ref{sec:main}), yielding optimal statistical guarantees for the considered fine-tuning procedure.

\subsection{Related work}

\paragraph{Theory of fine-tuning.} The theoretical literature on fine-tuning from pre-trained weights remains limited, especially compared with the vast literature on single-task learning. Several works \citep{shachaf2021theoretical,kumar2022finetuning,wu2022power,lee2023surgical,lauditi2026transfer} study fine-tuning in linear networks and identify settings in which pre-trained features can be successfully or unsuccessfully adapted to a downstream task. 
In a related line of work, \citet{shachaf2021theoretical,malladi2023kernel,tomihari2024understanding} study fine-tuning from an NTK perspective. These analyses, however, do not capture the learning of new, task-specific features during fine-tuning, and instead focus primarily on adapting the pre-trained representation to the downstream task. With the exception of \citet{shachaf2021theoretical}, who consider the simultaneous fine-tuning of multiple layers under restrictive assumptions on both the pre-training and fine-tuning tasks, these approaches essentially reduce fine-tuning to a linear adaptation around the pre-trained weights. They therefore cannot capture the rich feature-learning dynamics that may arise during fine-tuning. 
We instead study fine-tuning of a two-layer ReLU network in the feature-learning regime. Our downstream task requires the pre-trained model not only to reuse features acquired during pre-training, but also to learn a genuinely new, task-specific feature. This setting allows us to show that fine-tuning from pre-trained weights can learn a feature that would not be learned from random initialization.

More recently, \citet{lippl2024inductive,anguita2026a} have studied how the scale of the pre-trained initialization affects the optimization dynamics of fine-tuning and their implicit bias. These works focus primarily on optimization and do not provide a statistical analysis translating the resulting dynamics into improved generalization guarantees on the fine-tuning task. 
\citet{jones-mccormick2025provable} establish a statistical advantage of fine-tuning over single-task learning in a single-index model. In their setting, however, the feature relevant to the fine-tuning task can already be recovered from the unsupervised pre-training task via PCA.  
By contrast, our fine-tuning procedure must learn a feature that is  absent from the pre-trained representation, while  preserving the useful features acquired during pre-training.

\paragraph{Parameter efficient fine-tuning.} Given the enormous number of parameters in foundation models, parameter-efficient fine-tuning methods such as LoRA \citep{hu2022lora} adapt only a small subset of the model parameters. We view these methods as a distinct class of algorithms, with different expressivity and applications from methods that update all model weights. They have also attracted increasing theoretical interest \citep[see e.g.,][]{jang2024lora,dayi2024gradient,zhang2025lora,kim2025lora}. This line of work is complementary to ours, although our setting is inspired by \citet{dayi2024gradient} and our fine-tuning task could be solved by LoRA, as explained in Section~\ref{sec:twotimescale}.

\paragraph{Learning ReLU neurons.} Our theoretical setting can be viewed as an extension of the widely studied problem of learning a single ReLU neuron \citep{soltanolkotabi2017learning,yehudai2020learning,xu2023over}. While this literature considers learning a single neuron from scratch, we study the learning of a new ReLU neuron in the presence of $m$ neurons already learned during pre-training. 
Extending our setting to multiple new features would be a natural and important direction, but is outside our scope. Learning multiple neurons remains largely unsolved in fully general settings \citep{zhong2017recovery,zhang2019learning}. 
Interestingly, this literature typically trains only the hidden layer. Our analysis instead crucially relies on training both layers at different learning rates. This two layer structure yields the limiting dynamics in Equation~\eqref{eq:oneneurondyn}, which are key to learning the new feature while preserving the pre-trained ones.

\paragraph{Two-timescale optimization.} Two-timescale optimization, in which different layers are trained with different learning rates, has been studied in several single-task settings \citep{marion2023leveraging,berthier2025learning,bietti2025learning,barboni2025ultra},  through the limiting dynamics obtained when the outer weights are trained infinitely faster than the inner ones. In these works, the two-timescale regime is primarily motivated by analytical tractability. Differential learning rates are also widely used in practice for fine-tuning, as they help preserve pre-trained features and mitigate feature distortion \citep{howard2018universal}. Thus, the two-timescale regime in our work is motivated both by its relevance to practical fine-tuning and by the tractable training dynamics it provides.

\subsection{Outline.} 

We introduce our fine-tuning setting and derive the limiting dynamics of interest in Section~\ref{sec:setting}. Section~\ref{sec:twotimescale} presents our main results, showing these dynamics recover the optimal weights both at the population level and from $n\gtrsim d$ samples. Section~\ref{sec:discussion} discusses the results and their relation to the existing literature. Finally, Section~\ref{sec:expe} illustrates our theoretical findings through a toy numerical experiment.

\section{Setting and limit dynamics}\label{sec:setting}

We consider a teacher-student setting in which data $(x_i, y_i)_{i\in[n]}$ are generated i.i.d. as follows:
\begin{equation*}
 \textstyle x_i \sim \cN(0,\id)\qquad\text{and}\qquad
y_i =\sum_{i=1}^{m+1} \sigma(x^\top w_i^\star),
\end{equation*}
where $\sigma(\cdot)=\max(0,\cdot)$ is the ReLU activation. The labels are generated by a two-layer teacher network of width $m+1$, whose output weights are all equal to one. 
On these data, we train a two-layer ReLU network parametrized by $\theta = (a,W)\in\R^{m+1}\times\R^{(m+1)\times d}$ by minimizing the empirical square loss
\begin{equation*}
 \textstyle\cL_n(\theta) \coloneqq  \frac{1}{n}\sum_{i=1}^n\left( f_\theta(x_i) - y_i\right)^2 \qquad \text{where}\qquad f_\theta(x) = \sum_{i=1}^{m+1} a_i \sigma(x^\top w_i).
\end{equation*}
Our goal is to understand fine-tuning from a pre-trained initialization that has already learned $m$ of the $m+1$ teacher features. Specifically, we assume that pre-training has recovered the features $w_1^\star,\ldots,w_m^\star$, whereas the remaining feature
$w_{m+1}^\star$ was absent from the pre-training task and must therefore be learned during fine-tuning. 
Motivated by this picture, we initialize the hidden weights as follows, for a small $\alpha>0$,
\begin{equation}\label{eq:init}
w_i(0) = w_i^\star \text{ for any }i \in [m],\qquad 
w_{m+1}(0) = \alpha \breve{w} \ \text{ with }\ \breve{w}\sim \unif(\bS_{d-1}).
\end{equation}
This initialization thus idealizes finite-time pre-training on data generated only by the first $m$ teacher features; it is discussed further in Section~\ref{sec:discussion}. 

\subsection{Two-timescale dynamics}

Fine-tuning is typically performed by gradient-based optimization of the empirical loss $\cL_n(\theta)$.
For the sake of the analysis, we model this optimization by gradient flow, which describes its continuous-time behavior in the small-step-size limit:
\begin{equation*}
\textstyle\dot\theta(t) \in - \partial \cL_n(\theta(t)),
\end{equation*}
where $\partial$ stands for the Clarke subdifferential, which accounts for the non-differentiability of ReLU.
In the regime studied in this work, updating all layers at the same rate can substantially modify the pre-trained features and thereby degrade their generalization properties (see Section~\ref{sec:discussion}). A natural alternative, considered in several previous works, is to use a larger learning rate for the output layer than for the hidden layer. This allows the output coefficients to adapt rapidly to the downstream task while limiting the displacement of the pre-trained features. 
In the limit where the ratio between the output- and hidden-layer learning rates tends to infinity, the dynamics converge to a two-timescale regime in which the output weights are instantaneously fitted to their optimal values, while the hidden weights evolve on a much slower timescale.
We focus on this theoretical limit, which corresponds to the following training dynamics:
\begin{equation}\label{eq:twotimescale}
\dot W(t) \in -\partial_{W}\cL_n(a_n(W(t)), W(t)) \qquad \text{where} \qquad a_n(W) \coloneqq \argmin_{a\in\R^{m+1}}\cL_n(a,W).
\end{equation}
The $\argmin$ above may not be unique, making the dynamics of Equation~\eqref{eq:twotimescale} ill-defined. However, with our choice of initialization and sufficiently many samples, the $\argmin$ is unique along the whole trajectory, making the dynamics well defined. In particular, the trajectory remains within a set $\cT\subseteq \R^{(m+1)\times d}$ that satisfies Assumption~\ref{ass:conditioning} below (see Theorem~\ref{thm:mainpop} for details).
\begin{assumption}[Conditioning of the empirical feature matrix in $\cT$]\label{ass:conditioning}
The set $\cT$ is compact and for any $W\in\cT$, $\lambda_{\min}\left(\Sigma(W,\bX)\Sigma(W,\bX)^\top \right) > 0$, where $\Sigma(W,\bX)\in\R^{(m+1)\times n}$ is the empirical feature matrix, given by $\Sigma(W,\bX)_{i k}= \sigma(w_i^\top x_k)$ and $\lambda_{\min}$ denotes the smallest eigenvalue function.
\end{assumption}

\subsection{Separated neuron dynamics}

Studying the dynamics of Equation~\eqref{eq:twotimescale} remains extremely challenging in full generality. We therefore exploit the specific structure of the pre-trained initialization in Equation~\eqref{eq:init}. Importantly, the ReLU activation is homogeneous, meaning that $\sigma(\lambda z) = \lambda \sigma(z)$ for any $\lambda\geq 0$. As a first consequence, the norm of each neuron $w_i$ remains constant throughout training.
\begin{lemma}\label{lemma:constantnorm}
Let $W\in\R^{(m+1)\times d}$ follow Equation~\eqref{eq:twotimescale}, then for any $i\in[m+1]$, $\|w_i(t)\| = \|w_i(0)\|$.
\end{lemma}
In consequence, one only has to track the normalized neurons $\bw_i \coloneqq \frac{w_i}{\|w_i\|}$ during the dynamics. Using the homogeneity of the ReLU activation, one can easily derive the dynamics followed by the normalized neurons. If we denote by $\bW$ the matrix whose $i$-th row is given by $\bw_i$ and $W$ follows the differential inclusion of Equation~\eqref{eq:twotimescale}, then its row-normalized version satisfies the following differential inclusion for almost any $t\geq0$:
\begin{equation}\label{eq:twotimescalenormalized}
\textstyle\dot{\bw}_i(t) \in -\frac{1}{\|w_i(0)\|^2}\partial_{\bw_i}\cL_n(a_n(\bW(t)), \bW(t)) \qquad \text{for all }i\in[m+1].
\end{equation}
Notably, the evolution rate of neuron $i$ in Equation~\eqref{eq:twotimescalenormalized} scales as $\|w_i(0)\|^{-2}$. When initialized by Equation~\eqref{eq:init}, as $\alpha$ goes to $0$, the last neuron---that does not correspond to any pre-trained feature---therefore evolves on an infinitely faster timescale than the other neurons. In the limit, only this neuron evolves, while the other neurons, inherited from the pre-training procedure, remain fixed. This limiting dynamics is formalized in Proposition~\ref{prop:limitdyn} below.
\begin{proposition}\label{prop:limitdyn}
Consider $T\in\R_+$, any positive sequence $(\alpha_k)_{k\in\N}$ such that $\alpha_k \overset{k\to\infty}{\longrightarrow}0$, and corresponding solutions $\bW^{\alpha_k}$ of Equation~\eqref{eq:twotimescalenormalized}, where $W^{\alpha_k}(0)$ is initialized by Equation~\eqref{eq:init} with scale~$\alpha_k$.\\
Suppose that for $k$ large enough, the trajectory $\left(\bW^{\alpha_k}(\alpha_k^2 t)\right)_{t\in[0,T]}$ is included within some set $\cT$, independent of $k$, satisfying Assumption~\ref{ass:conditioning}. Then, one can extract a subsequence $\alpha_{\varphi(k)}$ such that, as $k\to \infty$, $(\bW(\alpha_{\varphi(k)}^2 t))_{t\geq 0}$ converges uniformly towards $\bW^\circ(t)$ on  $[0,T]$, where $\bW^\circ$ is a solution of the following differential inclusion for almost any $t\geq0$, with $\bw_{m+1}^\circ(0)=\breve{w}$,
\begin{equation}\label{eq:oneneurondyn}
\begin{cases}
\bw_i^\circ(t) = \frac{w_i^\star}{\|w_i^\star\|} \qquad \text{ for any }i\in[m],\\
\dot{\bw}_{m+1}^\circ(t) \in -\partial_{\bw_{m+1}}\cL_n(a_n(\bW^\circ(t)), \bW^\circ(t)).
\end{cases}
\end{equation}
\end{proposition}
Proposition~\ref{prop:limitdyn} implies that, as $\alpha\to 0$, the trajectories $\bW^\alpha(\alpha^2 \cdot)$ converge uniformly, on any compact of $\R_+$, towards the set of solutions of the limit differential inclusion \eqref{eq:oneneurondyn}. 
Importantly, the convergence is only to a set of trajectories, since Equation~\eqref{eq:oneneurondyn} may admit multiple solutions due to the non-differentiability of the loss. Characterizing these different solutions is the focus of Section~\ref{sec:twotimescale}.

Proposition~\ref{prop:limitdyn} also requires the trajectories to remain in a set $\cT$ satisfying Assumption~\ref{ass:conditioning}. In Section~\ref{sec:twotimescale}, we show that, with high probability, all trajectories of the limiting process remain in such a set (see Theorem~\ref{thm:mainempirical}). %
Lemma~\ref{lemma:conditioning}, given in Appendix~\ref{app:mainproof}, then implies that, for $\alpha$ sufficiently small, the trajectories $\bW^\alpha(\alpha^2 \cdot)$ also remain in a set $\cT$ satisfying Assumption~\ref{ass:conditioning}.

\section{Two-timescale small initialization dynamics}\label{sec:twotimescale}

The two-timescale dynamics in Equation~\eqref{eq:twotimescale} are difficult to analyze directly. In the limit $\alpha\to 0$, however, they reduce to the simpler dynamics in Equation~\eqref{eq:oneneurondyn}, which we study throughout the remainder of the paper. In this limiting regime, the pre-trained features $w_i^\circ$, $i\in[m]$, remain fixed, and only the new feature evolves. The dynamics are nevertheless nontrivial because the corresponding output weights $a_i^\circ$ of the pre-trained features continue to evolve and must be carefully controlled. In particular, as shown in the proof of Lemma~\ref{lemma:populationloss}, these output weights are not initially equal to their optimal value, namely $1$, and reach this value only at convergence. Early in training, they instead compensate for the fact that the $(m+1)$-th feature has not yet been learned. This transient compensation also explains why the \textit{Linear Probe then Fine-Tune} method (see Section~\ref{sec:discussion}) does not achieve zero test loss in our setting, as confirmed experimentally in Section~\ref{sec:expe}.

For the same reason,  the dynamics are more complex than simply freezing the first $m$  neurons---both their inner and output weights---and training only the $(m+1)$-th neuron. Such a procedure would be closer to LoRA fine-tuning and, in our setting, would reduce to learning a single ReLU neuron.
To keep the analysis tractable, we make the following assumption on the teacher features.
\begin{assumption}\label{ass:orthofeatures}
The teacher features $(w_i^\star)_{i\in[m+1]}$ form an orthonormal system of $\R^d$.
\end{assumption}
Assumption~\ref{ass:orthofeatures} simplifies the analysis by yielding a closed-form expression for $a_n(W)$, at least in the population limit studied in Section~\ref{sec:population}. Orthogonality of features is a standard assumption in the literature on learning from multi-neuron teacher networks \citep{safran2018spurious,simsek2023should,dayi2024gradient}. We discuss the dependence of our results on this assumption in Section~\ref{sec:discussion}.

\subsection{Warm up: gradient flow on population loss}\label{sec:population}

Before studying the training dynamics of Equation~\eqref{eq:oneneurondyn}, we first focus on the population loss case, where fine-tuning is done with access to an infinite number of data points. More precisely, we here consider the population loss dynamics given by the following ODE\footnote{In the population case, the loss becomes differentiable.}, for any $t\geq 0$:
\begin{equation}\label{eq:ODEpop}
w_i(t) = w_i^\star \text{ for any }i\in[m] ,\qquad \text{and}\qquad
\dot{w}_{m+1}(t) \in -\nabla_{w_{m+1}}\cL(a(W(t)), W(t));
\end{equation}
where $w_{m+1}(0)=\breve{w}$, $a(W) \coloneqq\argmin_{a\in \R^{m+1}} \cL(a, W)$ and
\begin{equation*}
\cL(a,W)  \coloneqq\textstyle \bE_{x\sim\cN(0,\id)}\left[\left( \sum_{i=1}^{m+1} a_i(t) \sigma(x^\top w_i(t)) - \sum_{i=1}^{m+1} \sigma(x^\top w_i^\star)\right)^2\right].
\end{equation*}
Although this simpler setting avoids any statistical consideration, it already provides a good grasp of what is happening from an optimization point of view. 
We show in this section that $w_{m+1}$ converges to $w_{m+1}^\star$ in the population loss dynamics. For that, we merely need to track the angular deviation between these two vectors.  
In the population case, explicit computations can be done, relying on the arc-cosine kernel formula \citep{cho2009kernel}, for any $u,v\in\bS_{d-1}$:
\begin{equation*}\label{eq:arccosine}
\bE_{x\sim\cN(0,\id)}\left[ \sigma(x^\top u) \sigma(x^\top v)\right] = \frac{1}{2\pi}\left(\sqrt{1-(u^\top v)^2} + (\pi-\arccos(u^\top v))u^\top v\right).
\end{equation*}
In particular, for some functions $H$ and $D$ depending on this formula and defined precisely by Equation~\eqref{eq:HDpop} in Appendix~\ref{app:popdynamics}, one can express precisely the loss and its gradient.

\begin{restatable}{lemma}{poplosslemma}
\label{lemma:populationloss}
Let Assumption~\ref{ass:orthofeatures} hold and $w_{m+1}\in\bS_{d-1}$ such that $D(w_{m+1})>0$. Then, noting $W=[w_1^{\star\,\top}, \ldots, w_m^{\star\,\top}, w_{m+1}^{\top}]$, $a(W)$ is uniquely defined and satisfies
$a(W)_{m+1} = \frac{H(w_{m+1})}{D(w_{m+1})}$.\\
Moreover, the loss and its subdifferential satisfy
\begin{gather*}
\cL(a(W),W) = D(w_{m+1}^\star)-\frac{H(w_{m+1})^2}{D(w_{m+1})},\\
\nabla_{w_{m+1}}\cL(a(W),W) = -\frac{H(w_{m+1})}{D(w_{m+1})}  \left(\id - w_{m+1}w_{m+1}^\top \right)\Psi(w_{m+1}),
\end{gather*}
where $\Psi(w) \coloneqq 2\nabla H(w) - \frac{H(w)}{D(w)}\nabla D(w)$.
\end{restatable}
Lemma~\ref{lemma:populationloss} then provides an intricate closed-form expression of the training dynamics, in terms of the function $H$ and $D$. From there, a careful analysis can be done to show convergence of the learned parameter towards the optimal choice, as given by Theorem~\ref{thm:mainpop} below. 
\begin{theorem}\label{thm:mainpop}
Let Assumption~\ref{ass:orthofeatures} hold and initialize as $w_{m+1}(0)\sim \unif(\bS_{d-1})$. There exist positive universal constants $m_0$, $c$ and $C$ such that for any $\delta\in(0,1/2)$, if $m\geq m_0$ and $d\geq Cm^2\ln(1/\delta)$, then with probability at least $1-\delta$ over the initialization, the solution of Equation~\eqref{eq:ODEpop} satisfies:
\begin{equation*}
w_{m+1}(t)^\top w_{m+1}^\star \geq 1- e^{-c(t-Cm)} \qquad \text{for any }t\geq Cm.
\end{equation*}
\end{theorem}
\begin{proof}[Sketch of proof.]
The key ingredient is Lemma~\ref{lemma:poprate} in Appendix~\ref{app:keypop}. Let $P_m w$ denote the orthogonal projection of $w$ onto $\Span (w_1^\star,\ldots,w_m^\star)$. The lemma relies on the orthogonal features assumption and shows that, as long as $\|P_m w_{m+1}(t)\|$ remains below some constant threshold,
\begin{enumerate}[topsep=0pt, itemsep=-2pt]
\item $D(w_{m+1}(t))\gtrsim 1$;
\item $\frac{\df }{\df t}(w_{m+1}(t)^\top w_{m+1}^\star ) \gtrsim \frac{H(w_{m+1}(t))}{D(w_{m+1}(t))}\left(1-w_{m+1}(t)^\top w_{m+1}^\star\right)$;
\item $\frac{\df }{\df t}\|P_m w_{m+1}(t)\|\lesssim \frac{H(w_{m+1}(t))}{D(w_{m+1}(t))}\left(\frac{1}{\sqrt{m}}+\|P_m w_{m+1}(t)\|\right)$.
\end{enumerate}
These inequalities reveal the key mechanism behind the result. As long as
$\frac{H(w_{m+1}(t))}{D(w_{m+1}(t))}>0$, 
the component of $w_{m+1}(t)$ along $w_{m+1}^\star$ grows quickly towards $1$, while its component in the span of the previously learned directions ${w_1^\star,\ldots,w_m^\star}$ evolves at a much slower rate, of order $\frac{1}{\sqrt{m}}+\|P_m w_{m+1}(t)\|$. 
In particular, when $m$ is sufficiently large and  $\|P_m w_{m+1}(0)\|$ is sufficiently small, the undesirable component $P_m w_{m+1}(t)$ grows slowly enough to remain small throughout the trajectory. This ensures that the conditions of Lemma~\ref{lemma:poprate} continue to hold.

It remains to verify that the ratio $H/D$ is indeed positive along the trajectory. Lemma~\ref{lemma:poprate} also implies that, provided
\begin{equation}\label{eq:initcondition}
\textstyle\|P_m w_{m+1}(0)\|\lesssim 1/\sqrt{m}
\qquad\text{and}\qquad
w_{m+1}(0)^\top w_{m+1}^\star\gtrsim -1/m,
\end{equation}
we have
$\frac{H(w_{m+1}(0))^2}{D(w_{m+1}(0))}
\gtrsim \frac{1}{m^2}$. 
When $d\gtrsim m^2$, Equation~\eqref{eq:initcondition} holds with high probability. 
Moreover, by Lemma~\ref{lemma:populationloss}, the quantity $\frac{H(w_{m+1})^2}{D(w_{m+1})}$ is, up to an additive constant, the opposite of the population loss. Since the population loss is non-increasing along the gradient flow, it follows that $\frac{H(w_{m+1}(t))^2}{D(w_{m+1}(t))}$ remains bounded away from zero for all $t\geq 0$. 
Together with item~1 above, this implies that $\frac{H(w_{m+1}(t))}{D(w_{m+1}(t))}$  remains positive (by continuity) and bounded away from zero throughout the trajectory.

The preceding argument therefore applies for all $t\geq 0$: the component along the new direction $w_{m+1}^\star$ grows substantially faster than the component along the previously learned directions. A more precise analysis of the ratio $H/D$ then yields the exact convergence rate stated in Theorem~\ref{thm:mainpop}.
\end{proof}

Theorem~\ref{thm:mainpop} shows that, when trained on the full population distribution, the learned feature $w_{m+1}(t)$ eventually converges towards the optimal parameter $w_{m+1}^\star$ at a linear rate. 
This convergence can nevertheless be substantially delayed by the initialization phase. Indeed, the convergence rate contains a time shift of order $Cm$. This reflects the slow dynamics at initialization: the update of $w_{m+1}(t)$ is initially only of order $1/m$ (due to the ratio $\frac{H(w_{m+1})}{D(w_{m+1})}$ at initialization), and the feature therefore requires a time of order $m$ to escape this slow-evolution regime. Once this initial phase is overcome, the dynamics enter the linear convergence regime described by Theorem~\ref{thm:mainpop}.

A notable difference with classical analyses of learning a single ReLU neuron is that, in our setting, convergence requires a suitable initialization. This is precisely what leads to the dimensional requirement $d\gtrsim m^2$---see Section~\ref{sec:discussion} for further discussion. 
The reason is that the two-timescale dynamics of Equation~\eqref{eq:ODEpop} are not globally attracted to $w_{m+1}^\star$. Instead, the dynamics possess spurious local minima, distinct from the desired solution $w_{m+1}^\star$. Their existence is implied by Lemma~\ref{lemma:populationloss}, but their precise location is not characterized by our analysis. 
The initialization condition thus ensures that $w_{m+1}(0)$ lies outside the attraction basins of these spurious minima. Once this condition is satisfied, the dynamics are driven towards the desired solution~$w_{m+1}^\star$.

\subsection{Gradient flow on empirical loss}\label{sec:main}
Let us now go back to the empirical loss case, through the study of the (possibly multiple) solutions $\bW^\circ$ of the differential inclusion given by Equation~\eqref{eq:oneneurondyn}. 
In the empirical case, we can still derive closed form expressions of the loss and its subdifferential, replacing the key quantities $H$ and $D$, by their empirical counterpart defined in Equation~\eqref{eq:HDn}. Given a large enough number of training samples, the different quantities of interest concentrate sufficiently close to their population counterpart, so that we can derive a similar analysis of the training dynamics. It then yields our main theorem. 
\begin{theorem}\label{thm:mainempirical}
Let Assumption~\ref{ass:orthofeatures} hold and initialize as $\bw_{m+1}^\circ(0)\sim \unif(\bS_{d-1})$. There exist positive universal constants $m_0$, $c$ and $C$ such that for any $\delta\in(0,1/2)$, if $m\geq m_0$, $d\geq Cm^2\ln(1/\delta)$ and $n\geq C\left(d+\ln(1/\delta)\right)$ then with probability at least $1-\delta$ over both the data and initialization, any solution of Equation~\eqref{eq:oneneurondyn} satisfies:
\begin{equation*}
\bw_{m+1}^\circ(t)^\top w_{m+1}^\star \geq 1- e^{-c(t-Cm)} \qquad \text{for any }t\geq Cm.
\end{equation*}
Moreover under the same random event, there exists a set $\cT$ satisfying Assumption~\ref{ass:conditioning} such that $\bW^\circ(t)\in\cT$ for any $t\geq 0$.
\end{theorem}
\begin{proof}[Sketch of proof.]
From an optimization perspective, the proof follows the same lines as the one of Theorem~\ref{thm:mainpop}. The main additional challenge is to establish sufficiently tight concentration bounds for the relevant empirical quantities around their population counterparts. The main concentration result is provided by Proposition~\ref{prop:empiricalrates} in Appendix~\ref{app:concetration}. Similarly to the population case, it shows that, under the stated sample complexity, with high probability and as long as $\|P_m \bw^\circ_{m+1}(t)\|$ is small enough,
\begin{enumerate}[topsep=0pt, itemsep=-2pt]
\item $D_n(\bw^\circ_{m+1}(t))\gtrsim 1$;
\item $\frac{\df }{\df t}(\bw^\circ_{m+1}(t)^\top w_{m+1}^\star) \gtrsim \frac{H_n(\bw^\circ_{m+1}(t))}{D_n(\bw^\circ_{m+1}(t))}\left(1-\bw^\circ_{m+1}(t)^\top w_{m+1}^\star\right)$;
\item $\frac{\df }{\df t}\|P_m \bw^\circ_{m+1}(t)\|\lesssim \frac{H_n(\bw^\circ_{m+1}(t))}{D_n(\bw^\circ_{m+1}(t))}\left(\frac{1}{\sqrt{m}}+\|P_m \bw^\circ_{m+1}(t)\|+\varepsilon\right)$;
\end{enumerate}
for a small $\varepsilon>0$. 
The proof of these three statements relies on two complementary concentration arguments. First, we establish uniform concentration on the sphere $\bS_{d-1}$ for the different empirical quantities involved. These bounds directly yield points~1 and~3, and also point~2 whenever $\bw^\circ_{m+1}$ lies outside a neighborhood of $w_{m+1}^\star$ (see Appendix~\ref{app:uniform_concentration}). The required uniform concentration bounds are obtained from empirical process concentration arguments, as detailed in Appendix~\ref{app:generalconcentration}, i.e., with tight chaining and VC dimension techniques \citep[we refer to][for a useful introduction to those techniques]{vaart1996weak}.

The second argument is needed to control the dynamics when $\bw^\circ_{m+1}$ enters a neighborhood of $w_{m+1}^\star$. Borrowing techniques from \citet{soltanolkotabi2017learning}, we obtain the required lower bound in point~2 in this regime as well (see Appendix~\ref{app:neighborconc}). 
Once these properties are established with high probability, the remainder of the proof follows the same lines as the one of Theorem~\ref{thm:mainpop}, with the additional ingredient of a sharper concentration bound at initialization, detailed in Appendix~\ref{app:concinit}.
\end{proof}
Interestingly, Theorem~\ref{thm:mainempirical} shows that perfectly recovering the optimal parameter $w_{m+1}^\star$ and achieving zero test loss only requires a sample complexity of $n\gtrsim d$, independent of $m$. This is particularly striking when compared with learning all $m+1$ features from scratch: from an information-theoretic perspective, this would require at least $\Omega(md)$ samples. 
In contrast, the optimization considered here effectively leverages the information already encoded in the pre-trained weights. As a result, it only needs to learn the single new feature $w_{m+1}^\star$, leading to a sample complexity of the same order as that required for learning a single neuron.

\section{Discussion and limitations}\label{sec:discussion}

In this section, we discuss the limitations of our results, and also compare the two-timescale optimization scheme considered here to other standard methods.
%
%
%
\paragraph{Comparison with other fine-tuning methods.}
The most natural fine-tuning strategy is \textit{full fine-tuning}, which adapts all model parameters using gradient-based methods with the same learning rate across layers. Despite its simplicity and strong empirical performance, full fine-tuning can substantially \textit{distort} pre-trained features \citep{kumar2022finetuning}, thereby losing the statistical benefits of pre-training and harming generalization. We observe the same phenomenon in our toy setting (see Section~\ref{sec:expe}). 
At the opposite extreme, \textit{linear probing} freezes the inner weights and adapts only the last layer. This prevents feature distortion and can be particularly effective with limited fine-tuning data, but at the cost of expressivity: it cannot learn tasks requiring new or substantially modified features.

Several intermediate strategies aim to preserve pre-trained features while retaining this expressivity. \textit{Two-timescale fine-tuning}, which we study here, assigns smaller learning rates to inner layers than to outer layers \citep{howard2018universal}. We show that this induces a favorable implicit regularization: new features can be learned while the pre-trained representation remains essentially unchanged. The mechanism is particularly striking in the limiting regime where the \textit{void} features vanish at initialization: their small magnitude is compensated by large outer weights, amplifying their gradients and causing them to evolve much faster than the pre-trained features.

A related strategy is \textit{Linear Probe then Finetune} (LP-FT), which first performs linear probing and then switches to full fine-tuning \citep{kumar2022finetuning}. While effective in practice, LP-FT only delays feature distortion rather than fully preventing it. Indeed, even with infinitely many data points, linear probing does not initially assign the optimal output weights to the pre-trained features: in our setting, the weights satisfy $a_i(0)\neq 1$ for $i\in[m]$ (see proof of Lemma~\ref{lemma:populationloss}). To learn $w_{m+1}^\star$, the output weights must therefore be readjusted. However, because LP-FT does not preserve the two-timescale structure during this second phase, this readjustment is accompanied by a distortion of the pre-trained features. 
In our experiments, this distortion is comparable to that of full fine-tuning and is sufficient to deteriorate test performance (see Section~\ref{sec:expe}). 
Finally, \textit{$\ell_2$ distance to starting point} (L2-SP) explicitly penalizes deviations from the pre-trained weights \citep{xuhong2018explicit}. This regularization directly preserves the pre-trained representation and is particularly well suited to our setting, as illustrated empirically in Section~\ref{sec:expe} and discussed in detail in Appendix~\ref{app:expe_details}. However, in more general settings, this regularization can also hinder the adaptation of the pre-trained features, for instance when the new task requires small but non-negligible modifications of the pre-trained weights.

\paragraph{Requirement $d\gtrsim m^2$.} 
As explained after Theorem~\ref{thm:mainpop}, the condition $d\gtrsim m^2$ is primarily needed to ensure that, with high probability, the $(m+1)$-th neuron is initialized within the attraction basin of the target weight $w_{m+1}^\star$. In lower dimensions, such a guarantee cannot hold with high probability for a single randomly initialized neuron. Overparameterization could potentially circumvent this limitation: by training multiple new neurons, one could hope that at least one is initialized in the correct attraction basin, thereby relaxing the dimensional requirement. Although similar dynamics appear to persist in the overparameterized setting (see Appendix~\ref{app:expe_overparam}), making this argument rigorous would require a substantially more intricate analysis, as the dynamics become considerably more challenging to characterize. This difficulty is already apparent in the single-task setting for learning a single ReLU \citep{xu2023over}.

\paragraph{Pre-trained initialization of the weights.}  A key assumption of our work concerns the structure of the initialization: we assume that the first $m$ neurons are initialized at the optimal weights $w_1^\star,\ldots,w_m^\star$, while the $(m+1)$-th neuron is initialized with a small norm and a uniformly random direction. This setting is inspired by pre-training an overparameterized model on a teacher network with features $w_1^\star,\ldots,w_m^\star$ \citep[see, e.g.,][]{bietti2025learning}. If pre-training succeeds, one would expect the learned network to contain features closely aligned with these optimal weights, while overparameterization may leave additional \textit{void} features that are not aligned with any of them. Assuming that such void features have small norm is also natural, as loss minimization tends to suppress unused features, and this effect is further reinforced by weight decay \citep{han2015learning}. 

We thus view Equation~\eqref{eq:init} as an idealized model of weights obtained after pre-training, chosen to make the resulting dynamics amenable to analysis. In Appendix~\ref{app:expe_real}, we complement this idealized setting with experiments initialized from an actually pre-trained model. The resulting dynamics are consistent with our theoretical findings, supporting the relevance of our idealized initialization.

\paragraph{Orthogonality of features.} 
Assumption~\ref{ass:orthofeatures} is primarily made for analytical convenience. In particular, orthogonality provides a simple and explicit expression for the population Gram matrix $G$, whose $(i,j)$-th entry is
$G_{i,j} = \bE_x\big[\sigma(x^\top w_i^\star)\sigma(x^\top w_j^\star)\big]$. 
This assumption can readily be relaxed to allow features with different norms and nonzero correlations. However, the analysis still requires an analogue of Item~2 in Lemma~\ref{lemma:poprate}, which imposes an angular condition on the different features, albeit a substantially weaker one than orthogonality. At a high level, Item~2 ensures that, along the relevant portion of the dynamics, the signal associated with the target feature $w_{m+1}^\star$ dominates the signal associated with the pre-trained features.


\section{Experiments}\label{sec:expe}

We complement our theoretical results with numerical experiments in a slightly more general setting, where the teacher weights are drawn independently from a standard Gaussian distribution and are therefore neither normalized nor pairwise orthogonal. Before fine-tuning, we initialize the network by setting the first $m$ weights to the corresponding optimal weights $w_1^\star,\ldots,w_m^\star$, while initializing the $(m+1)$-th weight with a uniformly random direction and a small norm. With problem parameters $m=20$, $d=500$ and $n=4000$, we then fine-tune the two-layer ReLU network using the different algorithms described in Section~\ref{sec:discussion}. Experimental details and additional experiments are provided in Appendix~\ref{app:expe}.

Figure~\ref{fig:MSEperfectpretrain} shows the evolution of the test loss during fine-tuning for the different methods. LP-FT performs well during the initial epochs, but its performance subsequently deteriorates as the model starts overfitting the fine-tuning data and, at the same time, degrading its pre-trained features (see below). Full fine-tuning exhibits a similar, but more pronounced, behavior: without an initial linear probing phase, it starts distorting the pre-trained features from the beginning of training, which leads to even worse generalization. The other methods are less affected by feature distortion and overfitting. Linear probing cannot reduce the test loss below a certain threshold because of its limited expressivity. In contrast, both two-timescale and L2-SP perform much better, although they require substantially more epochs to reach their best performance. 
L2-SP performs well as we regularize only the hidden layer, a choice that is particularly well suited to our setting; regularizing all parameters substantially degrades performance. We provide additional discussion and details on L2-SP in Appendix~\ref{app:expe_details}.

Figure~\ref{fig:MSEperfectpretrain} also highlights a trade-off between statistical performance and optimization speed: two-timescale and L2-SP achieve the best generalization by preserving the pre-trained features while learning the new feature, but might require more training epochs to do so. 
Importantly, all methods (except linear probing) eventually achieve nearly zero training loss. This indicates that their differences in test performance do not stem from their ability to minimize the training objective, but rather from the different forms of implicit regularization induced by their respective optimization schemes.

\begin{figure}[htbp]
\centering
\begin{subfigure}{.5\textwidth}
  \centering
  \includegraphics[width=0.95\linewidth, trim={0 0.7cm 0 0.7cm}, clip]{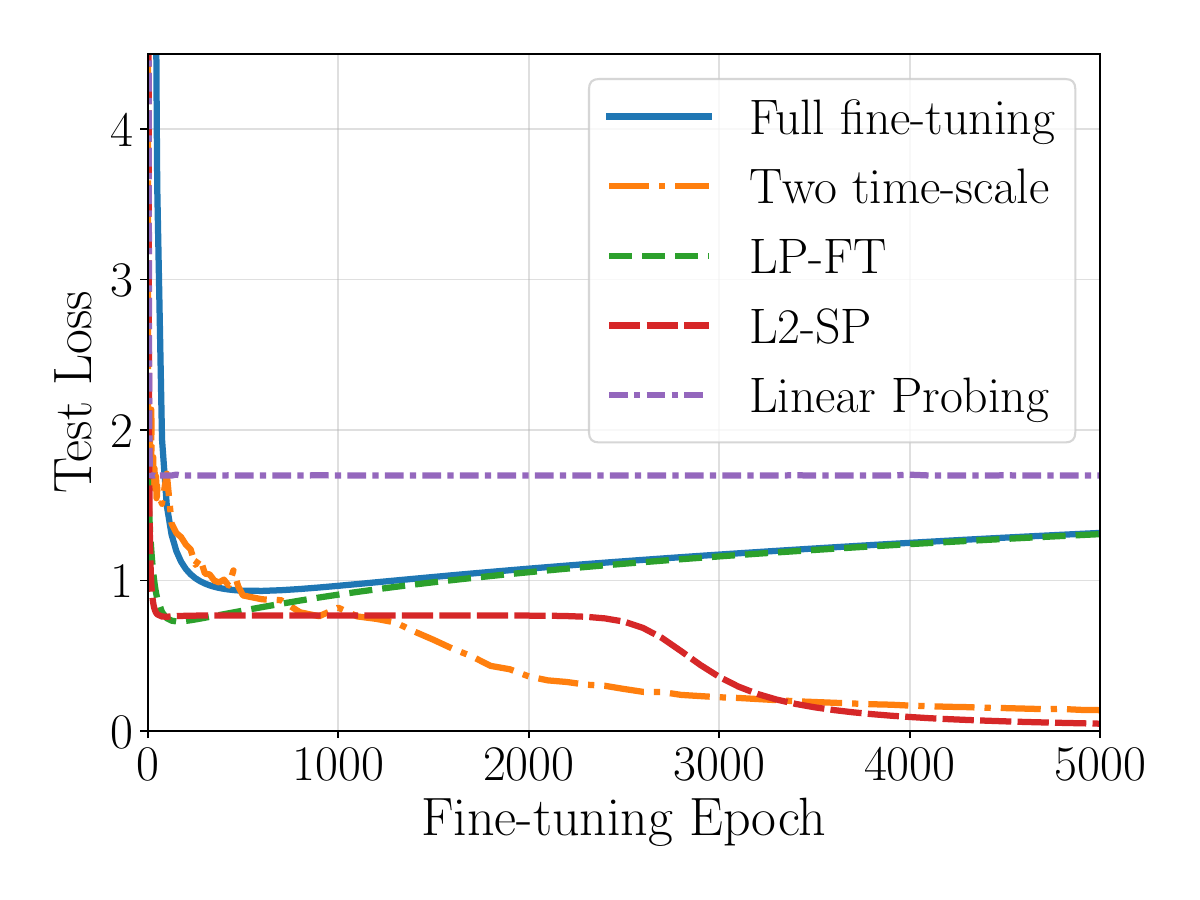}
  \caption{\label{fig:MSEperfectpretrain}Test loss.}
\end{subfigure}%
\begin{subfigure}{.5\textwidth}
  \centering
  \includegraphics[width=0.95\linewidth, trim={0 0.7cm 0 0.7cm}, clip]{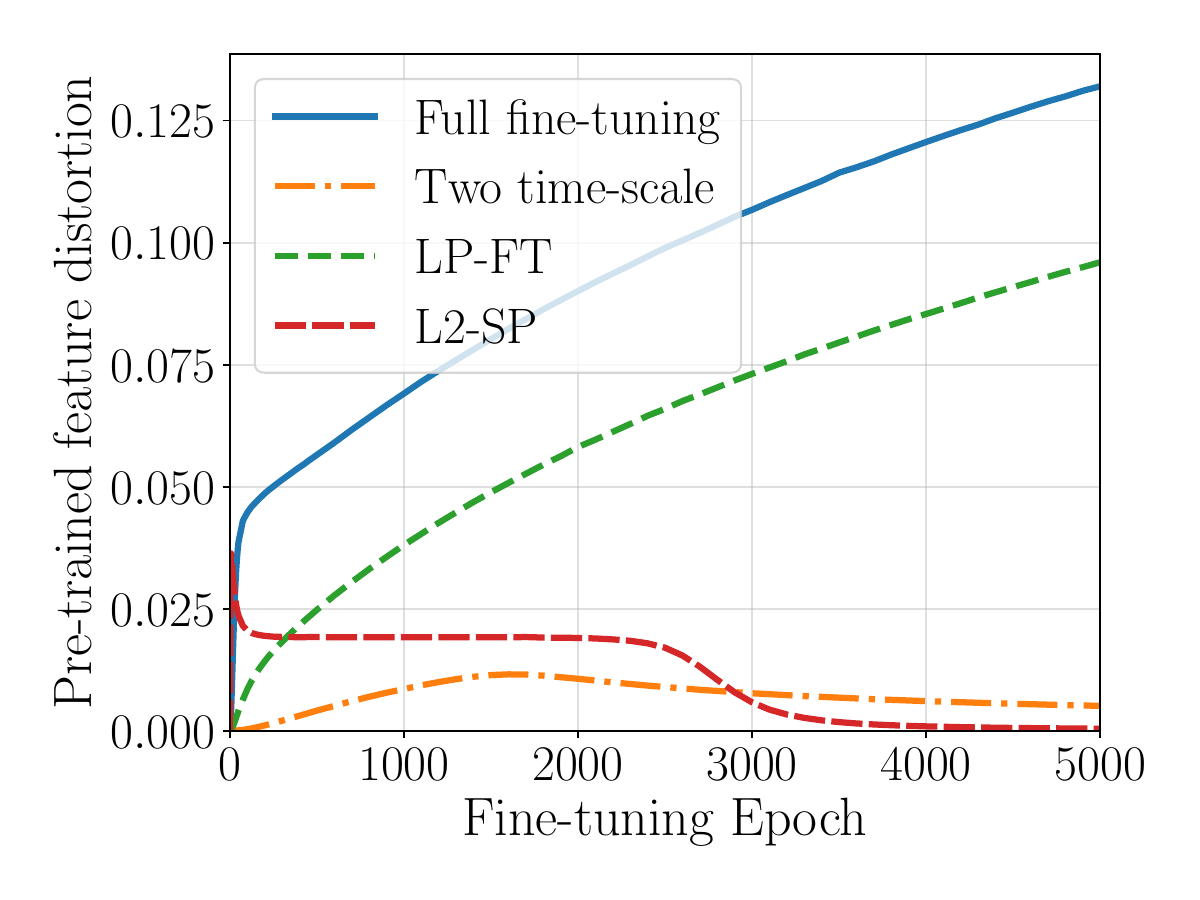}
  \caption{\label{fig:distortion}Distortion of pre-trained features.}

\end{subfigure}
\caption{\label{fig:main}Fine-tuning from pre-trained weights with different methods.}
\end{figure}
Figure~\ref{fig:distortion} quantifies the distortion of the pre-trained features induced by fine-tuning. \textit{Pre-trained feature distortion} is defined as the mean squared error of the fine-tuned model on data generated by the pre-training teacher, after re-optimizing the output weights by linear probing. Evaluating the pre-training error after this re-optimization isolates the quality of the learned features: it measures how well the fine-tuned representation still captures the features learned during pre-training.

Full fine-tuning and LP-FT lead to comparable levels of feature distortion. While feature distortion under full fine-tuning is well documented, LP-FT is generally expected to mitigate this effect. As explained in Section~\ref{sec:discussion}, however, LP-FT only delays feature distortion rather than preventing it, which is precisely what we observe in our experiments. 
On the other hand, two-timescale fine-tuning induces almost no feature distortion, suggesting that the limiting dynamics of Proposition~\ref{prop:limitdyn} accurately capture the fine-tuning dynamics even for finite initialization scales and learning-rate ratios. L2-SP also maintains low feature distortion, consistent with its explicit regularization toward the pre-trained hidden weights.

\subsection*{AI use statement}

In this work, we used generative AI tools to assist with polishing the writing, coding, and identifying relevant literature references. Generative AI tools were also used to explore some mathematical arguments. All mathematical proofs presented in the paper were developed, verified, and written by the human authors. Where AI tools provided useful suggestions, the authors critically evaluated, adapted, clarified, and improved them before incorporating the resulting arguments into the paper.

%

\ifarxiv
\subsubsection*{Acknowledgments}
This work was partially funded by the Swiss National Science Foundation, grant number 212111. 
This work benefited from the support of the FMJH Program PGMO. 
\fi

\bibliography{biblio.bib}
\bibliographystyle{iclr2027_conference}
\clearpage
\appendix

\addcontentsline{toc}{section}{Appendix} 
\part{Appendix} 
\parttoc 

\section{Additional experiments and experimental details}\label{app:expe}

\subsection{Experimental details}\label{app:expe_details}

For the experimental setup of Section~\ref{sec:expe}, we generate the ground truth features i.i.d. as $w_{i}^\star \cN(0,\id)$ and both pre-training teacher (used to evaluate feature distortion) and fine-tuning teacher are two-layer ReLU networks, where the outer weights are also drawn i.i.d. as standard Gaussian. The $m$ first features of the network are initialized as the ground truth features $w_i^\star$, while the $m+1$ feature is drawn as a random centered Gaussian of covariance $10^{-6}\id$. Output weights are also initialized at random, with unitary Gaussian for the $m$ first ones, and a Gaussian of variance $10^{-6}$ for the last one.

All the models are trained via stochastic gradient descent (SGD) over batches of size $200$. The learning rates are tuned among three or four values. For two-timescale, the learning rate associated to the output layer is $10^{-2}$, and $10^{-5}$ for the hidden layer. The regularization strength of $L2-SP$ is chosen as the best value among a dozen tested values. The Python code is available in the supplementary material.

\paragraph{L2-SP.} An important detail is that our implementation of L2-SP formulation does not regularize the squared distance between all model parameters and their pre-trained values. Instead, we only regularize the hidden parameters. This algorithmic choice is particularly well suited to the problem at hand and leads to substantially better performance for L2-SP.

This implementation choice is in fact crucial to the strong performance of L2-SP in our different experiments. When the new features are initialized close to $\mathbf{0}$, L2-SP can move them toward essentially any direction at little regularization cost. The corresponding output weights can then be scaled accordingly, allowing the model to add a new feature without incurring a significant regularization penalty. 

\paragraph{Random initialization.} We also trained the model from random initialization, corresponding to single-task learning without pre-trained weights. As expected, the model achieves nearly zero training loss but a much larger test loss, illustrating that the benefit of pre-training in our setting is statistical rather than purely optimization-based. The resulting test loss falls outside the range of the figures and is therefore not shown.

\subsection{Cosine similarities of learned features}\label{app:cosine}

\begin{figure}[htbp]
    \centering
\includegraphics[width=0.8\linewidth]{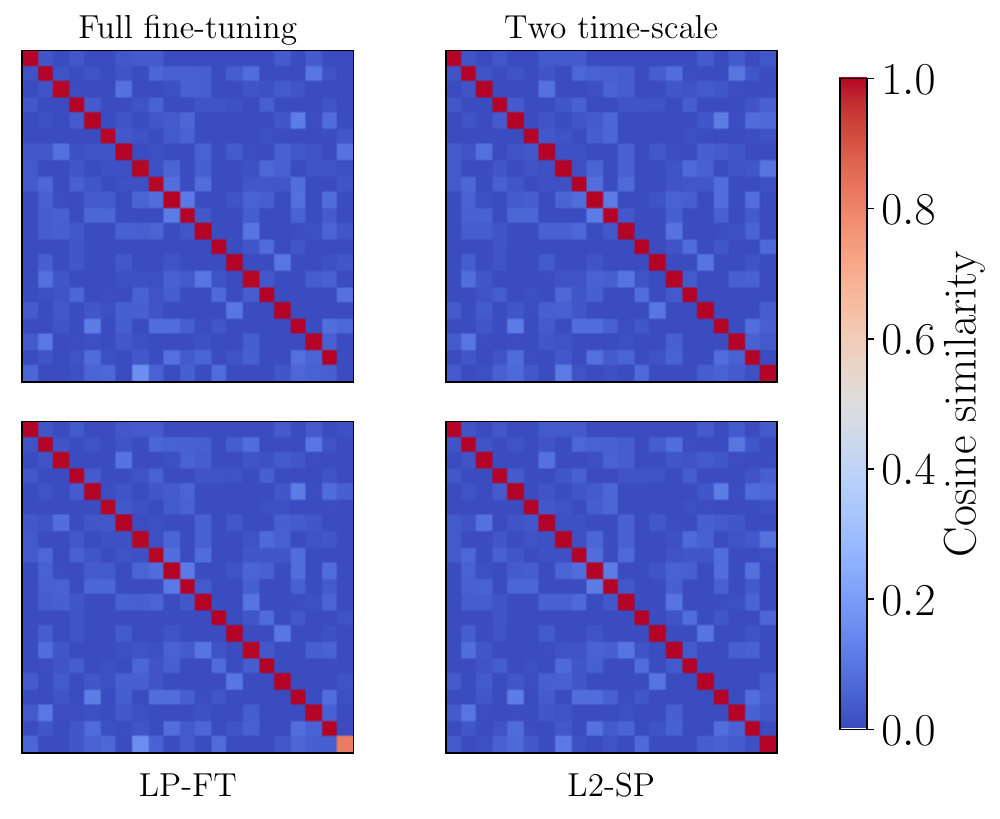}
    \caption{\label{fig:cosine}Cosine similarities between learned features and ground truth.}
\end{figure}

Figure~\ref{fig:cosine} illustrates the cosine similarities between the learned features and ground truth ones. Precisely, each colormap represents a matrix $C$ whose components are given by
\begin{equation*}
C_{ij} = \frac{w_j^\top w_{i}^{\star}}{\| w_j\| \cdot \|w_{i}^{\star}\|},
\end{equation*}
where the $w$ are the hidden weights of the represented model, after $5000$ training epochs. The columns are also permuted here for clarity of the figure. The setting is the same as in Figures~\ref{fig:MSEperfectpretrain} and \ref{fig:distortion}. Interestingly, it confirms that after fine-tuning, both two-timescale and L2-SP perfectly learned the new feature (given by the last red square on the diagonal), while other algorithms might not have especially learned it.

We here truncate the values below $0$ for a better visibility, as no value is much smaller than $0$.

LP-FT has learned the new feature up to some extent. However, this is only true at the late stage of fine-tuning. If one had stopped the fine-tuning of LP-FT at epoch $160$ (where its test loss is minimal in Figure~\ref{fig:MSEperfectpretrain}), this last diagonal square would also be blue, i.e., the new feature was not correctly learned at that time. If the model performance degrades although it learns this new feature between epoch $160$ and $5000$, it is thus because the pre-trained features are distorted in the meantime.
The features do not appear as distorted in Figure~\ref{fig:cosine} for both full fine-tuning and LP-FT. It is because even a tiny distortion, i.e., having a cosine similarity of approximately $99.7\%$ is sufficient enough to degrade the predictive performance, as observed in Figure~\ref{fig:distortion}. Thus, this kind of colormap is not necessarily the best to visualise feature distortion, but is very helpful to see whether the new feature has somehow been learned by the model.

\subsection{Overparameterized student}\label{app:expe_overparam}

In this section, we run the same experiment as in Section~\ref{sec:expe}, with the exception that the student network is overparameterized. To clarify, we consider the exact same $4000$ fine-tuning data samples as in Section~\ref{sec:expe} (see details in Appendix~\ref{app:expe_details}). However now, the learned model contains $100$ neurons instead of $21$: the first $m=20$ neurons are again initialized by the ground truth weights $w_1^\star, \ldots, w_m^\star$. And the remaining $80$ neurons are, again, initialized as i.i.d. random centered Gaussian of variance $10^{-6}$. 

Figure~\ref{fig:overparam} illustrates our findings in this setup. The empirical observations are very similar to the ones of Section~\ref{sec:expe}, suggesting that our theoretical findings should also hold when the student network is overparameterized.

\begin{figure}[htbp]
\centering
\begin{subfigure}{.5\textwidth}
  \centering
  \includegraphics[width=\linewidth, trim={0 0.7cm 0 0.7cm}, clip]{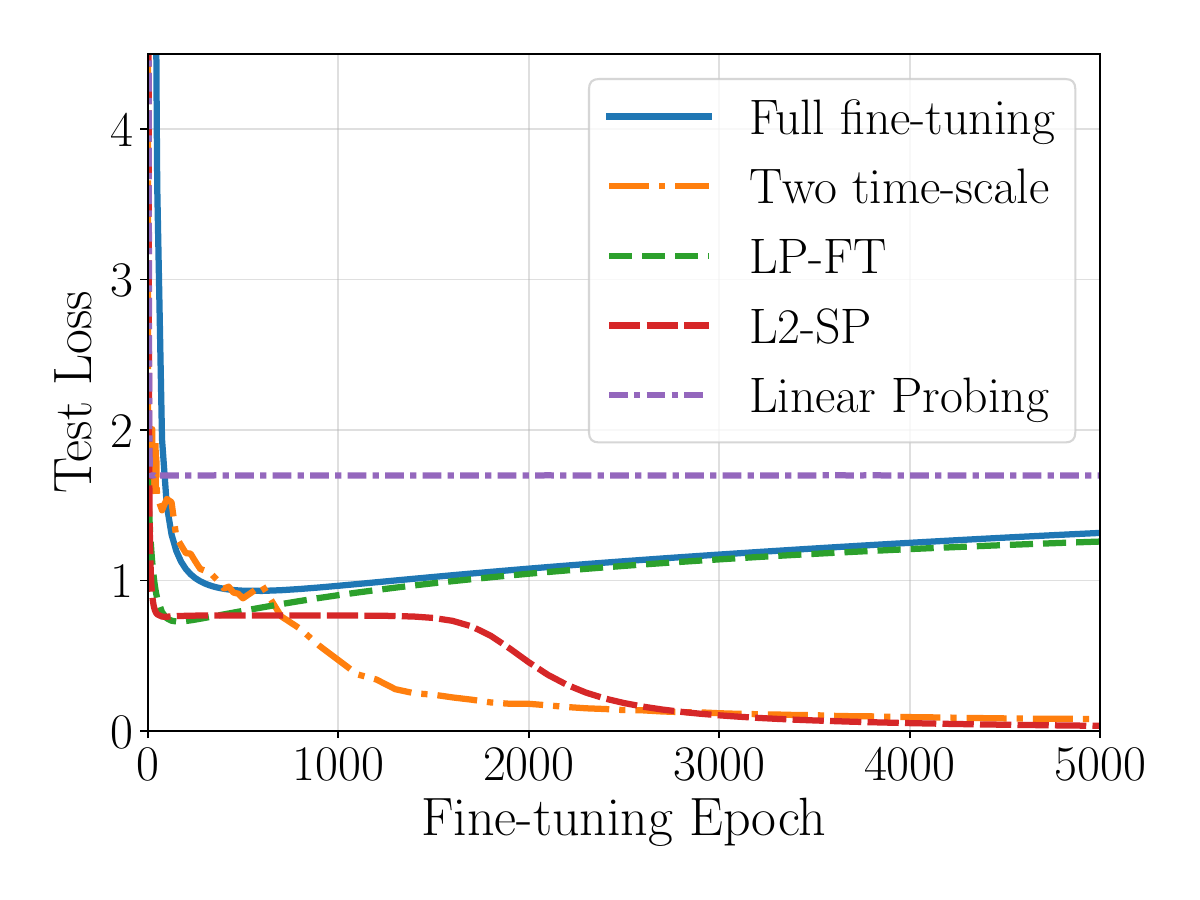}
  \caption{\label{fig:MSEperfectpretrain_overparam}Test loss.}
\end{subfigure}%
\begin{subfigure}{.5\textwidth}
  \centering
  \includegraphics[width=\linewidth, trim={0 0.7cm 0 0.7cm}, clip]{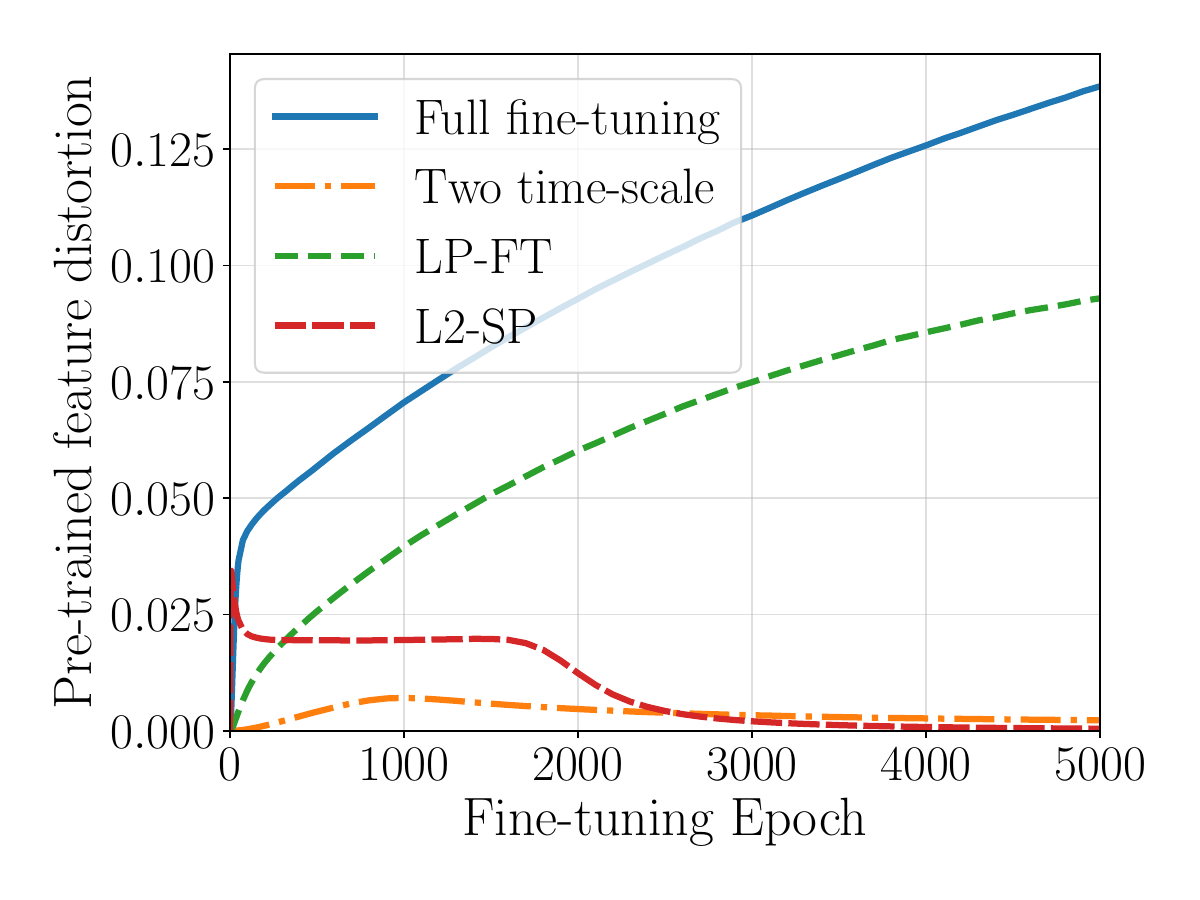}
  \caption{\label{fig:distortion_overparam}Distortion of pre-trained features.}

\end{subfigure}
\caption{\label{fig:overparam}Fine-tuning from pre-trained weights with an overparameterized model.}
\end{figure}

\subsection{Fine-tuning from real pre-training}\label{app:expe_real}

Until now, we considered fine-tuning from weights initialized according to Equation~\eqref{eq:init}. As discussed above, it idealizes the weights obtained after pre-training on data generated by a teacher given by the $m$ first ground truth parameters. This section empirically illustrates that similar fine-tuning behaviors happen, when fine-tuning on such a pre-trained model. Here is our empirical setup.

\textbf{Model pre-training.} We consider pre-training data of dimension $d = 500$, and represented by a teacher of $m_{\rm pre}=20$ neurons, $w_{1}^\star, \ldots, w_{m_{\rm pre}}^\star$ which are drawn i.i.d. as standard Gaussian. The pre-training data is then generated i.i.d. as follows:
\begin{equation*}
x \sim \cN(0,\id), \qquad y = \sum_{i=1}^{m_{\rm pre}} \sigma(x^\top w_i^\star).
\end{equation*}
We then consider an overparameterized two-layer ReLU network of $100$ neurons, initialized i.i.d. as centered Gaussian of variance $0.01$. We then pre-train the model via online SGD with an initial learning rate of $10^{-3}$, weight decay of $10^{-2}$, batches of size $256$ for $T=10^6$ steps. To stabilize the pre-training dynamics, we also use the standard PyTorch cosine learning rate scheduler.

\textbf{Fine-tuning.} We then fine-tune the model on a new fine-tuning task defined as follows. We first generate a teacher, by drawing at random $m_{\rm fine}=10$ features among the pre-trained ones, and add an additional feature $w^\star_{\rm fine}$ drawn at random following a standard Gaussian. More precisely, we let $\cI$ be a subset of $\{1,\ldots, m_{\rm pre}\}$ of size $m_{\rm fine}$---chosen uniformly at random---and the $n=8000$ fine-tuning samples are drawn i.i.d. as 
\begin{equation*}
x_k \sim \cN(0,\id), \qquad y_k =  \sigma(x_k^\top w_{\rm fine}^\star)+\sum_{i\in \cI} \sigma(x_k^\top w_i^\star).
\end{equation*}
We then fine-tune the model, from the pre-trained weights obtained by the pre-training procedure described above, for the different fine-tuning algorithms, following the same optimization procedure as described in Appendix~\ref{app:expe_details}.\footnote{We also use a cosine scheduler for L2-SP here for stability reasons.}

The fine-tuning procedure considered here is more challenging than the one in Section~\ref{sec:expe} for three reasons: 1) the initialization is obtained from an actually pre-trained model rather than from the idealized initialization considered in our theoretical analysis; 2) the model is overparameterized; 3) not all pre-trained features are relevant to the fine-tuning task.

The first point is the most challenging in practice. Overparameterization does not pose a significant difficulty, as shown in Appendix~\ref{app:expe_overparam}, and is in fact useful here, as it allows the model to retain \textit{void} features after pre-training. The third point simply considers a more general setting in which some of the features learned during pre-training are not reused by the fine-tuning task.

\begin{figure}[htbp]
\centering
\begin{subfigure}{.5\textwidth}
  \centering
  \includegraphics[width=\linewidth, trim={0 0.7cm 0 0.7cm}, clip]{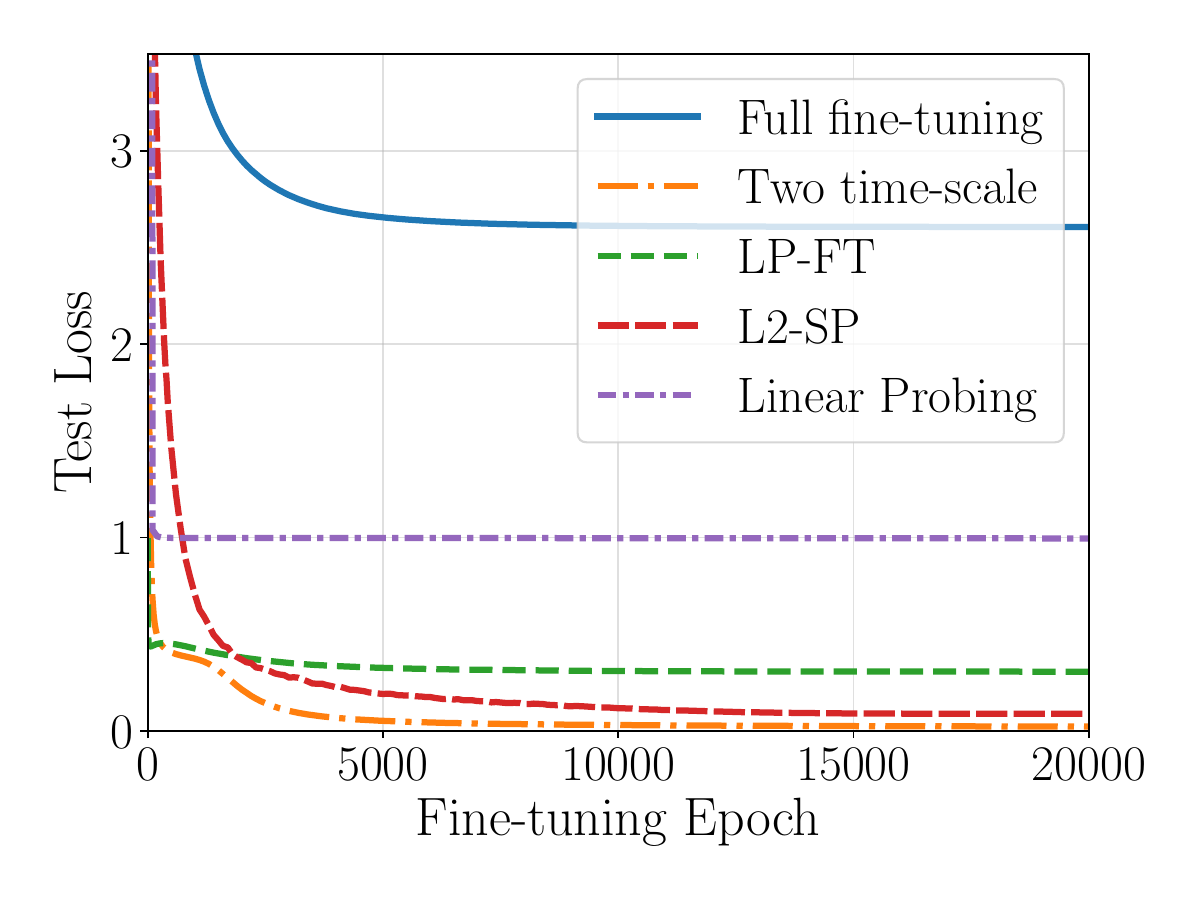}
  \caption{\label{fig:MSErealpretrain}Test loss.}
\end{subfigure}%
\begin{subfigure}{.5\textwidth}
  \centering
  \includegraphics[width=\linewidth, trim={0 0.7cm 0 0.7cm}, clip]{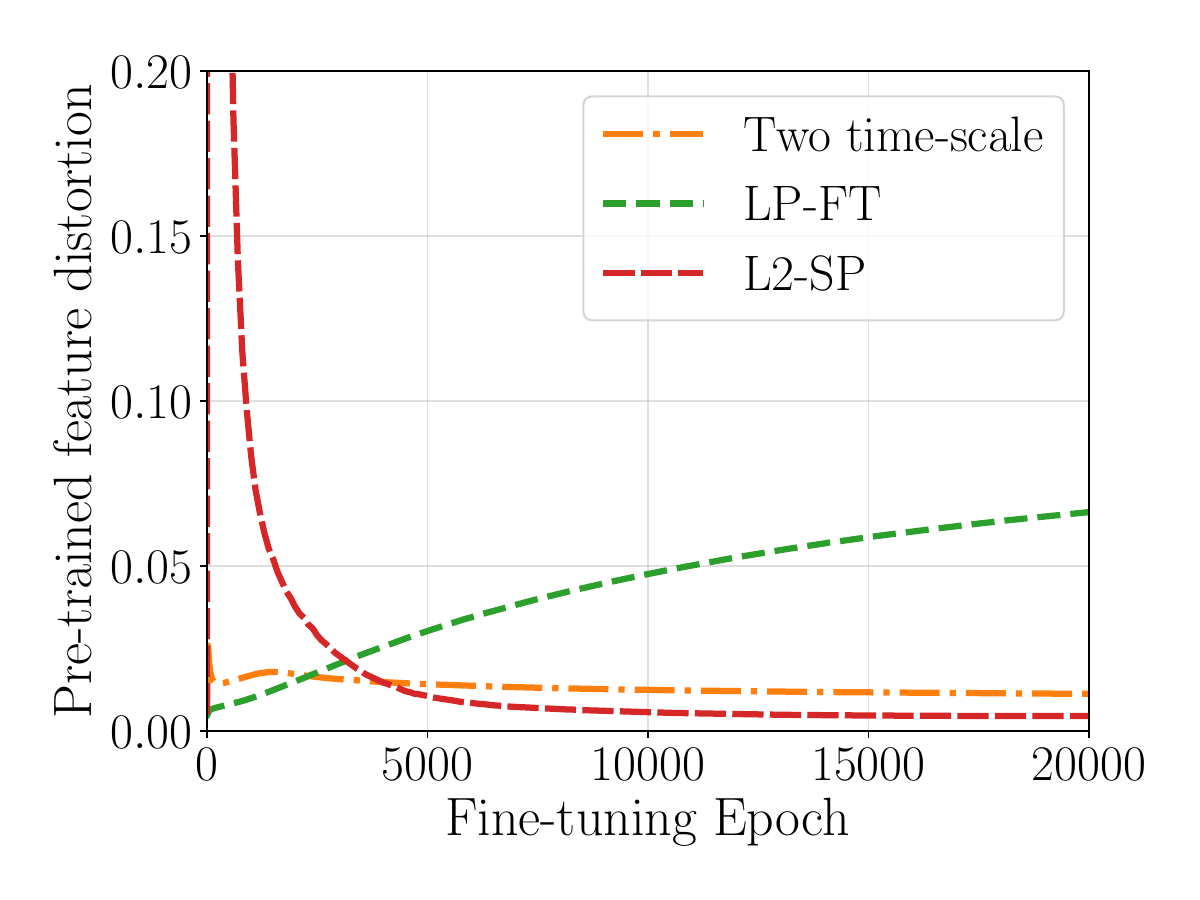}
  \caption{\label{fig:distortion_real}Distortion of pre-trained features.}

\end{subfigure}
\caption{\label{fig:realpretrain}Fine-tuning from real pre-trained weights (no idealized initialization).}
\end{figure}

Figure~\ref{fig:realpretrain} illustrates the evolution of the test loss and feature distortion during fine-tuning in this setting. We first note that achieving good performance here requires a larger sample size---and more epochs---than in the idealized setting of Section~\ref{sec:expe}. Thus, the theoretical setting does not directly translate to the practical setting considered here.

Nevertheless, the observations remain consistent with the main insights from our theoretical analysis. Two-timescale fine-tuning and L2-SP successfully recover the new feature while inducing little distortion of the pre-trained features---two-timescale even seems to perform better than L2-SP here. In contrast, both full fine-tuning and LP-FT suffer from feature distortion, although the effect is less pronounced for LP-FT.

Overall, the experiments in this section support the relevance of the idealized initialization scheme in Equation~\eqref{eq:init} as a model of practical fine-tuning, at least for the pre-training setups considered here.

\section{Proofs of Section~\texorpdfstring{\ref{sec:setting}}{2}}\label{app:setting}

\subsection{Proof of Lemma~\texorpdfstring{\ref{lemma:constantnorm}}{1}}
By definition of the two timescale dynamics, $\nabla_{a}\cL_n(a_n(W(t)), W(t))=0$, which can be rewritten for a any $i\in[m+1]$ as:
\begin{align*}
 - \frac{1}{n} \sum_{k=1}^n  (f_{a_n(W(t)), W(t)}(x_k)-y_k) \sigma(w_i(t)^\top x_k)=0.
\end{align*}
Lemma~\ref{lemma:constantnorm} is then proven by computing the time derivative of $\|w_i(t)\|^2$. Indeed, we have a.e.
\begin{align*}
\frac{1}{2}\frac{\df \|w_i(t)\|^2}{\df t} & =  - a_n(W(t))_{i}  \frac{1}{n} \sum_{k=1}^n  (f_{a_n(W(t)), W(t)}(x_k)-y_k) \sigma(w_i(t)^\top x_k)\\
& = 0.
\end{align*}
So that $\|w_i(t)\|$ is constant over time. \qed

\subsection{Proof of Proposition~\texorpdfstring{\ref{prop:limitdyn}}{1}}

The proof of Proposition~\ref{prop:limitdyn} is split in two main parts. The first one ensures that we can apply Arzel\`a-Ascoli theorem, ensuring the existence of a limit function -- up to extraction -- $\bW^\circ$. The second part then proves that this limit function is necessarily a solution of the limit process of Equation~\eqref{eq:oneneurondyn}.

\begin{lemma}\label{lemma:uniquea}
Suppose that Assumption~\ref{ass:conditioning} holds for some set $\cT$. We then have the following:
\begin{enumerate}
\item $a_n(W)$ is uniquely defined for any $W\in\cT$;
\item $W\mapsto a_n(W)$ is Lipschitz on $\cT$.
\end{enumerate}
\end{lemma}

\begin{proof}
By definition, $a_n(W) \in \argmin_{a\in\R^{m+1}} \|a^\top \Sigma(W,\bX)-Y\|^2$. This loss is convex and differentiable in $a$, so that it is equivalent to the following gradient condition:
\begin{align*}
\Sigma(W,\bX)\Sigma(W,\bX)^\top a_n(W)  = \Sigma(W,\bX) Y.
\end{align*}
In particular, if Assumption~\ref{ass:conditioning} holds for some set $\cT$, $\Sigma(W,\bX)\Sigma(W,\bX)^\top$ is invertible, so that this equation admits a unique solution $a_n(W)$, given by
\begin{equation*}
a_n(W)  = \left(\Sigma(W,\bX)\Sigma(W,\bX)^\top\right)^{-1}\Sigma(W,\bX) Y.
\end{equation*}
Note that, $W\mapsto \Sigma(W,\bX)$ is Lipschitz. Moreover, it is bounded on the bounded set $\cT$, so that $g:W\mapsto \Sigma(W,\bX)\Sigma(W,\bX)^\top$ is also Lipschitz on $\cT$. By continuity and compactness, the eigenvalues of $g(W)$ are bounded away from $0$ on $\cT$. Thus, the matrix inversion is Lipschitz on $g(\cT)$, so that $a_n$ is indeed Lipschitz on $\cT$ by composition of Lipschitz functions, and multiplication with bounded, Lipschitz functions. 
\end{proof}

\begin{corollary}\label{coro:uniformlipschitz}
Let $(\alpha_k)_{k\in\N}$ a positive sequence such that $\alpha_k \overset{k\to\infty}{\longrightarrow}0$. 
Suppose that for $k$ large enough, the trajectory $\left(\bW^{\alpha_k}(\alpha_k^2 t)\right)_{t\in[0,T]}$ is included within some set $\cT$, independent of $k$, satisfying Assumption~\ref{ass:conditioning}. Then for $k$ large enough, the sequence $\bW^{\alpha_k}(\alpha_k^2 \cdot)$ is uniformly Lipschitz.
\end{corollary}

\begin{proof}
$\cT$ is bounded and $a_n(W)$ is bounded on $\cT$ thanks to Lemma~\ref{lemma:uniquea}, so that $\partial \cL_n(a_n(W), W)$ is bounded on $\cT$. Equation~\eqref{eq:twotimescalenormalized} then directly yields that the time derivative of $\bW^{\alpha_k}(\alpha_k^2 \cdot)$ is uniformly bounded almost everywhere on $[0,T]$ for large enough $k$. In consequence, $\bW^{\alpha_k}(\alpha_k^2 \cdot)$ is uniformly Lipschitz on the considered time interval.
\end{proof}

Thanks to Corollary~\ref{coro:uniformlipschitz}, we can now apply the Arzel\`a-Ascoli theorem, for which we give a \textit{simplified} version below, to the sequence $\left(\bW^{\alpha_k}(\alpha_k^2 t)\right)_{t\in[0,T]}$.

\begin{theorem}[Arzel\'a-Ascoli, see e.g., \citealt{brezis2011functional}, Theorem 4.25]\label{thm:arzela-ascoli}
Let $(f_n)_{n\in\N}$ be a sequence of uniformly bounded and equicontinuous functions from a compact $K$ to $\R^d$. Then $(f_n)_{n\in\N}$ contains a uniformly convergent subsequence.
\end{theorem}
In Theorem~\ref{thm:arzela-ascoli}, a sequence of functions on $K$ is said to be uniformly equicontinuous (for the distance $d$) if for every $\varepsilon>0$, there exists $\delta>0$, such that for any $n\in\N$ and $x,y\in K$ with $d(x,y)<\delta$,
\begin{equation*}
\|f_n(x)-f_n(y)\| \leq \varepsilon.
\end{equation*}
Obviously, a sequence of uniformly Lipschitz functions is also equicontinuous. Thus, thanks to Corollary~\ref{coro:uniformlipschitz}, the sequence of functions $\bW^{\alpha_k}(\alpha_k^2 \cdot))_k$ satisfies the conditions of Theorem~\ref{thm:arzela-ascoli}, where the compact $K$ is given by $[0,T]$. We can then extract a subsequence $\alpha_{\varphi(k)}$ such that, as $k\to \infty$, $(\bW^{\alpha_{\varphi(k)}}(\alpha_{j_k}^2 t))_{t\geq 0}$ converges uniformly towards $\bW^\circ(t)$ on  $[0,T]$. It now remains to show that $\bW^\circ(t)$ is a solution of the differential inclusion given by Equation~\eqref{eq:oneneurondyn}.


First note that, by Equation~\eqref{eq:twotimescalenormalized}, for any $i\in[m]$, $\bw_{i}^{\alpha}(\alpha^2 t)$ obviously converges to $\frac{w_i^\star}{\|w_i^\star\|}$ for any $t\in\R_+$ as $\alpha\to0$. It directly implies the first line of Equation~\eqref{eq:oneneurondyn}.

It remains to show the second line of Equation~\eqref{eq:oneneurondyn}, i.e., the dynamics of $\bw^\circ_{m+1}$. Denote for shortness in the remaining of this proof $V^{k}(t) = \bW^{\alpha_{\varphi(k)}}(\alpha_{\varphi(k)}^2 t)$. Equation~\eqref{eq:twotimescalenormalized} rewrites for the $m+1$-th neuron:
\begin{equation*}
\dot{v}^{k}_{m+1}(t) \in F_{m+1}(V^{k}(t)),
\end{equation*}
where $F_{m+1}(V) = -\partial_{v_{m+1}}\cL(V)$. For $k$ large enough, $V^{k}$ is contained in $\cT$ on $[0,T]$, so that there exists a bounded function $f_k$ such that for any $k\in\N$ and $t\in[0,T]$,
\begin{equation}\label{eq:vkm1}
v^{k}_{m+1}(t) = \breve{w} +\int_0^t  f_k(s) \df s,
\end{equation}
where, almost everywhere, $f_k(s)\in  F_{m+1}(V^{k}(s))$.
Importantly, $f_k$ is uniformly bounded as soon as it is contained in$\cT$.  By Banach–Alaoglu theorem and separability of $L^1([0,T];\R^d)$, there exists $f\in L^\infty([0,T];\R^d)$ and a subsequence $(f_{\ell_k})$ of $(f_k)$ such that $f_{\ell_k}\rightharpoonup^* f$ in $L^\infty([0,T];\R^d)$ \citep[see e.g.,][Chapter~4]{brezis2011functional}, i.e., for any $\psi \in L^1([0,T];\R^d)$
\begin{equation*}
\int_{0}^T f_{\ell_k}(t) \psi(t) \df  t \longrightarrow \int_{0}^T f(t) \psi(t) \df  t.
\end{equation*}
The weak-$\star$ convergence and Equation~\eqref{eq:vkm1} directly imply that for any $t\in[0,T]$,
\begin{equation*}
v^k_{m+1}(t) \underset{k\to\infty}{\longrightarrow} \breve{w}+\int_{0}^t f(s)\df s.
\end{equation*}
In consequence, $\bw^\circ_{m+1}(t) =   \breve{w}+\int_{0}^t f(s)\df s$. As $f$ is bounded, $\bw^\circ_{m+1}$ is absolutely continuous (Lipschitz), and almost everywhere, 
\begin{equation}\label{eq:dotbwcirc}
\dot{\bw}^\circ_{m+1}(t) = f(t) .
\end{equation}
Since $L^2([0,T];\R^d) \subseteq L^1([0,T];\R^d)$, we also have $f_{\ell_k}\rightharpoonup f$ in $L^2([0,T];\R^d)$. 
By Mazur's lemma \citep[see e.g.,][Corollary 3.8]{brezis2011functional}, there exists a sequence $(g_k)_k$, of finite convex combinations of $(f_{\ell_k})_k$, more precisely,
\begin{equation*}
g_k = \sum_{j\geq k}\lambda^{(k)}_j f_{\ell_j},
\end{equation*}
such that $g_k$ converges strongly to $f$ in $L^2([0,T];\R^d)$. Passing to a subsequence \citep[][Theorem 4.9]{brezis2011functional}, we may therefore assume that $g_k(t)$ converges to $f(t)$ for almost any $t\in[0,T]$.

Moreover for any $k$, $f_k(t)\in  F_{m+1}(V^{k}(t))$. By upper semi-continuity of the Clarke subdifferential, $V^k(t) \to \bW^\circ(t)$ implies that for almost any $t\in[0,T]$,
\begin{equation*}
d(f_k(t),F_{m+1}(\bW^\circ(t))) \underset{k\to\infty}{\longrightarrow} 0,
\end{equation*}
where $d(x, S) = \inf_{y\in S} \|x-y\|_2$. By triangle inequality we thus also have  $d(g_k(t),F_{m+1}(\bW^\circ(t))) \underset{k\to\infty}{\longrightarrow} 0$. By closedness of the Clarke subdifferential and convergence of $g_k$ to $f$, it implies that, almost everywhere, $f(t)\in F_{m+1}(\bW^\circ(t))$, which combined with Equation~\eqref{eq:dotbwcirc} allows to conclude:
\begin{equation*}
\dot{\bw}^\circ_{m+1}(t) \in -\partial_{\bw^\circ_{m+1}}\cL(\bW^\circ(t)) .
\end{equation*}
\qed

\section{Proofs of Section~\texorpdfstring{\ref{sec:population}}{3.1}}\label{app:population}

\subsection{Notations and preliminaries}\label{app:prelipopulation}

In this section, we denote by $\phi_v:\R^d\to \R_+$ the function defined as $\phi_v(x) = \sigma(v^\top x)$. We also consider the Hilbert space $L^2(\cN(0,\id))$, defined by the scalar product, for any functions $f,g$ from $\R^d$ to $\R$:
\begin{equation*}
\langle f,g \rangle = \bE_{x\sim\cN(0,\id)}[f(x)g(x)].
\end{equation*}
We then define the linear operator $\Phi : \R^m \to L^2(\cN(0,\id))$, 
\begin{equation*}
\Phi(a) = \sum_{k=1}^m a_k \phi_{w_k^\star}.
\end{equation*}
Its adjoint is then given by $\Phi^*:L^2(\cN(0,\id))\to\R^m$ such that its $i$-th coordinate is given by
\begin{equation*}
\Phi^*(f)_i = \langle f, \phi_{w_i^\star} \rangle.
\end{equation*}
We also denote by $G=\Phi^* \Phi$ the $m\times m$ Gram matrix given for any $i,j\in[m]$ by $G_{ij} = \langle \phi_{w_i^\star}, \phi_{w_j^\star}\rangle$.
Whenever  it is invertible, we denote in the following by $\Pi = \id[L^2(\cN(0,\id))] - \Phi G^{-1} \Phi^*$ the orthogonal projection on $\{\phi_{w_1^\star}, \ldots, \phi_{w_m^\star}\}^\perp$.

Let also $P_m\in\R^{d\times d}$ be the orthogonal projection on $\Span(w_{1}^\star, \ldots, w_{m}^\star)$ and for any $w\in\bS_{d-1}$, let $P_{w^\perp} \coloneqq \id - ww^\top$ be the projection on the orthogonal of $w$.

The trajectory of $\bw^\circ_{m+1}$ will be confined within some set $\cT(r)$ defined for $r>0$ as:
\begin{equation*}
\cT(r) \coloneqq \{w\in\bS_{d-1} \mid \|P_m w\|_2\leq r \text{ and } -\frac{1}{4} \leq w^\top w_{m+1}^\star \}.
\end{equation*}
Define also $\kappa(x)=\frac{\sqrt{1-x^2}+(\pi-\arccos(x))x}{2\pi}$, so that the arc-cosine kernel yields for any $u,v\in\bS_{d-1}$, $\langle \phi_u, \phi_v \rangle = \kappa(u^\top v)$.

\subsection{Population dynamics}\label{app:popdynamics}

Define in the following
\begin{equation}\label{eq:HDpop}
\begin{aligned}
&H(w) \coloneqq \kappa(w^\top w^\star_{m+1})- \frac{1}{m+\pi-1}\sum_{i=1}^m \kappa(w^\top w_{i}^\star)\\
&D(w) \coloneqq \frac{1}{2} - \sum_{i,j\in[m]} (G^{-1})_{i,j} \kappa(w^\top w_{i}^\star)\kappa(w^\top w_{j}^\star).
\end{aligned}
\end{equation}

Recall Lemma~\ref{lemma:populationloss} here.

\poplosslemma*

\begin{proof}
\textbf{Computation of $a(W)$. }
By definition, $a(W)$ minimizes the optimization problem:
\begin{align}\label{eq:optpopa}
\min_{a\in\R^{m+1}} \|\sum_{k=1}^{m+1} a_k \phi_{w_k} - \sum_{k=1}^{m+1}\phi_{w_k^\star}\|^2_{L^2(0,\cN(0,\id))},
\end{align}
Note that for a given $a_{m+1}$, it admits a unique minimizer -- thanks to the invertibility of the matrix $G$ given below -- on its first coordinates, given by 
\begin{equation*}
a_{1:m}= G^{-1}\Phi^*\left(\sum_{k=1}^{m+1}\phi_{w_k^\star} - a_{m+1}Z(w_{m+1})\right).
\end{equation*}
$G$ is indeed invertible under Assumption~\ref{ass:orthofeatures}, since the arc-cosine kernel then directly gives: $G= \frac{1}{2\pi}\left((\pi-1)\id[m] + \mathbf{1}_m\mathbf{1}_m^\top\right)$.

The optimization problem \eqref{eq:optpopa}, then reaches the minimal value for a fixed $a_{m+1}$:
\begin{equation*}
\left\|\Pi \left( a_{m+1}\phi_{w_{m+1}} - \phi_{w_{m+1}^\star} \right)\right\|^2_{L^2(\cN(0,\id))},
\end{equation*}
which is exactly the squared norm of the projection on the orthogonal of $\{\phi_{w_1^\star} \phi_{w_m^\star}, \ldots,\}$. We can then develop this term, so that,
\begin{equation}\label{eq:optamplus1}
\left\|\Pi \left( a_{m+1}\phi_{w_{m+1}} - \phi_{w_{m+1}^\star} \right)\right\|^2_{L^2(\cN(0,\id))}  = a_{m+1}^2 \tilde{D}(w_{m+1}) - 2a_{m+1}\tilde{H}(w_{m+1}) + \tilde{D}(w_{m+1}^\star) ,
\end{equation}
where
\begin{equation}\label{eq:tildedefn}
\tilde{D}(w) = \|\Pi \phi_{w}\|^2_{L^2(\cN(0,\id))}
\quad\text{and}\quad 
\tilde{H}(w) = \langle \Pi \phi_{w}, \Pi \phi_{w_{m+1}^\star} \rangle_{L^2(\cN(0,\id))}.
\end{equation}
This value is then minimal for a unique value of $a_{m+1}$ whenever $\tilde{D}(w_{m+1})>0$, which is given by 
\begin{equation*}
a_{m+1} = \frac{\tilde{H}(w_{m+1})}{\tilde{D}(w_{m+1})}.
\end{equation*}

\textbf{Functions $\tilde{D}$ and $\tilde{H}$.} To conclude on the value of $a_{m+1}(W)$, let us now show that $\tilde{D}$ and $\tilde{H}$ actually coincide with $D$ and $H$, defined in Equation~\eqref{eq:HDpop}, under Assumption~\ref{ass:orthofeatures} on the sphere $\bS_{d-1}$. Since $\Pi = \id[L^2(\cN(0,\id))] - \Phi G^{-1}\Phi^*$, we have for $w\in\bS_{d-1}$,
\begin{align}
\tilde{D}(w) & = \|\phi_{w}\|^2_{L^2(\cN(0,\id))} - \Phi^*(\phi_w)^\top G^{-1} \Phi^*(\phi_w)\notag\\
& = \frac{1}{2}-\Phi^*(\phi_w)^\top G^{-1} \Phi^*(\phi_w)\label{eq:Dupper1}\\
& = D(w).\notag
\end{align}
where we used the arc-cosine kernel to get that $\|\phi_w\|^2_{L^2(\cN(0,\id))}=\kappa(1)=\frac{1}{2}$ and $\Phi^\star(\phi_{w}) = \left(\kappa(w^\top w_{i}^\star)\right)_{i\in[m]}$. 

Note that the arc-cosine kernel also implies, with Assumption~\ref{ass:orthofeatures}, that $G=\frac{1}{2\pi}\left((\pi-1)\id[m] + \mathbf{1}_m\mathbf{1}_m^\top\right)$. 
For $\tilde{H}$, we can first use the Sherman-Morrison formula to get
\begin{align}
G^{-1} & = 2\pi \left((\pi-1)\id[m] - \mathbf{1}_m\mathbf{1}_m^\top\right)^{-1}\notag\\
& = \frac{2\pi}{\pi-1} \id[m] - \frac{2\pi}{(\pi-1)^2} \frac{\mathbf{1}_m \mathbf{1}_m^\top}{1+ \frac{2\pi}{\pi-1} m} \notag\\
& = \frac{2\pi}{\pi-1}\left( \id[m] - \frac{\mathbf{1}_m \mathbf{1}_m^\top}{\pi+m-1}\right). \label{eq:sherman}
\end{align}
Then, thanks to Assumption~\ref{ass:orthofeatures} and the arc-cosine kernel, $\Phi^* \phi_{w_{m+1}^\star} = \frac{1}{2\pi}\mathbf{1}_m$, so that
\begin{align*}
G^{-1}\Phi^* \phi_{w_{m+1}^\star} & = \frac{1}{\pi-1}\left( 1 - \frac{m}{\pi+m-1} \right) \mathbf{1}_m\\
& = \frac{1}{\pi+m-1}\mathbf{1}_m.
\end{align*}
For any $w\in\bS_{d-1}$,
\begin{align*}
\tilde{H}(w) & = \langle \Pi \phi_w, \Pi \phi_{w_{m+1}^\star} \rangle \\
& = \langle \phi_w ,\phi_{w_{m+1}^\star}\rangle -    \Phi^*(\phi_w)^\top G^{-1}\Phi^* \phi_{w_{m+1}^\star} \\
& = \langle \phi_w ,\phi_{w_{m+1}^\star}\rangle -\frac{1}{\pi+m-1}\Phi^*(\phi_w)^\top\mathbf{1}_m = H(w),
\end{align*}
which allows to conclude on the closed-formula of $a(W)_{m+1}$.

\textbf{Population Loss. }
The value of $\cL(a(W),W)$ is then given by the value of Equation~\eqref{eq:optamplus1} when $a_{m+1} = \frac{H(w)}{D(w)}$, which yields for any $w_{m+1}$ such that $\widetilde{D}(w_{m+1})>0$
\begin{align} \label{eq:popvalue}
\cL(a(W), W)  = \left( \widetilde{D}(w_{m+1}^\star) - \frac{\widetilde{H}(w_{m+1})^2}{\widetilde{D}(w_{m+1})}\right).
\end{align}
This directly allows to conclude for the population loss value when $w_{m+1}\in\bS_{d-1}$, thanks to the previous paragraph.

\paragraph{Gradient of population loss.}
First note that $\widetilde{H}$, $\widetilde{D}$ and $\cL$ are all differentiable by differentiability under the integral sign. Using the envelope theorem, it then comes whenever $\widetilde{D}(w_{m+1})>0$ from Equation~\eqref{eq:popvalue} that
\begin{equation}\label{eq:envelope}
\nabla_{w_{m+1}}\cL(a(W),W) = -\frac{\widetilde{H}(w_{m+1})}{\widetilde{D}(w_{m+1})} \left(2\nabla \widetilde{H}(w_{m+1})-\frac{\widetilde{H}(w_{m+1})}{\widetilde{D}(w_{m+1})}\nabla \widetilde{D}(w_{m+1})\right).
\end{equation}
Although $\widetilde{D}$ and $D$ coincide on the sphere, their gradient does not (and similarly with $\widetilde{H},H$). This equivalence is only true for their Riemannian gradient on the sphere, i.e.,
\begin{equation*}
P_{w^\perp}\nabla \widetilde D(w) = P_{w^\perp}\nabla D(w) \quad \text{for any }w\in\bS_{d-1}.
\end{equation*}
A similar equality holds for $\nabla H$ and $\nabla \widetilde{H}$. Now, thanks to Lemma~\ref{lemma:constantnorm} (or a similar argument in the population case), we also have $w_{m+1}^\top \nabla_{w_{m+1}}\cL(a(W),W)=0$, or said otherwise, $P_{w^\perp}\nabla_{w_{m+1}}\cL(a(W),W) = \nabla_{w_{m+1}}\cL(a(W),W)$. So that finally, Equation~\eqref{eq:envelope} yields for any $w_{m+1}\in\bS_{d-1}$,
\begin{align*}
\nabla_{w_{m+1}}\cL(a(W),W) &= -\frac{\widetilde{H}(w_{m+1})}{\widetilde{D}(w_{m+1})} \left(2P_{w^\perp}\nabla \widetilde{H}(w_{m+1})-\frac{\widetilde{H}(w_{m+1})}{\widetilde{D}(w_{m+1})}P_{w^\perp}\nabla \widetilde{D}(w_{m+1})\right)\\
& =  -\frac{H(w_{m+1})}{D(w_{m+1})} \left(2P_{w^\perp}\nabla H(w_{m+1})-\frac{H(w_{m+1})}{D(w_{m+1})}P_{w^\perp}\nabla D(w_{m+1})\right).
\end{align*}
\end{proof}

\begin{corollary}\label{coro:populationODE}
If $\bW^\circ$ follows the population dynamics given by Equation~\eqref{eq:ODEpop} and Assumption~\ref{ass:orthofeatures} holds, then for any $t\geq 0$:
\begin{align*}
&\frac{\df}{\df t} \bw_{m+1}^{\circ\,\top}w_{m+1}^{\star} = \frac{H(\bw_{m+1}^{\circ})}{D(\bw_{m+1}^{\circ})} w_{m+1}^{\star\,\top}P_{\bw_{m+1}^{\circ}(t)^\perp} \Psi(\bw_{m+1}^{\circ}(t)),\\
&\frac{\df}{\df t}\|P_m \bw_{m+1}^{\circ}(t)\| \leq \frac{H(\bw_{m+1}^{\circ})}{D(\bw_{m+1}^{\circ})} \|P_m P_{\bw_{m+1}^{\circ}(t)^\perp} \Psi(\bw_{m+1}^{\circ}(t))\|.
\end{align*}
\end{corollary}

\begin{proof}
The first point is a direct application of Lemma~\ref{lemma:populationloss} to Equation~\eqref{eq:ODEpop}. For the second point, $\|P_m \bw_{m+1}^{\circ}(t)\|$ is absolutely continuous. So it is differentiable almost everywhere, and
\begin{equation*}
\frac{\df}{\df t}\|P_m \bw_{m+1}^{\circ}(t)\| \leq \left\|\frac{\df}{\df t}P_m \bw_{m+1}^{\circ}(t)\right\|,
\end{equation*}
which then allows to conclude with Lemma~\ref{lemma:populationloss}.
\end{proof}

\subsection{Key inequalities}\label{app:keypop}

\begin{lemma}\label{lemma:poprate}
Let Assumption~\ref{ass:orthofeatures} hold. There exist positive universal constants $\underline{D}, c, C, m_0$, such that if $m\geq m_0$, then for any $w\in \cT(c)$:
\begin{enumerate}
\item $\underline{D} \leq D(w) \leq \frac{1}{2}$;
\item $w_{m+1}^{\star\,\top}P_{w^\perp} \Psi(w) \geq c (1-w^\top w_{m+1}^\star)$;
\item $\|P_m P_{w^\perp} \Psi(w)\|\leq C\left(\frac{1}{\sqrt{m}}+\|P_m w\|\right)$;
\end{enumerate}
Moreover, if $\|P_m w\|\leq\frac{1}{4(\pi-1)\sqrt{m}}$ and $w^\top w_{m+1}^\star\geq - \frac{1}{4 \pi m}$, 
\begin{enumerate}[start=4]
\item $H(w) \geq \frac{c}{m}$.
\end{enumerate}
\end{lemma}

\begin{proof}
\textbf{Preliminaries.} 
By monotonicity of $\kappa'$ for any $w\in\bS_{d-1}$,
\begin{align}
\|\Phi^*(\phi_w) - \frac{1}{2\pi}\mathbf{1}_m\|_2 & = \sqrt{\sum_{k=1}^m \left(\kappa(w^\top w_k^\star)-\kappa(0)\right)^2}\notag\\
& \leq \max_{k} \kappa'(|w^\top w_k^\star|) \sqrt{\sum_{k=1}^m \left(w^\top w_k^\star\right)^2}\notag\\
& \leq \kappa'(\|P_m w\|) \|P_m w\|_2 \label{eq:lipschitzbound1}
\end{align}
where we used the orthogonality property for the last line. Also, $\kappa'$ is upper bounded by $1/2$ in general, which we might also use as a looser but sufficient bound in the proof. 

In the following, we will also rely on the fact that $G^{-1}$ is a positive definite matrix, with its largest eigenvalue equal to $\frac{2\pi}{\pi-1}$ thanks to Equation~\eqref{eq:sherman}; and $G^{-1}\mathbf{1}_m=\frac{2\pi}{\pi+m-1}\mathbf{1}_m$. In consequence, we also have the bound for any $w\in\bS_{d-1}$:
\begin{align}
\|G^{-1} \Phi^*(\phi_w)\|_2 & \leq \frac{1}{2\pi}\|G^{-1}\mathbf{1}_m \|_2 + \| G^{-1}\|_{\op}\ \|\Phi^*(\phi_w)-\frac{1}{2\pi}\mathbf{1}_m\|_2\notag\\
&\leq \frac{1}{\sqrt{m}}+ \frac{\pi}{\pi-1}\|P_m w\|_2, \label{eq:boundnD1}
\end{align}
By orthogonality of the features, another important observation is that
\begin{align*}
\|P_m w\|_2^2 \leq 1-\left(w^\top w_{m+1}^\star\right)^2 \leq 2\left(1-w^\top w_{m+1}^\star\right).
\end{align*}
Let now $w\in\cT(r)$ for some constant $r\in(0,1)$ to be fixed later.

\textbf{1. Bound on $D$.} The upper bound is direct by Equation~\eqref{eq:Dupper1}. 
Define $D_0 = \frac{1}{2} - \frac{1}{(2\pi)^2}\mathbf{1}_m^\top G^{-1}\mathbf{1}_m$. Recall that the Sherman-Morrison formula of Equation~\eqref{eq:sherman} implies that
\begin{align*}
D_0 & = \frac{1}{2}-\frac{m}{2\pi(\pi+m-1)}\\
& =\frac{(\pi+m)(\pi-1)}{2\pi(\pi+m-1)} \geq \frac{\pi-1}{2\pi}.
\end{align*}
From there, observe that
\begin{align}
D_0 - D(w) & =  \Phi^*\phi_w^\top G^{-1} \Phi^*\phi_w - \frac{1}{(2\pi)^2} \mathbf{1}_m^\top G^{-1}\mathbf{1}_m \notag\\
& =  \left(\Phi^*\phi_w+\frac{1}{2\pi} \mathbf{1}_m\right)^\top G^{-1} \left(\Phi^*\phi_w-\frac{1}{2\pi} \mathbf{1}_m\right)\notag\\
& = \frac{1}{\pi} \mathbf{1}_m^\top G^{-1} \left(\Phi^*\phi_w-\frac{1}{2\pi} \mathbf{1}_m\right)+ \left\|\Phi^*\phi_w-\frac{1}{2\pi} \mathbf{1}_m\right\|_{G^{-1}}^2. \label{eq:lipschitzkappa}
\end{align}
so that finally, using the bound on  $\left\|G^{-1}\right\|_{\op}$ and the value of $\frac{1}{\pi}\mathbf{1}_m^\top G^{-1}$,
\begin{align}
|D(w) - D_0| & \leq \frac{2}{\pi+m-1} \mathbf{1}_m^\top \left(\Phi^*\phi_w-\frac{1}{2\pi} \mathbf{1}_m\right) + \frac{2\pi}{\pi-1}\left\|\Phi^*\phi_w-\frac{1}{2\pi} \mathbf{1}_m\right\|_2^2 \notag\\
& \leq \frac{2}{\sqrt{m}}\left\|\Phi^*\phi_w-\frac{1}{2\pi} \mathbf{1}_m\right\|_2 + \frac{2\pi}{\pi-1}\left\|\Phi^*\phi_w-\frac{1}{2\pi} \mathbf{1}_m\right\|_2^2 \notag\\
& \leq \frac{1}{\sqrt{m}} \|P_m w\|_2 + \frac{\pi}{\pi-1}\|P_m w\|_2^2\label{eq:boundD}
\end{align}
where we used Equation~\eqref{eq:lipschitzkappa} for the last line. By definition of $\cT(r)$, $\|P_m w\|_2\leq r$. By noting that $D_0\geq \frac{\pi-1}{2\pi}$, we conclude for the first item when $c$ is small enough.

\textbf{2. Bound on $w_{m+1}^{\star\,\top}P_{w^\perp} \Psi(w)$.} 
First recall that $\Psi(w) =  2\nabla H(w) - \frac{H(w)}{D(w)}\nabla D(w)$. The definitions of $H$ in Equation~\eqref{eq:HDpop}, along with the function $\kappa$ that allows to rewrite $\langle \phi_u,\phi_v\rangle = \kappa(u^\top v)$, yields the gradient:
\begin{align*}
\nabla H (w) = \kappa'(w^\top w_{m+1}^\star) w_{m+1}^\star - \frac{1}{m+\pi-1} \sum_{k=1}^m \kappa'(w^\top w_{k}^\star) w_{k}^\star.
\end{align*}
So that, using the orthogonality of the features,
\begin{align*}
w_{m+1}^{\star\,\top}P_{w^\perp} \nabla H(w) & = \kappa'(w^\top w_{m+1}^\star)  (1-(w^\top w_{m+1}^\star)^2) + \frac{w^\top w_{m+1}^\star}{m+\pi-1} (\sum_{k=1}^m \kappa'(w^\top w_{k}^\star) w^\top w_{k}^\star.
\end{align*}
Similarly, we have by definition of $D$:
\begin{align}\label{eq:nablaD}
\nabla D(w) & = -2 \sum_{k=1}^m \left(G^{-1} \Phi^*(\phi_w)\right)_k \kappa'(w^\top w_k^\star)w_k^\star,
\end{align}
so that,
\begin{align}
w_{m+1}^{\star\,\top}P_{w^\perp} \nabla D(w) & = 2 w_{m+1}^{\star\,\top}w \sum_{k=1}^m \left(G^{-1} \Phi^*(\phi_w)\right)_k \kappa'(w^\top w_k^\star) w^\top w_k^\star \notag.
\end{align}
We thus have
\begin{align}
w_{m+1}^{\star\,\top} P_{w^\perp} \Psi(w) & = 2\kappa'(w^\top w_{m+1}^\star) \left(1-\left(w^\top w_{m+1}^\star\right)^2\right) \notag\\
& \phantom{=}+ 2w^\top w_{m+1}^\star \sum_{k=1}^m \kappa'(w^\top w_k^\star) w^\top w_k^\star \left(\frac{1}{m+\pi-1} - \frac{H(w)}{D(w)} (G^{-1}\Phi^*(\pi_w))_k\right)\notag\\
&\geq 2\kappa'(w^\top w_{m+1}^\star) \left(1-\left(w^\top w_{m+1}^\star\right)^2\right) \notag\\
& \phantom{\geq}- 2\max_{k} |\kappa'(w^\top w_k^\star)| \|P_m w\| \cdot \left\| \frac{1}{m+\pi-1} \mathbf{1}_m - \frac{H(w)}{D(w)}G^{-1}\Phi^*(\phi_w) \right\|,\label{eq:P2popbound1}
\end{align}
where the inequality is just Cauchy-Schwarz. We then decompose as follows
\begin{multline}\label{eq:P2popbound2}
\left\| \frac{1}{m+\pi-1} \mathbf{1}_m - \frac{H(w)}{D(w)}G^{-1}\Phi^*(\phi_w) \right\|_2  \leq \|G^{-1}\|_{\op} \left\|\frac{1}{m+\pi-1} G\mathbf{1}_m - \Phi^*(\phi_w)\right\| \\ + \left|1-\frac{H(w)}{D(w)}\right| \|G^{-1}\Phi^*(\phi_w)\|.
\end{multline}
Recall that $\|G^{-1}\|_{\op}\leq \frac{2\pi}{\pi-1}$, $\|G^{-1}\Phi^*(\phi_w)\|\leq \frac{1}{\sqrt{m}}+\frac{\pi}{\pi-1}\|P_m w\|$, and by Equation~\eqref{eq:lipschitzbound1}
\begin{align*}
\left\|\frac{1}{m+\pi-1} G\mathbf{1}_m - \Phi^*(\phi_w)\right\| & = \left\|\frac{1}{2\pi} \mathbf{1}_m - \Phi^*(\phi_w)\right\| \\
& \leq \kappa'(\|P_m w\|) \|P_m w\|.
\end{align*}
It thus remains to bound $\left|1-\frac{H(w)}{D(w)}\right|$, or equivalently the difference $|D(w) - H(w)|$ as $D$ is lower bounded by $\underline{D}$. It then comes, using the bounds above, that
\begin{align}
|D(w) - H(w)|  &= \left|\kappa(1) - \kappa(w^\top w_{m+1}^\star) + \Phi^*(\phi_w)^\top G^{-1} \left(\Phi^*(\phi_{w_{m+1}^\star}) -\Phi^*(\phi_{w}) \right)\right| \notag\\
& \leq \max_{x\in[-1,1]} \kappa'(x) (1-w^\top w_{m+1}^\star) + \|G^{-1}\Phi^*(\phi_w)\| \cdot \|\Phi^*(\phi_{w_{m+1}^\star}) -\Phi^*(\phi_{w})\|\notag\\
& \leq \frac{1}{2}(1-w^\top w_{m+1}^\star) +  \frac{1}{2}\|P_m w\| \left(\frac{1}{\sqrt{m}}+\frac{\pi}{\pi-1}\|P_m w\|\right).\label{eq:ratioHD}
\end{align}
So wrapping things up back to Equation~\eqref{eq:P2popbound1}, and using the fact that $\|P_m w\|^2 \leq (1-(w^\top w_{m+1}^\star)^2)$, there is a universal constant $C$ such that, as long as $D(w)\geq \underline{D}$,
\begin{multline*}
\left\| \frac{1}{m+\pi-1} \mathbf{1}_m - \frac{H(w)}{D(w)}G^{-1}\Phi^*(\phi_w) \right\|_2  \leq \kappa'(\|P_m w\|)\frac{2\pi}{\pi-1}\|P_m w\| \\
+ C \left(1-w^\top w_{m+1}^\star + \frac{\|P_m w\|}{m}\right).
\end{multline*}
And plugging it back in Equation~\eqref{eq:P2popbound1}, again using that $\|P_m w\|^2 \leq 1-(w^\top w_{m+1}^\star)^2$ and that $1-w^\top w_{m+1}^\star \leq \frac{4}{3}(1-(w^\top w_{m+1}^\star)^2)$ on the considered set then yields,
\begin{multline*}
w_{m+1}^{\star\,\top} P_{w^\perp} \Psi(w) \geq 2 \left(\kappa'(w^\top w_{m+1}^\star)-\frac{2\pi}{\pi-1}\kappa'(\|P_m w\|)^2-\frac{2C}{3\underline{D}}\left(\|P_m w\|+\frac{1}{m}\right)\right) \cdot\\ \left(1-(w^\top w_{m+1}^\star)^2\right)
\end{multline*}
Now note that $\kappa'(-\frac{1}{4}) \geq \frac{1}{4} > \frac{2\pi}{\pi-1}\kappa'(0)^2 = \frac{\pi}{8(\pi-1)}$, so we can choose universal positive constants $m_0$ large enough and $c$ small enough, such that if $m\geq m_0$ and $w\in\cT(c)$,
\begin{equation}\label{eq:P2popbound3}
w_{m+1}^{\star\,\top} P_{w^\perp} \Psi(w) \geq c\left(1-w^\top w_{m+1}^\star\right).
\end{equation}

\textbf{3. Bound on $\|P_m P_{w^\perp} \Psi(w)\|$.} 
Similarly to 2., we have by orthogonality
\begin{align*}
P_m \nabla H(w) & = - \frac{1}{m+\pi-1} \sum_{k=1}^m \kappa'(w^\top w_{k}^\star) w_{k}^\star,
\end{align*}
so that its norm is bounded by $\frac{1}{2\sqrt{m}}$. Moreover by orthogonality and thanks to Equation~\eqref{eq:boundnD1},
\begin{align}
\left\|\nabla D(w)\right\| & \leq  \|G^{-1} \Phi^*(\phi_w)\|_2 \notag\\
&\leq \frac{1}{\sqrt{m}}+ \frac{\pi}{\pi-1}\|P_m w\|_2, 
\end{align}
where we used Equation~\eqref{eq:boundnD1}.  
In consequence, $\|P_m \nabla D(w) \|_2 \leq \frac{1}{\sqrt{m}} + \frac{5}{2}\|P_m w\|_2$, so that
\begin{align*}
\|P_m P_{w^\perp} \Psi(w)\| & \leq  \|P_m \Psi(w)\| + |w^\top \Psi(w)| \ \|P_m w\|_2\\
& \leq \frac{1}{\sqrt{m}} +  \frac{1}{\underline{D}} \left( \frac{1}{\sqrt{m}} + \frac{5}{2}\|P_m w\|_2\right) + \|\Psi(w)\| \ \|P_m w\|_2.
\end{align*}
We then conclude by observing that $\|\Psi(w)\|_2 \leq \sqrt{2} + \frac{1}{\underline{D}} \|\nabla D(w) \|_2$ on $\cT(c)$, so that we recover the final bound on $\|P_m P_{w^\perp} \Psi(w)\|$ for some constant $C$.

\textbf{4. Bound on $H$.} 
Lastly, $\kappa$ is convex, so that for any $x\in[-1,1]$,
\begin{equation*}
\kappa(x) \geq \kappa(0)+\kappa'(0) x = \frac{1}{2\pi} + \frac{x}{4}.
\end{equation*}
Moreover it is below its chords, so for $x\in[0,1]$,
\begin{equation*}
\kappa(x) \leq \kappa(0)+(\kappa(1)-\kappa(0))x = \frac{1}{2\pi} + \frac{\pi-1}{2\pi}x.
\end{equation*}
Using these two inequalities, and the monotonicity of $\kappa$, it comes by definition of $H$ that
\begin{align*}
H(w) & =\kappa(w^\top w^\star_{m+1}) - \frac{1}{m+\pi-1}\sum_{k=1}^m \kappa(w^\top w^\star_k)\\
&\geq \frac{1}{2\pi}(1-\frac{m}{m+\pi-1}) + \frac{w^\top w^\star_{m+1}}{4} + \frac{\pi-1}{2\pi(\pi+m-1)}\sum_{k=1}^m |w^\top w^\star_k| \\
&\geq \frac{\pi-1}{2\pi(m+\pi-1)} -\frac{w^\top w^\star_{m+1}}{4m}- \frac{\pi-1}{2\pi m} \|P_m w\|_1\\
& \geq  \frac{1}{4\pi m}(1-\pi m \ w^\top w^\star_{m+1} - 2(\pi-1)\sqrt{m}\|P_m w\|_2),
\end{align*}
which allows to conclude for the considered values of $\|P_m w\|$ and $w^\top w^\star_{m+1}$.
\end{proof}

\begin{lemma}\label{lemma:initpop}
Let $\delta\in(0,1/2)$. If Assumption~\ref{ass:orthofeatures} holds, $\breve{w}\sim \unif(\bS_{d-1})$ and
\begin{equation*}
d \geq C m^2\ln(1/\delta),
\end{equation*}
where $C$ is a positive universal constant. Then with probability at least $1-\delta$,
\begin{equation*}
\|P_m \breve{w}\|\leq \frac{1}{4(\pi-1)\sqrt{m}}\quad \text{and}\quad \breve{w}^\top w_{m+1}^\star \geq -\frac{1}{4 \pi \ m}.
\end{equation*}
\end{lemma}

\begin{proof}
\textbf{Tail of $\|P_m w\|_2$.} First note that by orthogonality and symmetry, $\bE[\|P_m w\|^2_2] = \frac{m}{d}$. As $x\mapsto \|P_m x\|^2$ is $2$-Lipschitz on the unit sphere, we can apply Theorem 5.1.4 of \citet{vershynin2018high}, so that there exists a universal constant $c>0$, such that for any $t\geq0$,
\begin{align*}
\bP\left(\|P_m \breve{w}\| \geq \sqrt{t + \frac{m}{d}}\right) & \leq \exp\left(-cd t^2\right).
\end{align*}
Denote for shortness $r=\frac{1}{4(\pi-1)}$. Taking $t$ such that $\sqrt{t+\frac{m}{d}}=\frac{r}{\sqrt{m}}$ -- $t\geq0$ when $r\leq \frac{m}{\sqrt{d}}$. It then comes for $d\geq 2\frac{m^2}{r^2}$,
\begin{align*}
\bP\left(\|P_m \breve{w}\| \geq \frac{r}{\sqrt{m}}\right) & \leq \exp\left(-cd(r^2/m-m/d)^2\right)\\
&\leq  \exp\left(-cd\frac{r^4}{4m^2}\right)
\end{align*}
Thus, if $d\geq C m^2\ln(1/\delta)$ for a large enough universal constant $C$, $\bP\left(\|P_m \breve{w}\| \geq \frac{r}{\sqrt{m}}\right)\leq \delta/2$.

\textbf{Concentration of $\breve{w}^\top w_{m+1}^\star$.} 
Similarly, $\bE[\breve{w}^\top w_{m+1}^\star]=0$, so that applying Theorem 5.1.4 of \citet{vershynin2018high} again, for any $t\geq 0$:
\begin{align*}
\bP\left(|\breve{w}^\top w_{m+1}^\star|\geq t\right) &\leq \exp\left(-cd t^2\right).
\end{align*}
In particular,
\begin{align*}
\bP\left(|\breve{w}^\top w_{m+1}^\star|\geq \frac{1}{4 \pi m}\right) &\leq \exp\left(-cd \frac{1}{(4 \pi)^2 m^2}\right).
\end{align*}
Thus, if $d\geq C m^2\ln(1/\delta)$ for a large enough universal constant $C$, $\bP\left(|\breve{w}^\top w_{m+1}^\star|\geq \frac{\eta}{m}\right)\leq \delta/2$, which concludes by taking a union bound on the two events above.
\end{proof}

\subsection{Proof of Theorem~\texorpdfstring{\ref{thm:mainpop}}{1}}\label{app:popproof}

Consider the positive universal constants $C,c, \underline{D}$ and $m_0$ appearing in Lemma~\ref{lemma:poprate}.

Lemma~\ref{lemma:initpop} yields that with probability at least $1-\delta$, 
\begin{equation}\label{eq:goodinitpop}
\|P_m w_{m+1}(0)\|\leq \frac{1}{4(\pi-1)\sqrt{m}}\quad \text{and}\quad w_{m+1}(0)^\top w_{m+1}^\star \geq -\frac{1}{4 \pi \ m}.
\end{equation}
with the considered regime for $d$---for a large enough universal constant $C$. Let us assume that the random event of Equation~\eqref{eq:goodinitpop} holds in the following. 

Define now $\tau \coloneqq \inf \left\{t\geq 0 : w_{m+1}(t)\not\in \cT\left(c\right)\right\}$. Thanks to Lemma~\ref{lemma:poprate}, we have for any $t\in[0,\tau]$, a positive constant $\underline{D}$ such that $D(w_{m+1}(t))\in[\underline{D}, \frac{1}{2}]$. 
We then also have by the fourth point of Lemma~\ref{lemma:poprate} that $H(w_{m+1}(0))\geq \frac{c}{m}$. 
Note that thanks to Lemma~\ref{lemma:populationloss}, 
\begin{equation*}
\cL(a(W(t)),W(t)) = D(w_{m+1}^\star)-\frac{H(w_{m+1}(t))^2}{D(w_{m+1}(t))},
\end{equation*}
so that by chain rule, $\frac{H(w_{m+1}(t))^2}{D(w_{m+1}(t))}$ is increasing over time. In particular, $\frac{H(w_{m+1}(t))^2}{D(w_{m+1}(t))}\geq 2\frac{c^2}{m^2}$ on $[0,\tau]$. By continuity, $H(w_{m+1}(t))$ is of constant sign, so that, using the bound on $D$ (and $H$)
 \begin{equation}
\frac{H(w_{m+1}(t))}{D(w_{m+1}(t))}\in \left[\frac{Zc}{m}, \frac{1}{2\underline{D}}\right] \qquad \text{for any }t\in[0,\tau].
 \end{equation}
We can thus reparametrize time, i.e., define 
\begin{equation*}
\gamma : s \mapsto \int_{0}^s\frac{D(w_{m+1}(u))}{H(w_{m+1}(u))}\df u \qquad \text{and}\qquad\tw(s) \coloneqq w_{m+1}(\gamma(s)).
\end{equation*}
The function $\gamma$ is increasing and $ 2\underline{D} s \leq \gamma(s)\leq \frac{m}{2c} s$. By reparametrization, it comes for any $s\geq 0$:
 \begin{align}
 \dot{\tw}(s) & = \frac{D(w_{m+1}(\gamma(s)))}{H(w_{m+1}(\gamma(s)))} \dot{w}_{m+1}(\gamma(s))\notag\\
 & = P_{w_{m+1}(\gamma(s))^\perp}\Psi(w_{m+1}(\gamma(s))). \label{eq:reparamODE}
 \end{align}
Let\footnote{Letting $\overline{s}=\infty$ if $\tau=\infty$.} $\overline{s} = \gamma^{-1}(\tau)$. In particular, $\tw(s)\in\cT(c)$ for any $s\in[0,\overline{s}]$. Thanks to Equation~\eqref{eq:reparamODE} and Lemma~\ref{lemma:poprate}, for any $s\in[0,\overline{s}]$:
\begin{gather}
\frac{\df }{\df s}w_{m+1}^{\star\,\top}\tw(s) \geq c (1-w_{m+1}^{\star\,\top}\tw(s)),\label{eq:growthpop0}\\
\frac{\df}{\df s}\|P_m \tw(s)\|\leq C_0\left(\frac{1}{\sqrt{m}}+\|P_m w\|\right),
\end{gather}
for some constant $C_0$. 
In particular, the first point implies that $(1-w_{m+1}^{\star\,\top}\tw(s))$ is increasing over time, and by a Gr\"onwall argument, for any $s\in[0,\overline{s}]$ 
\begin{equation}\label{eq:growthpop1}
1- w_{m+1}^{\star\,\top}\tw(s) \leq \frac{5}{4}	e^{-cs}.
\end{equation}
Moreover the second point implies, again by a Gr\"onwall argument that for any $s\in[0,\overline{s}]$:
\begin{equation*}
\|P_m \tw(s)\| \leq \left(\|P_m \tw(0)\|+\frac{1}{\sqrt{m}}\right)e^{C_0 s}.
\end{equation*}
Moreover, $\|P_m \tw(s)\| \leq \sqrt{2(1-w_{m+1}^{\star\,\top}\tw(s))}$, so that for any $s\in[0,\overline{s}]$:
\begin{align*}
\|P_m \tw(s)\| & \leq 2\min\left(\frac{e^{C_0 s}}{\sqrt{m}} , e^{-cs/2} \right)\\
& = 2 \exp\left( -\frac{c/2}{C_0+c/2}\ln(\sqrt{m})\right).
\end{align*}
Thus, for a large enough choice of $m_0$, $\|P_m \tw(s)\| \leq \frac{c}{2}$ on $[0, \overline{s}]$, so that $\overline{s} =  \infty$, i.e., $w_{m+1}(t) \in \cT(c)$ for any $t\geq 0$. 

Moreover, Equation~\eqref{eq:ratioHD} in the proof of Lemma~\ref{lemma:poprate} shows that for a small enough positive universal constant $\eta$, $H(w)\geq \frac{\underline{D}}{2}$ for any $w\in \bS_{d-1}$ such that $w^\top w_{m+1}^\star \geq 1 - \eta$. In consequence, letting now $s_\eta = \inf \{s \geq 0 \mid w_{m+1}^{\star\,\top}\tw(s) \geq 1 - \eta\}$, Equation~\eqref{eq:growthpop1} and the fact that $\gamma(s)\leq \frac{m}{2c} s$ imply that $\gamma(s_\eta) \leq C_0 m$ for a large enough universal constant $C_0$. Moreover, for any $s\geq s_\eta$,
\begin{equation*}
\gamma'(s) = \frac{D(w_{m+1}(s))}{H(w_{m+1}(s))} \leq \frac{1}{\underline{D}}.
\end{equation*}
And in particular, for any $t\geq \gamma(s_\eta)$, Equation~\eqref{eq:growthpop0} now implies:
\begin{align*}
1- \bw_{m+1}^\circ(t)^\top w_{m+1}^\star & = 1- \tw(\gamma^{-1}(t))^\top w_{m+1}^\star \\
& \leq \eta \exp\left(-c(\gamma^{-1}(t)-s_\eta)\right) \\
& \leq \exp\left(-c \underline{D}(t-\gamma(s_\eta))\right),
\end{align*}
which concludes the proof. \qed

\section{Proof of Section~\texorpdfstring{\ref{sec:main}}{3.2}}\label{app:mainproof}

In this whole section, we consider Assumption~\ref{ass:orthofeatures} and $\bW^\circ$ a solution of the differential inclusion \eqref{eq:oneneurondyn}.

\subsection{Additional notations and preliminaries}\label{app:preliempirical}
Besides the quantities introduced in Appendix~\ref{app:prelipopulation}, we introduce here other useful notions.

For a random variable $Y$ and $\alpha\in\{1,2\}$, we use the standard notation $\|Y\|_{\psi_\alpha}$ for its $\alpha$-Orlicz norm, which is defined by
\begin{equation*}
\|Y\|_{\psi_\alpha} = \inf \{t> 0 \mid \bE[\exp(|Y|^\alpha/t^\alpha)]\leq 2 \}.
\end{equation*}
For a random vector $X$ in $\R^d$, we also define its Orlicz norm as $\|X\|_{\psi_\alpha} = \sup_{a\in\bS_{d-1}} \|a^\top X\|_{\psi_\alpha}$.

We define for any $w\in\bS_{d-1}$ the following empirical matrices/vectors, given by their component for any $i\in[n],j\in[m]$:
\begin{equation*}
(F_n)_{ij} = \sigma(x_i^\top w_{j}^\star),\qquad Z_n(w)_i = \sigma(x_i^\top w),
\end{equation*}
so that $F_n\in\R^{n\times m}$ and $Z_n(w)\in\R^{n}$. In the whole proof, we will assume $F_n^\top F_n$ is invertible, which is guaranteed with high probability as soon as $n\gtrsim m$.
\begin{assumption}\label{ass:Finvertible}
$F_n^\top F_n$ is invertible.
\end{assumption}
Whenever Assumption~\ref{ass:Finvertible} holds, we define the projection matrix $\Pi_n = \id[m]-F_n\left(F_n^\top F_n\right)^{-1}F_n^\top$, which is the orthogonal projection onto $\{Z_n(w^\star_{1}), \ldots, Z_n(w^\star_{m})\}^\perp$. 
We also define the following key quantities of the dynamics for any $w\in\bS_{d-1}$:
\begin{equation} \label{eq:HDn}
H_n(w)=\frac{1}{n}Z_n(w_{m+1}^\star)^\top \Pi_n Z_n(w),\qquad D_n(w) = \frac{1}{n}Z_n(w)^\top \Pi_n Z_n(w).
\end{equation}

\subsection{Empirical loss: expression and gradient}

\begin{lemma}\label{lemma:empiricalloss}
Let Assumption~\ref{ass:Finvertible} hold and $w_{m+1}\in\bS_{d-1}$ such that $D_n(w_{m+1})>0$. Then, noting $W=[w_1^{\star\,\top}, \ldots, w_m^{\star\,\top}, w_{m+1}^{\top}]$, $a_n(W)$ is uniquely defined and satisfies
\begin{equation*}
\begin{cases}	a_n(W)_{m+1} = \frac{H_n(w_{m+1})}{D_n(w_{m+1})}\\
a_n(W)_{1:m} = \mathbf{1}_m + (F_n^\top F_n)^{-1}\left(F_n^\top Z_n(w_{m+1}^\star)-\frac{H_n(w_{m+1})}{D_n(w_{m+1})}F_n^\top Z_n(w_{m+1})\right).
\end{cases}
\end{equation*}
Moreover, the loss and its subdifferential satisfy
\begin{gather*}
\cL_n(a_n(W),W) = \left(D_n(w_{m+1}^\star)-\frac{H_n(w_{m+1})^2}{D_n(w_{m+1})}\right),\\
\partial_{w_{m+1}}\cL_n(a_n(W),W) \subseteq -\frac{1}{n}\frac{H_n(w_{m+1})}{D_n(w_{m+1})}  J_n(w_{m+1})^\top\Pi_n \left(Z_n(w_{m+1}^\star) - \frac{H_n(w_{m+1})}{D_n(w_{m+1})}  Z_n(w_{m+1})\right),
\end{gather*}
where $J_n(w) = \prod_{i=1}^n J_n^{(i)}(w) \subset \R^{n\times d}$ and $J_n^{(i)}(w) = \partial\sigma(w^\top x_i) x_i \subset \R^{d}$ for each $i\in[n]$.
\end{lemma}

\begin{proof}
The proof here follows similar lines as the one of Lemma~\ref{lemma:populationloss}, so that we sometimes omit a few computations.

\textbf{Computation of $a_n(W)$. }
By definition, $a_n(W)$ minimizes the optimization problem:
\begin{equation*}
\min_{a\in\R^{m+1}} \|F_n a_{1:m}+a_{m+1}Z_n(w_{m+1})-Y\|^2,
\end{equation*}
where $Y=F_n\mathbf{1}_m + Z_n(w_{m+1}^\star)$. For a given $a_{m+1}$, it admits a unique minimizer on its first coordinates, given by $a_{1:m}= (F_n^\top F_n)^{-1}F_n^\top\left(Y - a_{m+1}Z(w_{m+1})\right)$ and reaches the value:
\begin{equation*}
\|\Pi_n\left(Z_n(w_{m+1}^\star) - a_{m+1}Z_n(w_{m+1})\right)\|^2=n\left( a_{m+1}^2 D_n(w_{m+1}) - 2a_{m+1}H_n(w_{m+1}) + D_n(w_{m+1}^\star) \right).
\end{equation*}
This value is then minimal for a unique value of $a_{m+1}$ whenever $D_n(w_{m+1})>0$, which is given by 
\begin{equation*}
a_{m+1} = \frac{H_n(w_{m+1})}{D_n(w_{m+1})}.
\end{equation*}

\textbf{Empirical Loss. }
It then comes the value of $\cL_n$ at this point:
\begin{align*}
\cL_n(a_n(W), W) &= \left( D_n(w_{m+1}^\star) - \frac{H_n(w_{m+1})^2}{D_n(w_{m+1})}\right).
\end{align*}
Moreover, using classical properties of the Clarke subdifferential \citep[see e.g.,][]{bolte2021conservative}, it comes by definition that\footnote{Using the envelope theorem is also possible here, but requires a more advanced version that extends to non-differentiable functions. We rely on simpler arguments here.}
\begin{align*}
\partial_{w_{m+1}}\cL_n(a_n(W),W) &\subseteq \frac{a_{n}(W)_{m+1}}{n} \sum_{i=1}^n (f_{a_n(W),W}(x_k) - y_k) \partial\sigma(w_{m+1}^\top x_k)x_k \\
&=  \frac{a_{n}(W)_{m+1}}{n} J_n(w_{m+1})^\top \left(F_n a_n(W)_{1:m} + a_n(W)_{m+1}Z_n(w_{m+1}) - Y\right).
\end{align*}
Moreover, as seen above by optimality of $a_n(W)$,  
\begin{equation*}
F_n a_n(W)_{1:m} + a_n(W)_{m+1}Z_n(w_{m+1}) - Y = a_n(W)_{m+1} \Pi_n Z_n(w_{m+1}) - \Pi_n Z_n(w_{m+1}^\star),
\end{equation*}
which then allows to conclude with the equality $a_n(W)_{m+1}=\frac{H_n(w_{m+1})}{D_n(w_{m+1})}$.
\end{proof}
In the following, we use for shortness for any $w\in\bS_{d-1}$ the notation:
\begin{equation*}
\Psi_n(w) \coloneqq \frac{1}{n}J_n(w)^\top\Pi_n \left(Z_n(w^\star_{m+1}) - \frac{H_n(w)}{D_n(w)}  Z_n(w)\right),
\end{equation*}
so that $\partial_{w_{m+1}}\cL_n(a_n(W),W) \subseteq -\frac{H_n(w_{m+1})}{D_n(w_{m+1})}\Psi_n(w_{m+1})$.

\begin{corollary}If Assumption~\ref{ass:Finvertible} holds, then for almost any $t\geq 0$, 
\begin{equation*}
\frac{\df}{\df t} \bw_{m+1}^{\circ\,\top}w_{m+1}^{\star} \in \frac{H_n(\bw_{m+1}^{\circ})}{D_n(\bw_{m+1}^{\circ})} w_{m+1}^{\star\,\top}P_{\bw_{m+1}^{\circ}(t)^\perp} \Psi_n(\bw_{m+1}^{\circ}(t)).
\end{equation*}
Moreover, we can also control the component of $\bw_{m+1}^{\circ}$ along the first $m$ features for almost any $t\geq 0$:
\begin{equation*}
\frac{\df}{\df t}\|P_m \bw_{m+1}^{\circ}(t)\| \leq \frac{H_n(\bw_{m+1}^{\circ})}{D_n(\bw_{m+1}^{\circ})} \|P_m P_{\bw_{m+1}^{\circ}(t)^\perp} \Psi_n(\bw_{m+1}^{\circ}(t))\|.
\end{equation*}
\end{corollary}
\begin{proof}
The proof is similar to the one of Corollary~\ref{coro:populationODE}.
\end{proof}

\subsection{Concentration of key quantities}\label{app:concetration}

In this section, we show that the key quantities appearing in the dynamics concentrate towards their population versions, given in Appendix~\ref{app:population}. The main result of this section is given by Proposition~\ref{prop:empiricalrates} below.
\begin{restatable}{proposition}{propconc}\label{prop:empiricalrates} Let $\varepsilon\in(0,1)$ and Assumption~\ref{ass:orthofeatures} holds. There exist universal positive constants $C, c, \underline{D}, \overline{D},m_0$ such that if $m\geq m_0$ and $n \geq \frac{C}{\varepsilon^2}\ln(1/\varepsilon)^2\left( d +\ln(1/\delta)\right)$, then with probability at least $1-\delta$, Assumption~\ref{ass:Finvertible} holds and for any $w\in\cT(c)$:
\begin{enumerate}
\item $\underline{D}\leq D_n(w) \leq \overline{D}$;
\item $\inf_{\psi \in \Psi_n(w)}w_{m+1}^{\star\,\top}P_{w^\perp} \psi \geq c (1-w^\top w_{m+1}^\star)$;
\item $\sup_{\psi \in \Psi_n(w)}\|P_m P_{w^\perp} \psi\| \leq C \left(\frac{1}{\sqrt{m}}+\|P_m w\|\right) +\varepsilon$.
\end{enumerate}\end{restatable} 
The hardest item of Proposition~\ref{prop:empiricalrates} to prove is the second one. In that goal, we rely on distinguishing two cases:
\begin{itemize}[leftmargin=20pt]
\item for $w$ bounded away from $w_{m+1}^\star$, we rely on uniform concentration bounds of the different quantities of interest (see Appendix~\ref{app:uniform_concentration}). These bounds are also sufficient to yield the first and third items of Proposition~\ref{prop:empiricalrates}.
\item When $w$ is close to $w_{m+1}^\star$, the term $(1-w^\top w_{m+1}^\star)$ vanishes, so that uniform concentration bounds fail at providing the desired bound. Instead, we rely on different bounding techniques (Appendix~\ref{app:neighborconc}), at the neighborhood of $w_{m+1}^\star$, using concentration bounds developed by \citet{soltanolkotabi2017learning}.
\end{itemize}

\subsubsection{Uniform concentration bounds}\label{app:uniform_concentration}
To prove Proposition~\ref{prop:empiricalrates}, we first rely on general concentrations of empirical processes given in Appendix~\ref{app:generalconcentration}, which yields the uniform concentration of different intermediate quantities as stated in Lemma~\ref{lemma:concentration}. 
\begin{lemma}\label{lemma:concentration}
Consider Assumption~\ref{ass:orthofeatures}. There exists a universal constant $C>0$ such that for any $\varepsilon,\delta\in(0,1/2)$, if 
\begin{equation*}
n \geq \frac{C}{\varepsilon^2}\ln(1/\varepsilon)^2\left( d +\ln(1/\delta)\right),
\end{equation*}
then with probability at least $1-\delta$, the following statements hold simultaneously:
\begin{enumerate}
\item $\|G^{-1/2}\left(\frac{1}{n}F_n^\top F_n\right)G^{-1/2} - \id[m] \|_{\op} \leq \varepsilon$;
\item $\sup_{w,v\in\bS_{d-1}} \left|\frac{1}{n} Z_n(w)^\top Z_n(v) - \kappa(w^\top v)\right| \leq \varepsilon$;
\item $\sup_{w\in\bS_{d-1}}\left\|G^{-1/2}\left( \frac{1}{n}F_n^\top Z_n(w) - \Phi^* \phi_{w}\right)\right\|_2\leq \varepsilon $;
\item $\sup_{w,v,u\in\bS_{d-1}} \big|B_n(w,v,u) - B(w,v,u)\big| \leq \varepsilon$;
\item $\sup_{w,v\in\bS_{d-1}}\left\|G^{-1/2}\left( b_n(w,v) - b(w,v)\right)\right\|_2\leq \varepsilon$;
\end{enumerate}
where $B_n,B : \bS_{d-1}^3 \to \R$ and $b_n,b : \bS_{d-1}^2 \to \R^m$ are defined as follows
\begin{gather*}
B_n(w,v,u) \coloneqq \frac{1}{n}\sum_{k=1}^n \iind{w^\top x_k>0} v^\top x_k \sigma(u^\top x_k) ; \quad B(w,v,u) \coloneqq \bE_{x\sim\cN(0,\id)}[\iind{w^\top x>0} v^\top x \sigma(u^\top x)];\\
b_n(w,v)_i \coloneqq B_n(w,v,w_i^\star); \quad b(w,v)_{i} \coloneqq B(w,v,w_i^\star).
\end{gather*}
\end{lemma}
\begin{proof}
In the whole proof, we use the notations $\bS_{G} \coloneqq \{a \in\R^d \mid G^{1/2}a \in\bS_{m-1}\}$, $\cH_w \coloneqq \{x\in\R^d \mid x^\top w \geq 0\}$ for any $w\in\R^d$, and we define the function:
\begin{equation}\label{eq:functionh}
h : \R^d \ni x \longmapsto \begin{pmatrix}
\sigma(x^\top w_{1}^\star) \\
\vdots \\
\sigma(x^\top w_{m}^\star)
\end{pmatrix} \in \R^{m}.
\end{equation}

\textbf{1.} For the first item, it comes by definition of the mapping $\Phi:\R^m\to L^2(\cN(0,\id))$ that
\begin{align*}
\|G^{-1/2}\left(\frac{1}{n}F_n^\top F_n\right)G^{-1/2} - \id[m] \|_{\op} & = \sup_{a\in \bS_{m-1}}a^\top\left(G^{-1/2}\left(\frac{1}{n}F_n^\top F_n\right)G^{-1/2} - \id[m]  \right) a \\
&  = \sup_{a\in \bS_{G}}a^\top\left(\frac{1}{n}F_n^\top F_n - G  \right) a \\
& = \sup_{a\in \bS_{G}} a^\top \frac{1}{n}\sum_{k=1}^n h(x_k) h(x_k)^\top - \bE[h(x)h(x)^\top] a \\
& \leq \|G^{-1}\|_{\op}\sup_{a\in\bS_{d-1}} a^\top \left(\frac{1}{n}\sum_{k=1}^n h(x_k) h(x_k)^\top - \bE[h(x)h(x)^\top]\right) a .
\end{align*}
First, $\|G^{-1}\|_{\op} \leq \frac{2\pi}{\pi-1}$, thanks to Assumption~\ref{ass:orthofeatures}.
Item 1 is then a direct application of \footnote{We could here use a simpler, covariance estimation bound.} Lemma~\ref{lemma:chainingVC}, with $h_1=h_2=h$ and $\cA=\{\R^d\}$.

The different assumptions are indeed satisfied here, as the VC dimension of $\cA$ is $2$ here and $h(X)$ is $C$-sub-Gaussian for some universal constant $C$, thanks to the independence of its coordinates.

\textbf{2.} Note that for any $w,v\in\bS_{d-1}$,
\begin{align*}
\left|\frac{1}{n} Z_n(w)^\top Z_n(v) - \kappa(w^\top v)\right|  &= \left|\frac{1}{n} \sum_{k=1}^n \sigma(w^\top x_k)\sigma(v^\top x_k) - \bE[ \sigma(w^\top X)\sigma(v^\top X)]\right|\\
& = \left|\frac{1}{n} \sum_{k=1}^n (w^\top x_k)(v^\top x_k)\iind{x_k \in \cH_w\cap\cH_v} - \bE[ (w^\top X)(v^\top X)\iind{X \in \cH_w\cap\cH_v}]\right|.
\end{align*}
This point is thus a direct application of Lemma~\ref{lemma:chainingVC} with $\cA = \{\cH_w \cap \cH_v \mid v,w\in\bS_{d-1}\}$ and $h_1=h_2=\id$. Classical VC bounds yield that $\cA$ has VC dimension at most $2(d+2)$ \citep[see e.g.,][Chapter 2.6, Exercise 14 and Lemma 2.6.17.ii.]{vaart1996weak}. The sub-Gaussian assumption is immediate.

\textbf{3.} Here, for any $w\in\bS_{d-1}$, we can write:
\begin{align*}
\left\|G^{-1/2}\left( \frac{1}{n}F_n^\top Z_n(w) - \Phi^* \phi_{w}\right)\right\|_2 & = \sup_{a\in\bS_{m-1}} a^\top G^{-1/2}\left( \frac{1}{n}F_n^\top Z_n(w) - \Phi^* \phi_{w}\right) \\
& = \sup_{a\in\bS_{G}} a^\top \left( \frac{1}{n}F_n^\top Z_n(w) - \Phi^* \phi_{w}\right)\\
& = \sup_{a\in\bS_{G}} \left| \frac{1}{n} \sum_{k=1}^n a^\top h(x_k) \sigma(w^\top x_k) - \bE[a^\top h(X)\sigma(w^\top X)] \right| \\
&\leq \|G^{-1/2}\|_{\op}\sup_{a\in\bS_{d-1}} \left| \frac{1}{n} \sum_{k=1}^n a^\top h(x_k) \sigma(w^\top x_k) - \bE[a^\top h(X)\sigma(w^\top X)] \right|
\end{align*}
Again, the third point is a direct application of Lemma~\ref{lemma:chainingVC} with $\cA=\{\cH_w \mid w\in\bS_{d-1}\}$, $h_1=h$ and $h_2=\id$. Here again, a classical VC bound yields that  $\cA$ has VC dimension at most $d+2$ \citep[][Chapter 2.6, exercise 14]{vaart1996weak}.

\textbf{4.} It is again a direct application of Lemma~\ref{lemma:chainingVC}, with $\cA = \{\cH_w \cap \cH_u \mid u,w\in\bS_{d-1}\}$ and $h_1=h_2=\id$.

\textbf{5.} Again note that for any $w,v\in\bS_{d-1}$,
\begin{align*}
\left\|G^{-1/2}\left( b_n(w,v) - b(w,v)\right)\right\|_2 & = \sup_{a\in\bS_G} a^\top\left( b_n(w,v) - b(w,v)\right) \\
& = \sup_{a\in\bS_G} \left|\frac{1}{n}\sum_{k=1}^n \iind{w^\top x_k>0} a^\top h(x_k) x_k^\top v  -  \bE[\iind{w^\top X>0} a^\top h(X) X^\top v  ] \right|\\
& \leq \|G^{-1/2}\|_{\op} \sup_{a\in\bS_{d-1}} \left|\frac{1}{n}\sum_{k=1}^n \iind{w^\top x_k>0} a^\top h(x_k) x_k^\top v  -  \bE[\iind{w^\top X>0} a^\top h(X) X^\top v  ] \right|.
\end{align*}
It is again a consequence of Lemma~\ref{lemma:chainingVC} with  $\cA=\{\cH_w \mid w\in\bS_{d-1}\}$, $h_1=h$ and $h_2=\id$.
\end{proof}
\begin{lemma} \label{lemma:concHD}
Consider Assumption~\ref{ass:orthofeatures} and the items 1. to 5. of Lemma~\ref{lemma:concentration} all hold with $\varepsilon\in(0,1/2]$. Then $F_n^\top F_n$ is invertible and there is a universal constant $C$ such that for any $w\in\bS_{d-1}$,
\begin{enumerate}
\item $|D_n(w) - D(w)| \leq C \varepsilon$;
\item $|H_n(w) - H(w)| \leq C \varepsilon$;
\item $\sup_{J\in J_n(w)} \|\frac{1}{n}J^\top \Pi_n Z_n(w) - \frac{1}{2}\nabla \widetilde{D}(w)\| \leq C \varepsilon$;
\item $\sup_{J\in J_n(w)} \|\frac{1}{n}J^\top \Pi_n Z_n(w_{m+1}^\star) - \nabla \widetilde{H}(w)\| \leq C \varepsilon$;
\end{enumerate}
where $\widetilde{D}$ and $\widetilde{H}$ are defined in Equation~\eqref{eq:tildedefn}.
\end{lemma}
\begin{proof}
Let us write for this proof $G_n = \frac{1}{n}F_n^\top F_n$.
The first item of Lemma~\ref{lemma:concentration} implies that 
\begin{align*}
 \frac{1}{2}G  \preceq G_n \preceq \frac{3}{2}G.
\end{align*}
In particular, $G_n$ is invertible, and $\|G_n^{-1}\|_{\op}\leq \frac{4\pi}{\pi-2}$. Moreover, the eigenvalues of $G^{-1/2}G_n G^{-1/2}$ are all within $[1-\varepsilon, 1+\varepsilon]$, so that
\begin{align}\label{eq:concGminus}
\|G^{1/2}G_n^{-1} G^{1/2} - \id[m]\|_{\op} \leq \frac{\varepsilon}{1-\varepsilon}\leq 2\varepsilon.
\end{align}
Let us now prove the different items of Lemma~\ref{lemma:concHD}. Again, we use in this proof the notation  $\bS_{G} = \{a \in\R^d \mid G^{1/2}a \in\bS_{m-1}\}$.

\textbf{1.} By definition of both $D_n$ and $D$, it comes for any $w\in\bS_{d-1}$:
\begin{align*}
D_n(w) - D(w) = \left(\frac{1}{n}Z_n(w)^\top Z_n(w) - \frac{1}{2}\right) + \Phi^*\phi_w^\top G^{-1}\Phi^*\phi_w - \left(\frac{1}{n}F_n^\top Z_n(w)\right)^\top G_n^{-1}\left(\frac{1}{n}F_n^\top Z_n(w)\right).
\end{align*}
The second item of Lemma~\ref{lemma:concentration} directly yields $|\frac{1}{n}Z_n(w)^\top Z_n(w) - \frac{1}{2}|\leq \varepsilon$. It remains to bound the second term. For that, denote here for shortness $u = \Phi^*\phi_w$ and $u_n = \frac{1}{n}F_n^\top Z_n(w)$. Note that the third item of Lemma~\ref{lemma:concentration} implies that $\|G^{-1/2}(u-u_n)\|\leq\varepsilon$. Moreover, it comes that 
\begin{align*}
\|G^{-1/2}u\|_2 & = \|G^{-1/2}\Phi^*\phi_w\| \\
& = \sup_{s\in\bS_{m-1}} (G^{-1/2}s)^\top \Phi^*\phi_w \\
& = \sup_{s\in\bS_{m-1}} \langle \Phi(G^{-1/2}s), \phi_w\rangle \\
& \leq \sup_{s\in\bS_{m-1}}  \|\Phi(G^{-1/2}s)\|_{L^2(\cN(0,\id))} \|\phi_w\|_{L^2(\cN(0,\id))}.
\end{align*}
From there, observe that, by definition, $\|\Phi(G^{-1/2}s)\|_{L^2(\cN(0,\id)}^2 = \|s\|_2$ and $\|\phi_w\|_{L^2(\cN(0,\id))}\leq \frac{1}{\sqrt{2}}$ thanks to the arc-cosine kernel. In consequence, 
\begin{equation}\label{eq:normbounda}
\|G^{-1/2}u\|_2 = \|G^{-1/2}\Phi^*\phi_w\|_2 \leq \frac{1}{\sqrt{2}}.
\end{equation}
Using the third item of Lemma~\ref{lemma:concentration} and Equation~\eqref{eq:concGminus} then yields
\begin{align*}
|u^\top G^{-1}u - u_n^\top G_n^{-1} u_n| & \leq |u^\top G^{-1} (u-u_n)| + |(u-u_n)^\top G^{-1} u_n | + |u_n (G^{-1}-G_n^{-1})u_n|\\
& \leq \|G^{-1/2}u\|\cdot \|G^{-1/2}(u-u_n)\| + \|G^{-1/2}u_n\|\cdot \|G^{-1/2}(u-u_n)\| \\& \phantom{\leq} + \|G^{-1/2}u_n\|^2 \cdot \|\id[m]-G^{1/2}G_n^{-1}G^{1/2}\|_{\op}\\
& \leq \frac{1}{\sqrt{2}} \varepsilon + (\varepsilon+\frac{1}{\sqrt{2}})\varepsilon +  (\varepsilon+\frac{1}{\sqrt{2}})^2 2\varepsilon \\
& \leq 2(1+\sqrt{2})\varepsilon.
\end{align*}
So that finally, for any $w\in\bS_{d-1}$, $|D_n(w)-D(w)| \leq (3+2\sqrt{2})\varepsilon$.

\textbf{2.} The proof of this item follows the exact same lines, since
\begin{align*}
H_n(w)- H(w) &= \frac{1}{n}Z_n(w_{m+1}^\star)^\top Z_n(w) - \langle \phi_{w_{m+1}^\star}, \phi_w\rangle \\&\phantom{=} + \Phi^*\phi_{w_{m+1}^\star}^\top G^{-1}\Phi^*\phi_w - \left(\frac{1}{n}F_n^\top Z_n(w_{m+1}^\star)\right)^\top G_n^{-1}\left(\frac{1}{n}F_n^\top Z_n(w)\right).
\end{align*}

\textbf{3.} Let $w\in\bS_{d-1}$ and assume first that the considered loss is differentiable at $w$, so that $J_n(w)$ reduces to a singleton, i.e., we can write by abuse of notation that $J_n(w)\in\R^{n\times d}$ with 
\begin{align*}
J_n(w)_k = \iind{w^\top x_k>0} x_k \qquad \text{for any }k\in[n].
\end{align*}
In that case, note that for any $v\in\bS_{d-1}$, and $b_n$ defined as in Lemma~\ref{lemma:concentration}
\begin{align*}
\frac{1}{n}v^\top J_n(w)^\top \Pi_n Z_n(w)  & = B_n(w,v,w) -  b_n(w,v)^\top G_n^{-1}\left(\frac{1}{n}F_n^\top Z_n(w)\right).
\end{align*}
The definition of $\widetilde{D}$ yields
\begin{align*}
\frac{1}{2}u^\top \nabla \widetilde{D}(w) & = B(w,u,w) - b(w,u)^\top G^{-1}\Phi^*\phi_w.  
\end{align*}
In consequence, it comes
\begin{align*}
\|\frac{1}{n}J_n(w)^\top \Pi_n Z_n(w) - \frac{1}{2}\nabla \widetilde{D}(w)\| & = \sup_{v\in\bS_{d-1}} v^\top\left(\frac{1}{n}J_n(w)^\top \Pi_n Z_n(w) - \frac{1}{2}\nabla \widetilde{D}(w) \right) \\
& =  B_n(w,v,w) - B(w,v,w) \\&\phantom{=}+ b(w,v)^\top G^{-1} \Phi^*\phi_w -  b_n(w,v)^\top G_n^{-1}\left(\frac{1}{n}F_n^\top Z_n(w)\right).
\end{align*}
Similarly to the first item of Lemma~\ref{lemma:concHD}, we can bound these different terms. For that, first observe by a similar argument than above that $\|G^{-1/2}b(w,v)\| \leq \frac{1}{\sqrt{2}}$. Indeed,
\begin{align*}
\|G^{-1/2}b(w,v)\| & = \sup_{a\in\bS_G} \bE[\iind{w^\top x>0}v^\top x \Phi(a)(x)] \\
&\leq \sup_{a\in\bS_G} \sqrt{\bE[\iind{w^\top x>0}(v^\top x)^2]} \|\Phi(a)\|_{L^2(\cN(0,\id)}\\
&= \frac{1}{\sqrt{2}},
\end{align*}
where we used that $\bE[\iind{w^\top x>0}(v^\top x)^2]=\frac{1}{2}$ by symmetry, and again that $\|\Phi(a)\|_{L^2(\cN(0,\id)}=1$ for $a\in\bS_G$.

We can now use the items 1, 3, 4 and 5 of Lemma~\ref{lemma:concentration}, similarly to the first point, to conclude with a similar bound on $\|\frac{1}{n}J_n(w)^\top \Pi_n Z_n(w) - \frac{1}{2}\nabla \widetilde{D}(w)\|$. 

We have thus shown for almost any $w\in \bS_{d-1}$ (i.e., at differentiable points), the inequality $\|\frac{1}{n}J_n(w)^\top \Pi_n Z_n(w) - \frac{1}{2}\nabla \widetilde{D}(w)\| \leq C \varepsilon$. 
By convexity of the function $J\mapsto \|\frac{1}{n}J^\top \Pi_n Z_n(v) - \frac{1}{2}\nabla \widetilde{D}(v)\|$, continuity of $Z_n$, $\nabla \widetilde{D}$ and as $J_n(w)$ is the convex hull of limits $J_n(w_k)$ for differentiable points $w_k\to w$, the inequality actually holds for any $w\in\bS_{d-1}$.

\textbf{4.} The proof of this item follows the exact same lines as the third item, since at any differentiability point $w$:
\begin{align*}
\|\frac{1}{n}J_n(w)^\top \Pi_n Z_n(w_{m+1}^\star) - \nabla \widetilde{H}(w)\| & = \sup_{v\in\bS_{d-1}} B_n(w,v,w_{m+1}^\star) - B(w,v,w_{m+1}^\star) \\&\phantom{=}+ b(w,v)^\top G^{-1} \Phi^*\phi_{w_{m+1}^\star} -  b_n(w,v)^\top G_n^{-1}\left(\frac{1}{n}F_n^\top Z_n(w_{m+1}^\star)\right).
\end{align*}
\end{proof}
\subsubsection{Concentration near \texorpdfstring{$w_{m+1}^\star$}{optimal weight}}\label{app:neighborconc}

\begin{lemma}\label{lemma:P2conc1}
Let $\delta\in(0,1/2)$ and Assumption~\ref{ass:orthofeatures} holds. There exist positive universal constants $c$ and $C$ such that, if $n\geq C(d+\ln(1/\delta))$, then with probability at least $1-\delta$, the following hold simultaneously:
\begin{enumerate}
\item $\|\mathbf{X}\|_{\op} \leq C \sqrt{n}$;
\item $\frac{1}{n}\|\Pi_n Z_n(w_{m+1}^\star)\|^2 \geq c$;
\item $\frac{1}{n}\|P_{n,w_{m+1}^\star} J_n(w_{m+1}^\star) q\|^2\geq c\|q\|^2$ for any $q\in \{w_{m+1}^\star\}^\perp$;
\end{enumerate}
where $P_{n,w_{m+1}^\star} \coloneqq \Pi_n - \frac{\Pi_n Z_n(w_{m+1}^\star)Z_n(w_{m+1}^\star)^\top \Pi_n}{Z_n(w_{m+1}^\star)^\top\Pi_n Z_n(w_{m+1}^\star)}$ is the orthogonal projection on the orthogonal of the column space of $[F_n, Z_n(w_{m+1}^\star)]$; and $\mathbf{X} = [x_1^\top, \ldots, x_n^\top]^\top \in \R^{n\times d}$ is the data matrix.
\end{lemma}
Note that $J_n(w_{m+1}^\star)$ is almost surely a singleton, so that, by abuse of notation, point 3. is stated for the single element in $J_n(w_{m+1}^\star)$.
\begin{proof}
\textbf{1. Bound on $\|X\|_{\op}$.} This is a classical concentration bound for random matrices \citep[see e.g.,][Theorem 4.4.5]{vershynin2018high}.

\textbf{3. Bound on $\|P_{n,w_{m+1}^\star} J_n(w_{m+1}^\star) q\|$.} Let $(u_1, \ldots, u_{d-1})$ be an orthonormal basis of $\{w_{m+1}^\star\}^\perp$ and define $U = [u_{1}^\top,\ldots, u_{d-1}^\top]\in \R^{d\times (d-1)}$. Without loss of generality, we can even assume that the first $m$ vectors of the basis are exactly $(w_1^\star, \ldots, w_m^\star)$, i.e., $u_k = w_k^{\star}$ for $k\leq m$. Now define for the remaining of the proof the function:
\begin{equation*}
h : \R^d \ni x \longmapsto \begin{pmatrix}
\sigma(x^\top w_{1}^\star) \\
\vdots \\
\sigma(x^\top w_{m}^\star)\\
\sigma(x^\top w_{m+1}^\star)\\
\iind{x^\top w_{m+1}^\star > 0} U^\top x
\end{pmatrix} \in \R^{m+d}.
\end{equation*}
First note that as $X\sim \cN(0,\id)$, $h(X)$ is $C$-sub-Gaussian for some universal constant $C$. Indeed, for any $a\in\R^{m+d}$,
\begin{align*}
\|a^\top h(X)\|_{\psi_2} & \leq \|a_{1:m+1}^\top h(X)_{1:m+1}\|_{\psi_2} +  \|a_{m+2:m+d}^\top h(X)_{m+2:m+d}\|_{\psi_2}.
\end{align*}
Moreover, thanks to Assumption~\ref{ass:orthofeatures}, the different coordinates of $h(X)_{1:m+1}$ are independent and identically distributed\footnote{And each coordinate is $1$-sub-Gaussian.}, so that $\|a_{1:m+1}^\top h(X)_{1:m+1}\|_{\psi_2} \leq C\|a_{1:m+1}\|$ for some universal constant. The same reasoning holds for the second term, so that $\|a^\top h(X)\|_{\psi_2}\leq C\|a\|$. 

Moreover, the arc-cosine kernel with Assumption~\ref{ass:orthofeatures} provides the following block structure:
\begin{gather*}
M \coloneqq \bE[h(X) h(X)^\top] =
\left(
\begin{array}{c|c}
\frac{\pi-1}{2\pi}\id[m+1] + \frac{1}{2\pi}\mathbf{1}_{m+1} \mathbf{1}_{m+1}^\top & \frac{1}{4}\tI_m \\[6pt]
\hline\\[-6pt]
\frac{1}{4}\tI_m^\top & \frac{1}{2}\id[d-1]
\end{array}
\right),\\
\text{where} \qquad \tI_{ij} = \begin{cases}1 \text{ if } i=j\leq m\\0 \text{ otherwise.}\end{cases}
\end{gather*}
In particular, it implies that for any $a\in\R_{m+d}$:
\begin{align*}
a^\top M a& \geq \frac{\pi-1}{2\pi} \|a_{1:m+1}\|^2 + \frac{1}{2} \|a_{m+2:m+d}\|^2 - \frac{1}{2}\|a_{1:m+1}\|\|a_{m+2:m+d}\|\\
&\geq \left(\frac{\pi-1}{2\pi}-\frac{1}{4}\right) \|a_{1:m+1}\|^2 + \left(\frac{1}{2}-\frac{1}{4}\right) \|a_{m+2:m+d}\|^2 \\
& \geq 0.09 \|a\|^2,
\end{align*}
i.e. the smallest eigenvalue of $M$ is larger than $0.09$---which is a positive universal constant.

From the above sub-Gaussian property and smallest eigenvalue bound on $M$, it thus implies that there exists a universal constant $C$, such that for any $a\in\R_{m+d}$, $\|a^\top h(X)\|_{\psi_2}^2\leq C a^\top M a$. 
A typical covariance concentration bound \citep[see e.g.,][Theorem 4.7.1 and Exercise 4.7.3]{vershynin2018high} implies if $n\geq C(m+d+\ln(1/\delta))$ that with probability at least $1-\delta$,
\begin{equation*}
\left\| \frac{1}{n}\sum_{i=1}^n M^{-1/2} h(x_i) h(x_i)^\top M^{-1/2}- \id[m+d] \right\|_{\op} \leq \frac{1}{2}.
\end{equation*}
Assume for the remaining of the proof that this event holds\footnote{By Assumption~\ref{ass:orthofeatures}, $m\leq d$.}. In particular, it implies for any $a\in\R_{m+d}$:
\begin{equation}\label{eq:conccovariance}
\frac{1}{n}\sum_{i=1}^n (a^\top h(x_i))^2 \geq \frac{1}{2}a^\top M a \geq c \|a\|^2,
\end{equation}
for some universal constant $c>0$. 
Let $q\in\{w_{m+1}^\star\}^\perp$. By definition of $U$, we can consider $s\in\R^{d-1}$ such that $Us = q$ and $\|s\|=\|q\|$. From there,
observe that
\begin{align*}
\frac{1}{n}\|P_{n,w_{m+1}^\star} J_n(w_{m+1}^\star) q\|^2 & = \inf_{\alpha\in\R^{m+1}}\frac{1}{n} \left\|J_n(w_{m+1}^\star) q - \sum_{k=1}^{m+1} \alpha_k Z_n(w_{k}^\star)\right\|^2 \\
& = \inf_{\alpha\in\R^{m+1}} \frac{1}{n}\sum_{i=1}^n \left(\iind{x_i^\top w_{m+1}^\star>0} x_i^\top q - \alpha^\top [\sigma(x_i^\top w_1^\star), \ldots, \sigma(x_i^\top w_{m+1}^\star)] \right)^2\\
& = \inf_{\alpha\in\R^{m+1}}\frac{1}{n} \sum_{i=1}^n (h(x_i)^\top[-\alpha, s])^2\\
& \geq c\inf_{\alpha\in\R^{m+1}}\|[-\alpha, s]\|^2\\
& \geq c \|s\|^2.
\end{align*}
By definition of $s$, it finally yields the third point of Lemma~\ref{lemma:P2conc1}.

\textbf{2. Bound on $\|\Pi_n Z_n(w_{m+1}^\star)\|$.} Again, assuming that Equation~\eqref{eq:conccovariance} holds, it comes
\begin{align*}
\frac{1}{n}\|\Pi_n Z_n(w_{m+1}^\star)\|^2 & = \inf_{\alpha \in \R^m}\frac{1}{n} \|Z_n(w_{m+1}^\star) - \sum_{k=1}^m \alpha_k Z_n(w_{k}^\star) \|^2 \\
& =  \inf_{\alpha\in\R^{m}} \frac{1}{n}\sum_{i=1}^n \left(h(x_i)^\top[-\alpha, 1, \mathbf{0}_{d-1}]\right)^2\\
& \geq c\inf_{\alpha\in\R^{m}} \|[-\alpha, 1, \mathbf{0}_{d-1}]\|^2 \\
& \geq c.
\end{align*}
\end{proof}

\begin{lemma}\label{lemma:P2conc2}
Let $\delta\in(0,1/2)$ and $\varepsilon\in(0,1)$. There exist positive constants $\eta_\varepsilon$ and $C_\varepsilon$ such that, if $n\geq C_\varepsilon(d+\ln(1/\delta))$, then with probability at least $1-\delta$, uniformly for all $w\in\bS_{d-1}$ such that $1-w^\top w_{m+1}^\star \leq \eta_\varepsilon$,
\begin{enumerate}
\item $\displaystyle\frac{1}{n}\sum_{i=1}^n \iind{(x_i^\top w) (x_i^\top w_{m+1}^\star)\leq 0} \left(x_i^\top P_{w^\perp} w_{m+1}^\star\right)^2 \leq \varepsilon \| P_{w^\perp} w_{m+1}^\star\|^2$;
\item $\displaystyle\frac{1}{n}\sum_{i=1}^n \iind{(x_i^\top w) (x_i^\top w_{m+1}^\star)\leq 0} \left(x_i^\top  w_{m+1}^\star\right)^2 \leq \varepsilon \| P_{w^\perp} w_{m+1}^\star\|^2.$.
\end{enumerate}
\end{lemma}

\begin{proof}
Let $\eta\in(0,1/2)$ and $\varepsilon_1\in(0,1/2)$ to be fixed later. Lemma~\ref{lemma:solta} \citep[which is adapted from][Lemma 5.5]{soltanolkotabi2017learning} yields that, if $n\geq \frac{C}{\varepsilon_1}\left(d+\ln(1/\delta)\right)$ for some universal constant $C\geq 1$, then with probability at least $1-\delta$, uniformly over all $w\in\bS_{d-1}$ such that $\|P_{w_{m+1}^{\star\,\perp}}w\|\leq \sqrt{2\eta}$,
\begin{equation}\label{eq:solta}
\frac{1}{n}\sum_{i=1}^n \iind{\frac{1}{2}|x_i^\top w^\star_{m+1}|\leq |x_i^\top P_{w_{m+1}^{\star\,\perp}} w|} \left(x_i^\top P_{w_{m+1}^{\star\,\perp}} w\right)^2 \leq C\left(\varepsilon_1+\sqrt{\eta}\right) \|P_{w_{m+1}^{\star\,\perp}}w\|^2.
\end{equation}
Assume in the following of the proof that Equation~\eqref{eq:solta} holds for any $w\in\bS_{d-1}$ such that $\|P_{w_{m+1}^{\star\,\perp}}w\|\leq \sqrt{2\eta}$. 

Let now $w\in\bS_{d-1}$ with $w^\top w_{m+1}^\star\geq 1-\eta$. First notice that $\|P_{w_{m+1}^{\star\,\perp}}w\|\leq \sqrt{2\eta}$. 
Moreover, decomposing $w_{m+1}^\star = (w^\top w_{m+1}^\star)  w_{m+1}^\star + P_{ w_{m+1}^{\star\,\perp}}w$, observe that for any $x\in\R^d$
\begin{align}
\iind{(x^\top w) (x^\top w_{m+1}^\star)\leq 0} & \leq \iind{(x^\top  w_{m+1}^\star)^2 w^\top w_{m+1}^\star \leq - (x^\top  w_{m+1}^\star) (x^\top P_{ w_{m+1}^{\star\,\perp}}w)}\notag\\
& \leq  \iind{(w^\top w_{m+1}^\star) |x^\top w_{m+1}^\star | \leq |x^\top P_{ w_{m+1}^{\star\,\perp}}w|}\notag\\
& \leq  \iind{\frac{1}{2}|x^\top w_{m+1}^\star | \leq |x^\top P_{ w_{m+1}^{\star\,\perp}}w|}.\label{eq:inclevent1}
\end{align}
In particular, it yields for any $i\in[n]$
\begin{align*}
\iind{(x_i^\top w) (x_i^\top w_{m+1}^\star)\leq 0} \left(x_i^\top  w_{m+1}^\star\right)^2 & \leq \iind{\frac{1}{2}|x_i^\top w_{m+1}^\star | \leq |x_i^\top P_{ w_{m+1}^{\star\,\perp}}w|} \left(x_i^\top  w_{m+1}^\star\right)^2 \\
&\leq 4 \ \iind{\frac{1}{2}|x_i^\top w_{m+1}^\star | \leq |x_i^\top P_{ w_{m+1}^{\star\,\perp}}w|}\left(x_i^\top P_{w_{m+1}^{\star\,\perp}} w\right)^2,
\end{align*}
which yields, thanks to Equation~\eqref{eq:solta}, the second inequality in Lemma~\ref{lemma:P2conc2} when choosing $\eta\leq\left(\frac{\varepsilon}{8C}\right)^2$ and $\varepsilon_1\leq\frac{\varepsilon}{8C}$.

Moreover, note that
\begin{align*}
P_{w^\perp} w_{m+1}^\star & = w_{m+1}^\star - (w^\top w_{m+1}^\star)w\\
& =  w_{m+1}^\star - (w^\top w_{m+1}^\star)\left(P_{w_{m+1}^{\star\,\perp}} w +(w^\top w_{m+1}^\star)w_{m+1}^\star\right) \\
&= ( 1- (w^\top w_{m+1}^\star)^2) w_{m+1}^\star - (w^\top w_{m+1}^\star)P_{w_{m+1}^{\star\,\perp}} w.
\end{align*}
So that whenever $\frac{1}{2}|x^\top w_{m+1}^\star | \leq |x^\top P_{ w_{m+1}^{\star\,\perp}}w|$, we have $|x^\top P_{w^\perp} w_{m+1}^\star| \leq 3|x^\top P_{w_{m+1}^{\star\,\perp}} w|$. Along with Equation~\eqref{eq:inclevent1}, this yields for any $i\in[n]$
\begin{align*}
\left(x_i^\top P_{w^\perp} w_{m+1}^\star\right)^2 \iind{(x_i^\top w) (x_i^\top w_{m+1}^\star)\leq 0} \leq 9 \left(x_i^\top P_{w_{m+1}^{\star\,\perp}} w\right)^2\iind{\frac{1}{2}|x_i^\top w_{m+1}^\star | \leq |x_i^\top P_{ w_{m+1}^{\star\,\perp}}w|}.
\end{align*}
The first inequality of Lemma~\ref{lemma:P2conc2} is then also a consequence of Equation~\eqref{eq:solta}, when taking $\eta=\left(\frac{\varepsilon}{18C}\right)^2$ and $\varepsilon_1=\frac{\varepsilon}{18C}$.
\end{proof}

\begin{corollary}\label{coro:concentrationP2}
Assume that the concentration events of both Lemmas~\ref{lemma:P2conc1} and \ref{lemma:P2conc2} hold with parameters $\varepsilon, \eta_{\varepsilon}, c$ and $C$. We then have for any $w\in\bS_{d-1}$ such that $1-w^\top w_{m+1}^\star \leq \min\left(\eta_\varepsilon,\frac{1}{2}\right)$
\begin{enumerate}
\item $\frac{1}{n}\left\| P_{n,w_{m+1}^\star} J_n(w_{m+1}^\star) P_{w^\perp} w_{m+1}^\star \right\|^2\geq \frac{3c}{8} \left(1-w^\top w_{m+1}^\star\right)$;
\item $\|P_{n,w}-P_{n,w_{m+1}^\star}\|_{\op} \leq 4\sqrt{2}\,\frac{\sqrt{2c}+C}{c}\, C\sqrt{\left(1-w^\top w_{m+1}^\star\right)}$;
\end{enumerate}
and for any $J\in J_n(w)$:
\begin{enumerate}[start=3]
\item $\frac{1}{n} \|J P_{w^\perp} w_{m+1}^\star \|^2 \leq 2C^2 \left(1-w^\top w_{m+1}^\star\right)$;
\item $\frac{1}{n}\left\| \left(J-J_n(w_{m+1}^\star)\right)P_{w^\perp} w_{m+1}^\star \right\|^2 \leq 2\varepsilon \left(1-w^\top w_{m+1}^\star\right)$;
\item $\frac{1}{n} \left\|Z_n(w_{m+1}^\star) - w^\top w_{m+1}^\star Z_n(w) - JP_{w^\perp} w_{m+1}^\star\right\|^2 \leq 2\varepsilon \left(1-w^\top w_{m+1}^\star\right)$.
\end{enumerate}
\end{corollary}
\begin{proof}
In the whole proof, let $w\in\bS_{d-1}$ be such that $1-w^\top w_{m+1}^\star \leq \min\left(\eta_\varepsilon,\frac{1}{2}\right)$ and denote $u=P_{w^\perp} w_{m+1}^\star$ for shortness. We first recall the useful identity:
\begin{equation}\label{eq:projidentity}
u = P_{w^\perp} w_{m+1}^\star = \left(1-(w^\top w_{m+1}^\star)^2\right)w_{m+1}^\star - (w^\top w_{m+1}^\star)P_{w_{m+1}^{\star\,\perp}}w.
\end{equation}
In the remaining, we will also be using the fact that $\|u\|^2=1-(w^\top w_{m+1}^\star)^2 \leq 2 (1-w^\top w_{m+1}^\star)$.

\textbf{1.}  By definition, note that $J_n(w_{m+1}^\star) w_{m+1}^\star = Z_n(w_{m+1}^\star)$, so that $P_{n,w_{m+1}^\star}J_n(w_{m+1}^\star) w_{m+1}^\star = \mathbf{0}$. In consequence, using Equation~\eqref{eq:projidentity}
\begin{align*}
\frac{1}{n}	\left\|P_{n,w_{m+1}^\star} J_n(w_{m+1}^\star) u\right\|^2 & = (w^\top w_{m+1}^\star)^2 \frac{1}{n}\left\|P_{n,w_{m+1}^\star} J_n(w_{m+1}^\star) P_{w_{m+1}^{\star\,\perp}}w\right\|^2 \\
& \geq c (w^\top w_{m+1}^\star)^2 \|P_{w_{m+1}^{\star\,\perp}}w\|^2 \\
& \geq c (w^\top w_{m+1}^\star)^2 \left(1-(w^\top w_{m+1}^\star)^2\right) \\
& \geq \frac{3c}{8} \left(1-w^\top w_{m+1}^\star\right),
\end{align*}
where the second line is due to item 3. of Lemma~\ref{lemma:P2conc1}.

\textbf{2. } For shortness of notation, denote here the vectors $a = \frac{1}{\sqrt{n}}\Pi_n Z_n(w_{m+1}^\star)$ and $b = \frac{1}{\sqrt{n}}\Pi_n Z_n(w)$. By definition, $P_{n,w}-P_{n,w_{m+1}^\star}=\frac{aa^\top}{\|a\|^2} - \frac{bb^\top}{\|b\|^2}$. In consequence,
\begin{align}
\|P_{n,w}-P_{n,w_{m+1}^\star}\|_{\op} & = \left\|\frac{aa^\top}{\|a\|^2} - \frac{bb^\top}{\|b\|^2}\right\|_{\op} \notag\\
& \leq \frac{1}{\|a\|^2} \left( \|aa^\top - bb^\top\|_{\op}\right) + \|bb^\top\|_\op \left(\frac{1}{\|b\|^2}- \frac{1}{\|a\|^2}\right)\notag\\
& \leq \frac{1}{\|a\|^2} \left( \|ba^\top - bb^\top\|_{\op}+\|aa^\top - ba^\top\|_{\op}\right) + (1-\frac{\|b\|^2}{\|a\|^2})\notag\\
& \leq 2\frac{ \|a\|+\|b\|}{\|a\|^2}\|a-b\| \leq \frac{ 4\|a\|+2\|a-b\|}{\|a\|^2}\|a-b\| .\label{eq:rankonediff}
\end{align}
Note that, thanks to item 2. of Lemma~\ref{lemma:P2conc1}, $\|a\|^2 \geq c$. Moreover, 
\begin{align*}
\|a - b\|^2 & = \frac{1}{n} \|Z_n(w) - Z_n(w_{m+1}^\star)\|^2 \\
&\leq \frac{1}{n}  \|\mathbf{X}\|^2_{\op} \ \|w-w_{m+1}^\star\|^2 \\
& \leq 2 C^2 (1-w^\top w_{m+1}^\star). 
\end{align*}
Plugging this into Equation~\eqref{eq:rankonediff} concludes item 3.

\textbf{3.} It is a simple consequence of the fact that $\|J\|_{\op}\leq \|\mathbf{X}\|_{\op}$ for any $J\in J_n(w)$, along with item 1. of Lemma~\ref{lemma:P2conc1}.

\textbf{4.} It is a direct application of the first item of Lemma~\ref{lemma:P2conc2}.

\textbf{5.} It is a consequence of the following inequality for any $J\in J_n(w)$ and $i\in[n]$:
\begin{equation*}
0 \leq \left(Z_n(w_{m+1}^\star)-w^\top w_{m+1}^\star Z_n(w) - J q\right)_i \leq |x_i^\top w_{m+1}^\star| \iind{(x_i^\top w) (x_i^\top w_{m+1}^\star)\leq 0},
\end{equation*}
where the right inequality is an equality whenever $x_i^\top w \neq 0$. 
Item 5. then directly follows from the second item of Lemma~\ref{lemma:P2conc2}.
\end{proof}

We rely on the following decomposition.
\begin{lemma}\label{lemma:P2decomposition}
For any $w\in\R^d$ such that $D_n(w)>0$, and noting $u=P_{w^\perp} w_{m+1}^{\star}$,
\begin{equation*}
u^\top \Psi_n(w) = \left\{ \frac{1}{n}\left\|P_{n,w} J u\right\|^2  + \frac{1}{n}\left(J u\right)^\top P_{n,w} \left(Z_n(w_{m+1}^\star)-w^\top w_{m+1}^\star Z_n(w) - J u \right) \mid J\in J_n(w)\right\},
\end{equation*}
where $P_{n,w} \coloneqq \Pi_n - \frac{\Pi_n Z_n(w)Z_n(w)^\top \Pi_n}{Z_n(w)^\top\Pi_n Z_n(w)}$ is the orthogonal projection on the orthogonal of the column space of $[F_n, Z_n(w)]$.
\end{lemma}
\begin{proof}
Recall that $\Psi_n = \left\{\frac{1}{n}J^\top\Pi_n \left(Z_n(w^\star_{m+1}) - \frac{H_n(w)}{D_n(w)}  Z_n(w)\right) \mid J\in J_n(w) \right\}$.

It then suffices to show that for any $J\in J_n(w)$
\begin{align}
\Pi_n \left(Z_n(w_{m+1}^\star)-\frac{H_n(w)}{D_n(w)} Z_n(w)\right) & = P_{n,w} J u + P_{n,w} \left(Z_n(w_{m+1}^\star)-w^\top w_{m+1}^\star Z_n(w) - J u \right) \notag\\
& = P_{n,w} \left(Z_n(w_{m+1}^\star)-w^\top w_{m+1}^\star Z_n(w)\right) \label{eq:decompositionP2}.
\end{align}
Let us prove Equation~\eqref{eq:decompositionP2} above---which does not depend on $J$ anymore. 
For that, recall the observation from the proof of 
Lemma~\ref{lemma:empiricalloss} (computation of $a_n(W)$), that
\begin{equation*}
\frac{H_n(w_{m+1})}{D_n(w_{m+1})} = \arg\min_{a\in\R} \left\|\Pi_n\left(Z_n(w_{m+1}^\star) - a Z_n(w_{m+1})\right)\right\|^2.
\end{equation*}
Equivalently, $\Pi_n \left(Z_n(w_{m+1}^\star)-\frac{H_n(w)}{D_n(w)} Z_n(w)\right)$ is the orthogonal projection of $\Pi_n Z_n(w_{m+1}^\star)$ onto $\Pi_n Z_n(w)^\perp$, i.e.,
\begin{align*}
\Pi_n \left(Z_n(w_{m+1}^\star)-\frac{H_n(w)}{D_n(w)} Z_n(w)\right) &= \left(\id[n]-\frac{\Pi_n Z_n(w)Z_n(w)^\top \Pi_n}{Z_n(w)^\top\Pi_n Z_n(w)}\right) \Pi_n Z_n(w_{m+1}^\star)\\
& = P_{n,w}  Z_n(w_{m+1}^\star).
\end{align*}
Since $P_{n,w}Z_n(w) = \mathbf{0}$ by definition of $P_{n,w}$, it directly implies Equation~\eqref{eq:decompositionP2} and thus Lemma~\ref{lemma:P2decomposition}.
\end{proof}
Although this decomposition might not seem helpful at first hand, it is actually very useful thanks to the concentration bounds of Corollary~\ref{coro:concentrationP2}, which allow to control both terms appearing in the decomposition of Lemma~\ref{lemma:P2decomposition} when close to $w_{m+1}^\star$. The first term then dominates the second term, ensuring local convergence of $\bw_{m+1}^\circ(t)$ towards $w_{m+1}^\star$.

\begin{lemma}\label{lemma:P2emprate}
Let $\delta\in(0,1/2)$ and Assumption~\ref{ass:orthofeatures} holds. There exist positive universal constants $\eta_0$, $c$ and $C$ such that if $n\geq C\left( d + \ln(1/\delta)\right)$, then with probability at least $1-\delta$, for any $w\in\bS_{d-1}$ with $1-w^\top w_{m+1}^\star \leq \eta_0$:
\begin{equation}
\inf_{\psi\in\Psi_n(w)}w_{m+1}^{\star\,\top} P_{w^\perp} \psi \geq c\left(1-w^\top w_{m+1}^\star\right).
\end{equation}
\end{lemma}

\begin{proof}
Let $\varepsilon\in(0,1)$ to be fixed later. If $n\geq C_\varepsilon(d+m+\ln(1/\delta))$, then with probability at least $1-\delta$, the five concentration bounds of Corollary~\ref{coro:concentrationP2} all hold uniformly on any $w\in\bS_{d-1}$ such that $1-w^\top w_{m+1}^\star \leq \eta_\varepsilon$. From now on, assume these five items hold, and let $w\in\bS_{d-1}$ such that $1-w^\top w_{m+1}^\star \leq \min(\eta_\varepsilon,\frac{1}{2})$.

Thanks to Lemma~\ref{lemma:P2decomposition}, noting $u=P_{w^\perp} w_{m+1}^{\star}$ for shortness,
\begin{equation}\label{eq:P2decompbound}
\inf_{\psi\in\Psi_n(w)} u^\top \psi  \geq \inf_{J\in J_n(w)} \frac{1}{n} \|P_{n,w}Ju\|^2 - \sup_{J\in J_n(w)}\frac{1}{n}\|Ju\| \|Z_n(w_{m+1}^\star)-w^\top w_{m+1}^\star Z_n(w) - Ju\|.
\end{equation}
Items 3. and 5. of Corollary~\ref{coro:concentrationP2} directly imply that
\begin{equation}
\sup_{J\in J_n(w)}\frac{1}{n}\|Ju\| \|Z_n(w_{m+1}^\star)-w^\top w_{m+1}^\star Z_n(w) - Ju\| \leq 2C \sqrt{\varepsilon}(1-w^\top w_{m+1}^\star).\label{eq:P2boundterm2}
\end{equation}
Moreover for the first term, note that for any $J\in J_n(w)$:
\begin{align*}
\frac{1}{n} \|P_{n,w}Ju\|^2 & = \frac{1}{n} \|P_{n,w_{m+1}^\star}J_n(w_{m+1}^\star)u\|^2 +  \frac{1}{n} \left(\|P_{n,w}Ju\|^2 - \|P_{n,w_{m+1}^\star}J_n(w_{m+1}^\star)u\|^2 \right)\\
& \geq \frac{3c}{8} (1-w^\top w_{m+1}^\star) -  \frac{1}{n} \left(\|P_{n,w}Ju\| + \|P_{n,w_{m+1}^\star}J_n(w_{m+1}^\star)u\|\right)\left\|P_{n,w_{m+1}^\star}J_n(w_{m+1}^\star)u-P_{n,w}Ju\right\|,
\end{align*}
thanks to item 1. of Corollary~\ref{coro:concentrationP2}. From there, first note that
\begin{align*}
 \frac{1}{\sqrt{n}}\left(\|P_{n,w}Ju\| + \|P_{n,w_{m+1}^\star}J_n(w_{m+1}^\star)u\|\right) & \leq \frac{2}{\sqrt{n}} \|Ju\| + \frac{1}{\sqrt{n}} \|(J-J_n(w_{m+1}^\star))u\| \\
 & \leq \left( 2\sqrt{2}C +\sqrt{2\varepsilon}\right)\sqrt{1-w^\top w_{m+1}^\star},
\end{align*}
thanks to items 3. and 4. of Corollary~\ref{coro:concentrationP2}. Finally, we have by triangle inequality
\begin{align*}
 \frac{1}{\sqrt{n}}\left\|P_{n,w_{m+1}^\star}J_n(w_{m+1}^\star)u-P_{n,w}Ju\right\| & \leq  \frac{1}{\sqrt{n}}\left\|P_{n,w_{m+1}^\star}J_n(w_{m+1}^\star)u-P_{n,w_{m+1}^\star}Ju\right\| +  \frac{1}{\sqrt{n}}\left\|(P_{n,w_{m+1}^\star}-P_{n,w})Ju\right\| \\
 & \leq  \frac{1}{\sqrt{n}}\|(J-J_n(w_{m+1}^\star))u\| + \frac{1}{\sqrt{n}} \|Ju\| \|P_{n,w_{m+1}^\star}-P_{n,w}\|_{\op}\\
 & \leq \sqrt{2\varepsilon} \sqrt{1-w^\top w_{m+1}^\star} + 8 \frac{2c+C}{c}C^{3/2} (1-w^\top w_{m+1}^\star),
\end{align*}
thanks to items 2., 3. and 4. of Corollary~\ref{coro:concentrationP2}. 
Plugging everything back together, it finally yields that:
\begin{multline*}
\inf_{J\in J_n(w)} \frac{1}{n} \|P_{n,w}Ju\|^2 \geq \left(\frac{3c}{8}-\left(2\sqrt{2}C+\sqrt{2\varepsilon}\right)\cdot\left(\sqrt{2\varepsilon}+8\frac{2c+C}{c}C^{3/2}\sqrt{1-w^\top w_{m+1}^\star}\right) \right)\\\cdot(1-w^\top w_{m+1}^\star).
\end{multline*}
Thus, for small enough universal constants $\varepsilon$ and $\eta_0$, it comes that for any $w\in\bS_{d-1}$ such that $1-w^\top w_{m+1}^\star \leq \eta_0$,
\begin{equation*}
\inf_{J\in J_n(w)} \frac{1}{n} \|P_{n,w}Ju\|^2 \geq \frac{c}{4} (1-w^\top w_{m+1}^\star),
\end{equation*}
and simultaneously, thanks to Equation~\eqref{eq:P2boundterm2}:
\begin{equation*}
\sup_{J\in J_n(w)}\frac{1}{n}\|Ju\| \|Z_n(w_{m+1}^\star)-w^\top w_{m+1}^\star Z_n(w) - Ju\| \leq \frac{c}{8}(1-w^\top w_{m+1}^\star).
\end{equation*}
It then allows to conclude, thanks to the decomposition of Equation~\eqref{eq:P2decompbound}.
\end{proof}

\subsubsection{Proof of Proposition~\texorpdfstring{\ref{prop:empiricalrates}}{2}}
Let us first restate Proposition~\ref{prop:empiricalrates}.

\propconc*

\begin{proof}
Let $\eta\in(0,1)$ to be fixed later. First recall the identity 
\begin{equation*}
\Psi_n(w)= \frac{1}{n}J_n(w)^\top\Pi_n\left( Z_n(w^\star) - \frac{H_n(w)}{D_n(w)}  Z_n(w)\right).
\end{equation*}
If $n \geq \frac{C}{\varepsilon^2}\ln(1/\varepsilon)^2\left( d +\ln(1/\delta)\right)$ for a large enough constant $C$, Lemma~\ref{lemma:concHD} directly implies that for any $w\in\bS_{d-1}$, with probability at least $1-\delta$, uniformly on $\bS_{d-1}$:
\begin{enumerate}
\item $|D_n(w)-D(w)|\leq C_0\eta$;
\item $\inf_{\psi \in \Psi_n(w)} \|\psi - \nabla \tilde{H}(w) - \frac{H(w)}{2D(w)}\nabla\tilde{D}(w)\| \leq C_0\min(\eta,\varepsilon)$,
\end{enumerate}
for some universal constant $C_0$. Assume in the following that these different concentrations hold.

Since $P_{w^\perp} \nabla \widetilde{D}(w) = P_{w^\perp} \nabla D(w)$---and similarly for $\widetilde{H}$ and $H$---it then implies, thanks to Lemmas~\ref{lemma:poprate} that for a small enough choice of constants $\eta$ and $c$ that for any $w\in\cT(c)$,
\begin{enumerate}
\item $\underline{D}\leq D_n(w) \leq \overline{D}$;
\item $\inf_{\psi \in \Psi_n(w)}w_{m+1}^{\star\,\top}P_{w^\perp} \psi \geq c (1-w^\top w_{m+1}^\star) - C\eta$;
\item $\sup_{\psi \in \Psi_n(w)}\|P_m P_{w^\perp} \psi\| \leq C \left(\frac{1}{\sqrt{m}}+\|P_m w\|\right) +\varepsilon$.
\end{enumerate}
It already implies items 1. and 3. of Proposition~\ref{prop:empiricalrates}--assuming $\eta$ will be taken of constant order. For item 2., it also proves it when sufficiently far from $w_{m+1}^\star$. Indeed, we have for any $w\in\cT(c)$ such that $(1-w^\top w_{m+1}^\star) \geq \frac{C\eta}{2c}$,
\begin{equation}\label{eq:growthrateprop}
\inf_{\psi \in \Psi_n(w)}w_{m+1}^{\star\,\top}P_{w^\perp} \psi \geq \frac{c}{2} (1-w^\top w_{m+1}^\star).
\end{equation}

For the neighborhood of $w_{m+1}^\star$, Lemma~\ref{lemma:P2emprate} directly implies that for a small enough universal constant $\eta$, then with probability at least $1-\delta$, for any $w\in\cT(c)$ such that $(1-w^\top w_{m+1}^\star) \leq \frac{C\eta}{2c}$:
\begin{equation}
\inf_{\psi\in\Psi_n(w)}w_{m+1}^{\star\,\top} P_{w^\perp} \psi \geq c\left(1-w^\top w_{m+1}^\star\right).
\end{equation}
This concludes the proof.
\end{proof}

\subsubsection{Concentration at initialization}\label{app:concinit}

The population loss also relied on a lower bound of $H$ of order $\frac{1}{m}$ on the trajectory. We also need such a lower bound in the empirical case. However, we cannot rely on a uniform concentration over $\bS_{d-1}$, as we would get an error term $\Theta(1)$, which is larger than $1/m$ and thus yields no guarantee on the sign of $H_n$ on the points of interest. 
Instead, we rely on a single point concentration at the initialization $\breve{w}$, which is chosen independently of the training data.

An important detail that can be noted from the proof of Lemma~\ref{lemma:concentration} is that for its first item, we only need a sample complexity $n\geq \frac{C}{\varepsilon^2}(m+\ln(1/\delta))$, since we only relied on the use of Lemma~\ref{lemma:chainingVC} with parameters $p_1=p_2=m$ and $\cA= \{\R^d\}$ for that point. Using the same proof technique, we can thus also prove the following pointwise concentration, which yields tighter bounds at the initialization point of the dynamics.
\begin{lemma}\label{lemma:concentrationpointwise}
Consider Assumption~\ref{ass:orthofeatures}. There exists a universal constant $C>0$ such that for any $\varepsilon,\delta\in(0,1/2)$ and $w,v\in\R^d$, if $n \geq \frac{C}{\varepsilon^2}\left( m +\ln(1/\delta)\right)$, 
then with probability at least $1-\delta$, the following statements hold simultaneously:
\begin{enumerate}
\item $\|G^{-1/2}\left(\frac{1}{n}F_n^\top F_n\right)G^{-1/2} - \id[m] \|_{\op} \leq \varepsilon$;
\item $\left\|G^{-1/2}\left( \frac{1}{n}F_n^\top Z_n(w) - \Phi^* \phi_{w}\right)\right\|_2\leq \varepsilon $.
\end{enumerate}
If $n \geq \frac{C}{\varepsilon^2}\ln(1/\delta)$, 
then with probability at least $1-\delta$, the following statements hold simultaneously:
\begin{enumerate}[start=3]
\item $\left|\frac{1}{n}Z_n(v)^\top Z_n(w) - \kappa( w^\top v) \right|\leq \varepsilon$;
\item $\left|\Phi^*(\phi_v)^\top G^{-1/2} \left(\frac{1}{n}G^{-1/2}F_n^\top F_nG^{-1/2}-\id[m]\right)G^{-1/2} \Phi^* \phi_{w}\right| \leq \varepsilon$;
\item $\left|\Phi^*(\phi_v)^\top G^{-1}\left( \frac{1}{n}F_n^\top Z_n(w) - \Phi^* \phi_{w}\right)\right| \leq \varepsilon$.
\end{enumerate}
\end{lemma}
The difference with Lemma~\ref{lemma:concentration} here is that the concentrations are pointwise in $w$ and $v$, while Lemma~\ref{lemma:concentration} is stated uniformly over all $w$ and $v$, thus requiring a larger sample complexity.
\begin{proof}
Let the function $h:\R^d\to \R^m$ be defined as in Equation~\eqref{eq:functionh}. This proof follows similar arguments to the one of Lemma~\ref{lemma:concentration}, where we only use Lemma~\ref{lemma:chainingVC} in the particular case $\cA=\{\R^d\}$, leading to improved bounds.

\textbf{1.} Exactly as the first item of Lemma~\ref{lemma:concentration}, it is a consequence of Lemma~\ref{lemma:chainingVC} with $h_1=h_2=h$ and $\cA=\{\R^d\}$.

\textbf{2.} It is a consequence of Lemma~\ref{lemma:chainingVC} with $h_1=h$, $h_2=\phi_w$ and $\cA=\{\R^d\}$.

\textbf{3.} It is a consequence of Lemma~\ref{lemma:chainingVC} with $h_1=\phi_v$, $h_2=\phi_w$ and $\cA=\{\R^d\}$---which is actually Bernstein's inequality.

\textbf{4.} Denote $a = G^{-1/2} \Phi^* \phi_{w}$ and $b = G^{-1/2} \Phi^* \phi_{v}$. It comes that $\|a\|_2 \leq \frac{1}{\sqrt{2}}$ thanks to Equation~\eqref{eq:normbounda}.
Also $\|G^{-1/2}a\|_2\leq \sqrt{\frac{\pi}{\pi-1}}$, using the bound on $\|G^{-1}\|_{\op}$. 
We now derive similarly to the proof of item 1. of Lemma~\ref{lemma:concentration}:
\begin{multline*}
b^\top \left(\frac{1}{n}G^{-1/2}F_n^\top F_nG^{-1/2}-\id[m]\right) a = \frac{1}{n} \sum_{k=1}^n (G^{-1/2}a)^\top h(x_k) h(x_k)^\top (G^{-1/2}b)\\- \bE[ (G^{-1/2}a)^\top h(X) h(X)^\top (G^{-1/2}b) ].
\end{multline*}
Bernstein's inequality then allows to conclude.

\textbf{5.} A similar Bernstein's inequality argument allows to conclude here.
\end{proof}

\begin{corollary}\label{coro:initconcentrationemp}
Let $\breve{g}$ be drawn uniformly at random on $\bS_{d-1}$, independently from the training data $(x_1,\ldots,x_n)$. Then there exist positive universal constants $c,C$ such that for any $\delta\in(0,1/2)$, if
\begin{equation*}
\min(n,d)\geq C m^2\ln(1/\delta),
\end{equation*}
then with probability at least $1-\delta$, the following statements hold simultaneously:
\begin{enumerate}
\item $\|P_m \breve{w}\|\leq \frac{1}{4(\pi-1)\sqrt{m}}$;
\item $\breve{w}^\top w_{m+1}^\star \geq -\frac{1}{4 \pi \ m}$;
\item  $H_n(\breve{w})\geq \frac{c}{m}$.
\end{enumerate}
\end{corollary}
\begin{proof}
First, the condition on $d$ along with Lemma~\ref{lemma:initpop} guarantees that with probability at least $1-\delta$, the first two items of Corollary~\ref{coro:initconcentrationemp} hold.

For the third point, by independence between $\breve{w}$ and the training data, we can derive concentration bounds conditionally on $\breve{w}$. We thus assume in the following $\breve{w}$ fixed such that $\|P_m \breve{w}\|\leq \frac{1}{4(\pi-1)\sqrt{m}}$ and $\breve{w}^\top w_{m+1}^\star \geq -\frac{1}{4 \pi \ m}$. First note that, thanks to Lemma~\ref{lemma:poprate}, $H(\breve{w})\geq\frac{c}{m}$ for some positive constant $c$.

From there, define for the remaining of the proof:
\begin{gather*}
A_n = G^{-1/2}\left(\frac{1}{n}F_n^\top F_n\right)G^{-1/2}\\
u = G^{-1/2} \Phi^* \phi_{\breve{w}}; \qquad u_n = \frac{1}{n} G^{-1/2} F_n^\top Z_n(\breve{w}) \\
v = G^{-1/2} \Phi^* \phi_{w_{m+1}^\star}; \qquad v_n = \frac{1}{n} G^{-1/2} F_n^\top Z_n(w_{m+1}^\star).
\end{gather*}
Assume now that $n\geq C\max\left(\frac{1}{\varepsilon^2}, \frac{m}{\varepsilon}\right) \ln(1/\delta)$ for some universal constant $C$ and $\varepsilon\in(0,1/2)$. 
Lemma~\ref{lemma:concentrationpointwise} applied to the two fixed vectors $\breve{w}$ and $w_{m+1}^\star$ implies with probability at least $1-\delta$:
\begin{enumerate}
\item $\max\left(\|A_n-\id[m] \|_{\op},\|u-u_n\|_2,\|v-v_n\|_2\right)\leq \sqrt{\varepsilon}$;
\item $\left|\frac{1}{n}Z_n(w_{m+1}^\star)^\top Z_n(\breve{w}) - \kappa( \breve{w}^\top w_{m+1}^\star) \right|\leq \varepsilon$;
\item $\max\left(|v^\top (A_n-\id[m]) u| ,|(v-v_n)^\top u |,|v^\top (u-u_n)|\right)\leq\varepsilon$.
\end{enumerate}
By definition, we have
\begin{equation*}
H_n(\breve{w})-H(\breve{w}) = \frac{1}{n}Z_n(w_{m+1}^\star)^\top Z_n(\breve{w}) - \kappa( \breve{w}^\top w_{m+1}^\star) + v^\top u - v_n^\top A_n^{-1} u_n.
\end{equation*}
The first difference is directly bounded by $\varepsilon$, thanks to the second item above. For the second difference, a technical development yields
\begin{align*}
v^\top u - v_n^\top A_n^{-1} u_n & = (v-v_n)^\top u + v^\top (\id[m]-A_n^{-1})u- (v-v_n)^\top(\id[m]-A_n^{-1})u \\
&\phantom{=}+v^\top (u-u_n) - (v-v_n)^\top(u-u_n) - v_n^\top(\id[m]-A_n^{-1})(u-u_n).
\end{align*}
Now using the identity $\id[m]-A_n^{-1} = A_n-\id[m]-(A_n-\id[m])^2 A_n^{-1}$ for the last term, along with the fact that $\max(\|v\|, \|u\|) \leq \frac{1}{\sqrt{2}}$ and $\|A_n^{-1}\|_{\op}\leq 2$, it comes with the items 1 and 3 above that:
\begin{equation*}
|v^\top u - v_n^\top A_n^{-1} u_n| \leq (5 + 2\sqrt{2}) \varepsilon.
\end{equation*}
Finally, we thus have $H_n(\breve{w})\geq H(\breve{w}) - (6 + 2\sqrt{2}) \varepsilon$. Since $H(\breve{w})\geq \frac{c}{m}$ for some positive constant $c$, choosing $\varepsilon=\frac{c_0}{m}$ for a small enough universal constant $c_0>0$ then allows to conclude.
\end{proof}

\subsection{Proof of Theorem~\texorpdfstring{\ref{thm:mainempirical}}{2}}

Consider the positive universal constants $C,c, \underline{D},\overline{D}$ and $m_0$ appearing in Proposition~\ref{prop:empiricalrates} and Corollary~\ref{coro:initconcentrationemp}.

Let $\varepsilon\in(0,1)$ a small universal constant to be fixed later. Under the stated regime in $n$ and $d$, we have thanks to Corollary~\ref{coro:initconcentrationemp}, with probability at least $1-\delta/2$, that the following events simultaneously hold for some positive constant $c$:
\begin{enumerate}
\item $\|P_m \bw_{m+1}^\circ(0)\|\leq \frac{1}{4(\pi-1)\sqrt{m}}$;
\item $\bw_{m+1}^\circ(0)^\top w_{m+1}^\star \geq -\frac{1}{4 \pi \ m}$;
\item  $H_n(\bw_{m+1}^\circ(0))\geq \frac{c}{m}$.
\end{enumerate}
Moreover, thanks to Proposition~\ref{prop:empiricalrates}, the following events also hold with probability at least $1-\delta/2$ for any $w\in\cT(c)$:
\begin{enumerate}[start=3]
\item $\underline{D}\leq D_n(w) \leq \overline{D}$;
\item $\inf_{\psi \in \Psi_n(w)}w_{m+1}^{\star\,\top}P_{w^\perp} \psi \geq c (1-w^\top w_{m+1}^\star)$;
\item $\sup_{\psi \in \Psi_n(w)}\|P_m P_{w^\perp} \psi\| \leq C_0 \left(\frac{1}{\sqrt{m}}+\|P_m w\|\right) +\varepsilon$.
\end{enumerate}

Assume these five random events all hold. From that point, the remaining of the proof follows exactly the same lines as in the population case, given in Appendix~\ref{app:popproof}. The chain rule still applies \citep[thanks to][]{bolte2021conservative}, so that $\frac{H_n(\bw_{m+1}^\circ(t))^2}{D_n(\bw_{m+1}^\circ(t))}$ increases over time. From there, the reparametrized time variable given by
\begin{equation*}
\gamma : s \mapsto \int_{0}^s\frac{D_n(\bw_{m+1}^\circ(u))}{H_n(\bw_{m+1}^\circ(u))}\df u \qquad \text{and}\qquad\tw(s) \coloneqq \bw_{m+1}^\circ(\gamma(s))
\end{equation*}
satisfies almost anywhere, as long as $\widetilde{\bw}(s)\in\cT(c)$, that
\begin{gather}
\frac{\df }{\df s}w_{m+1}^{\star\,\top}\tw(s) \geq c (1-w_{m+1}^{\star\,\top}\tw(s))\label{eq:growthemp0}\\
\frac{\df}{\df s}\|P_m \tw(s)\|\leq C_0\left(\frac{1}{\sqrt{m}}+\|P_m w\|\right)+ \varepsilon.\notag
\end{gather}
Defining $\overline{s}\coloneqq \inf \left\{s\geq 0 : \tw(s)\not\in \cT\left(c\right)\right\}$, $(1-w_{m+1}^{\star\,\top}\tw(s))$ is increasing on $[0,\overline{s}]$ and by a Gr\"onwall argument, for any $s\in[0,\overline{s}]$ 
\begin{equation}\label{eq:growthemp1}
1- w_{m+1}^{\star\,\top}\tw(s) \leq \frac{5}{4}	e^{-cs}.
\end{equation}
Moreover the second point implies, with $\|P_m \tw(s)\| \leq \sqrt{2(1-w_{m+1}^{\star\,\top}\tw(s))}$, that for any $s\in[0,\overline{s}]$:
\begin{align*}
\|P_m \tw(s)\| & \leq 2\min\left(\left(\frac{1}{\sqrt{m}}+\frac{\varepsilon}{C_0}\right)e^{C_0 s} , e^{-cs/2} \right)\\
& = 2 \exp\left( \frac{c/2}{C_0+c/2}\ln\left(\frac{1}{\sqrt{m}}+\frac{\varepsilon}{C_0}\right)\right).
\end{align*}
Thus, for a large enough choice of $m_0$ and small enough, constant, choice of $\varepsilon$, $\|P_m \tw(s)\| \leq \frac{c}{2}$ on $[0, \overline{s}]$, so that $\overline{s} =  \infty$, i.e., $w_{m+1}(t) \in \cT(c)$ for any $t\geq 0$. 

Similarly to the proof of Theorem~\ref{thm:mainpop}, we can again conclude that for a small enough constant $\eta$ and $s_\eta = \inf \{s \geq 0 \mid w_{m+1}^{\star\,\top}\tw(s) \geq 1 - \eta \}$,
\begin{equation*}
\gamma'(s) = \frac{D_n(w_{m+1}(s))}{H_n(w_{m+1}(s))} \leq \frac{\overline{D}}{2\underline{D}} \quad \text{for any }s\geq s_\eta,
\end{equation*}
and $\gamma(s_\eta) \leq C_0 m$ for a large enough universal constant $C_0$. 
In particular, for any $t\geq \gamma(s_\eta)$, Equation~\eqref{eq:growthemp0} then implies:
\begin{align*}
1- \bw_{m+1}^\circ(t)^\top w_{m+1}^\star  \leq \exp\left(-\frac{2c\underline{D}}{\overline{D}}(t-\gamma(s_\eta))\right),
\end{align*}
which concludes on the first part of Theorem~\ref{thm:mainempirical}. 

Letting now $\cT= \{[w_{1}^{\star\,\top}, \ldots, w_{m}^{\star\,\top}, w_{m+1}^\top] \mid w_{m+1} \in \cT(c)\}$, the trajectory $\bW^\circ$ is obviously confined to $\cT$. Moreover, $\cT$ satisfies Assumption~\ref{ass:conditioning}: it is obviously compact. Moreover, for any $W\in\cT$,
\begin{equation*}
\Sigma(W, \mathbf{X}) = [F_n, Z_n(w)^\top]^\top.
\end{equation*} 
Note that from Lemma~\ref{lemma:concHD}, under the same random event, $F_n^\top F_n$ is invertible and $D_n(w)>0$ for any $w\in\cT(c)$. Those two conditions imply that the family $(Z_n(w_1^\star),\ldots, Z_n(w_m^\star), Z_n(w))$ is linearly independent for any $w\in\cT(c)$ or, equivalently, the matrix $[F_n, Z_n(w)^\top]^\top$ is full rank. In other words, $\lambda_{\min}\left(\Sigma(W, \mathbf{X})\Sigma(W, \mathbf{X})^\top\right)>0$ for any $W\in\cT$. By compactness and continuity of the considered function, this directly yields that $\cT$ satisfies Assumption~\ref{ass:conditioning}.

\qed

\section{Additional lemmas}

\begin{lemma}\label{lemma:conditioning}
Consider $T\in\R_+$ and any positive sequence $(\alpha_k)_{k\in\N}$ such that $\alpha_k \overset{k\to\infty}{\longrightarrow}0$, and corresponding solutions $\bW^{\alpha_k}$ of Equation~\eqref{eq:twotimescalenormalized} such that 
$\|w_i(0)\| = \begin{cases} \|w_i^\star\| \quad \text{for }i\in[m] \\
\alpha \quad \text{otherwise}\end{cases}$.

Suppose all solutions $\bW^\circ(t)$ of Equation~\eqref{eq:oneneurondyn} are contained within some set $\cT$ satisfying Assumption~\ref{ass:conditioning} on $[0,T]$. Then there exists an augmented set $\cT_{+}$ satisfying Assumption~\ref{ass:conditioning}, such that for $k$ large enough, the trajectory $\left(\bW^{\alpha_k}(\alpha_k^2 t)\right)_{t\in[0,T]}$ is included within $\cT_{+}$.
\end{lemma}

\begin{proof}
Observe that by continuity of $W\mapsto \lambda_{\min}\left(\frac{1}{n}\Sigma(W,\bX)\Sigma(W,\bX)^\top\right)$ and compactness, we can choose a small enough $\varepsilon>0$ such that the set $\cT_{+}=\cT+\bar{B}(0,\varepsilon)$ also satisfies Assumption~\ref{ass:conditioning}. From there, define for any $k$, $\tau_k=\inf\left\{t\geq 0 \mid \bW^{\alpha_k}(\alpha_k^2 t)\not\in\cT_+\right\}$. Since $\bW^{\alpha_k}(0)=\bW^\circ(0)\in\cT$ by definition, $\tau_k\geq 0$. Then consider $\tau=\liminf_k \tau_k\geq 0$ and assume without loss of generality -- by extraction -- that $\tau=\lim_k \tau_k$.

Assume here that $\tau\leq T$ and define $\tau_\delta=\max(0,\tau-\delta)$ for some $\delta>0$. For large enough $k$, $|\tau_k-\tau|\leq \delta$, so that we can apply Proposition~\ref{prop:limitdyn} on $[0,\tau_\delta]$; i.e., up to extraction, $\bW^{\alpha_k}(\alpha_k^2 \cdot)$ converges uniformly on $[0,\tau_\delta]$ to some solution $W^\circ$ of Equation~\eqref{eq:oneneurondyn}. In particular, for $k$ large enough, $\bW^{\alpha_k}(\alpha_k^2 t)\in \cT+B(0,\varepsilon/3)$ for any $t\in[0,\tau_{\delta}]$.  
Moreover, thanks to Corollary~\ref{coro:uniformlipschitz}, $\bW^{\alpha_k}(\alpha_k^2 \cdot)$ is uniformly Lipschitz\footnote{Corollary~\ref{coro:uniformlipschitz} is stated for a fixed interval $[0,T]$ for simplicity, but it can directly be extended to allow for different time intervals of the type $[0,\tau_k]$, as long as the weights are in $\cT_+$.} on $[0,\max(\tau_k,\tau_{\delta})]$. In particular, if we denote $L$ the Lipschitz constant and take $\delta<\varepsilon/(3L)$, the condition $\bW^{\alpha_k}(\alpha_k^2\tau_\delta)\in \cT+B(0,\varepsilon/3)$ implies that for $k$ large enough, $\bW^{\alpha_k}(\alpha_k^2\tau_k)\in \cT+B(0,\frac{2\varepsilon}{3})$, which contradicts the definition of $\tau_k$. 
By contradiction, we here showed that  $\tau>T$, i.e., $\liminf_k \tau_k> T$, which directly implies Lemma~\ref{lemma:conditioning}.
\end{proof}

\subsection{Useful concentration bounds}\label{app:generalconcentration}
The different uniform concentration bounds of Appendix~\ref{app:uniform_concentration} rely on the following key concentration lemma for empirical processes, which derives from classical chaining and VC subgraph classes arguments \citep[see][Chapter 2]{vaart1996weak}.

For a collection of subsets $\cA$ and a family of functions $\cF$, we denote their VC dimension and VC subgraph dimension by $\vcdim(\cA)$ and $\vcdim(\cF)$, respectively, following the standard definitions of \citet[Chapter 2.6]{vaart1996weak}.
\begin{lemma}\label{lemma:chainingVC}
Let $\cF$ be a family of functions from $\R^d$ to $\R$, defined as
\begin{equation*}
\cF = \left\{ x \mapsto (u^\top h_1(x)) (h_2(x)^\top v)\iind{x\in A} \mid u\in\bS_{p_1-1}, v\in\bS_{p_2-1}, A\in\cA   \right\}
\end{equation*}
where $\cA$ is a VC class of subsets of $\R^d$ and $h_1: \R^d \to \R^{p_1}$, $h_2 : \R^d \to \R^{p_2}$ satisfy: 
$\max(\|h_1(X)\|_{\psi_2},\|h_2(X)\|_{\psi_2})\leq C_0$ for some universal constant $C_0$ when $X \sim \cN(0,\id)$.

Then, there is a universal constant $C$ such that for any $\varepsilon, \delta \in(0,1/2)$, if $n\geq C\varepsilon^{-2}\ln(1/\varepsilon)^2\left(\vcdim(\cA)+p_1+p_2+\ln(\frac{1}{\delta})\right)$,
\begin{equation*}
\bP\left(\sup_{f\in\cF} \left|\frac{1}{n}\sum_{i=1}^n f(X_i)- \bE\left[f(X)\right]\right|\geq \varepsilon \right)\leq \delta,
\end{equation*}
where the probability is over $X_1,\ldots,X_n$ drawn i.i.d. from $\cN(0,\id)$.

Moreover, if $\cA=\{\R^d\}$, $n\geq C\varepsilon^{-2}\left(p_1+p_2+\ln(\frac{1}{\delta})\right)$ samples are sufficient to reach the same concentration bound.
\end{lemma}
\begin{proof}
We first fix $u\in\bS_{p_1-1}, v\in\bS_{p_2-1}$ and denote for convenience $F(x)=(u^\top h_1(x)) (h_2(x)^\top v)$, $\cF_{u,v} = \{ x\mapsto F(x) \iind{x\in A} \mid A \in \cA\}$.

\paragraph{Bound on expected supremum.} Theorems 2.14.1 and 2.6.7 of \citet{vaart1996weak} directly yield that for some universal constant $C$,
\begin{multline*}
 \bE\left[\sup_{f\in\cF_{u,v}} \left| \frac{1}{n} \sum_{i=1}^n f(X_i) - \bE[f(X)]\right| \right] \leq \frac{C}{\sqrt{n}} \|F\|_{L^2(\cN(0,\id))} \cdot \\ \int_{0}^1 \sqrt{1+\ln \left(C \cdot \vcdim(\cF_{u,v})(16e)^{\vcdim(\cF_{u,v})}\left(\frac{1}{\varepsilon}\right)^{2\vcdim(\cF_{u,v})-2}\right)}\df \varepsilon.
\end{multline*}
Since $\max(\|h_1(X)\|_{\psi_2},\|h_1(X)\|_{\psi_2})\leq C_0$, $\|F(X)\|_{\psi_1}\leq C_0^2$. In particular, its second moment, i.e., $\|F\|_{L^2(\cN(0,\id))}$, is of constant order. A simple Cauchy-Schwarz argument then yields that for some universal constant $C$,
\begin{equation*}
\bE\left[\sup_{f\in\cF_{u,v}} \left| \frac{1}{n} \sum_{i=1}^n f(X_i) - \bE[f(X)]\right| \right] \leq C \sqrt{\frac{\vcdim(\cF_{u,v})+1}{n}}.
\end{equation*}
Moreover, by definition of $\cF_{u,v}$, $\vcdim(\cF_{u,v})\leq\vcdim(\cA)$ \citep[see e.g.,][Lemma 2.6.18]{vaart1996weak}, so that
\begin{equation*}
\bE\left[\sup_{f\in\cF_{u,v}} \left| \frac{1}{n} \sum_{i=1}^n f(X_i) - \bE[f(X)]\right| \right] \leq C \sqrt{\frac{\vcdim(\cA)+1}{n}}.
\end{equation*}

\paragraph{Truncated decomposition.} Let $M\geq 0$ to be fixed later. Define $\cF_M = \{x \mapsto f(x)\iind{F(x)\in[-M,M]} \mid f\in\cF_{u,v}\}$ and $$\mu_{n,M} \coloneqq \bE\left[\sup_{f\in\cF_M} \left| \frac{1}{n} \sum_{i=1}^n f(X_i) - \bE[f(X)]\right|\right].$$
The exact same argument as above first yields that $\mu_{n,M}\leq  C \sqrt{\frac{\vcdim(\cA)+1}{n}}$ for some constant $C$. Moreover,
\begin{align*}
\sup_{f\in\cF_{u,v}} \left| \frac{1}{n} \sum_{i=1}^n f(X_i) - \bE[f(X)]\right| & \leq \sup_{f\in\cF_M} \left| \frac{1}{n} \sum_{i=1}^n f(X_i) - \bE[f(X)]\right| +  \frac{1}{n} \sum_{i=1}^n |F(X_i)|\iind{|F(X_i)|\geq M} \\ 
&\phantom{\leq}+ \bE[|F(X)|\iind{|F(X)|\geq M}] \\
 & \leq \sup_{f\in\cF_M} \left| \frac{1}{n} \sum_{i=1}^n f(X_i) - \bE[f(X)]\right|+ 2\bE[|F(X)|\iind{|F(X)|\geq M}]  \\ 
&\phantom{\leq}+  \frac{1}{n} \sum_{i=1}^n |F(X_i)|\iind{|F(X_i)|\geq M} - \bE[|F(X)|\iind{|F(X)|\geq M}].
\end{align*}
Let us now bound each of these terms with high probability. First, Theorem 2.14.25 of \citet{vaart1996weak} yields for any $t\geq 0$ that
\begin{align*}
\bP\left(\sup_{f\in\cF_M} \left| \frac{1}{n} \sum_{i=1}^n f(X_i) - \bE[f(X)]\right| \geq C(\mu_{n,M}+t)\right) \leq \exp\left(-Cn \min\left(\frac{t^2}{M^2}, \frac{t}{M}\right)\right).
\end{align*}
Taking $t=M \sqrt{\frac{\ln(1/\delta)}{Cn}}$ and $n\geq \frac{\ln(1/\delta)}{C}$---so that $\frac{t}{M}\leq 1$---it yields that with probability at least $1-\delta$,
\begin{equation}\label{eq:chainingterm1}
\sup_{f\in\cF_M} \left| \frac{1}{n} \sum_{i=1}^n f(X_i) - \bE[f(X)]\right| \leq C\sqrt{\frac{\vcdim(\cA)+1}{n}} + C M \sqrt{\frac{\ln(1/\delta)}{n}}
\end{equation}
for some universal constant $C$.

For the second term, it directly comes from the sub-exponential tail of $F(X)$, i.e.,
\begin{align}
\bE[|F(X)|\iind{|F(X)|\geq M}] &\leq \sqrt{\bE[F(X)^2]}\cdot \sqrt{\bP(|F(X)|\geq M)} \notag\\
& \leq Ce^{-cM},\label{eq:chainingterm2}
\end{align}
for positive universal constants $C$ and $c$. Lastly for the third term, note that $|F(X_i)|\iind{|F(X_i)|\geq M} - \bE[|F(X)|\iind{|F(X)|\geq M}]$ is also $C$ sub-exponential for some constant $C$. So that by classical Bernstein inequality \citep[][Theorem 2.8.1]{vershynin2018high} and if $n\geq C \ln(1/\delta)$, with probability at least $1-\delta$:
\begin{equation}\label{eq:chainingterm3}
\frac{1}{n} \sum_{i=1}^n |F(X_i)|\iind{|F(X_i)|\geq M} - \bE[|F(X)|\iind{|F(X)|\geq M}] \leq \sqrt{\frac{\ln(1/\delta)}{n}}.
\end{equation}
Gathering altogether Equations~\eqref{eq:chainingterm1} to \eqref{eq:chainingterm3} with $M=C\ln(1/\varepsilon)$ for a large enough universal constant $C$ finally yields that if $n\geq C \ln(1/\delta)$, with probability at least $1-\delta$,
\begin{equation}\label{eq:pointwisechaining}
\sup_{f\in\cF_{u,v}} \left| \frac{1}{n} \sum_{i=1}^n f(X_i) - \bE[f(X)]\right| \leq \frac{1}{2}\varepsilon + C\sqrt{\frac{\vcdim(\cA)+1}{n}} + C \ln(1/\varepsilon) \sqrt{\frac{\ln(1/\delta)}{n}}.
\end{equation}

\paragraph{Covering net argument.} Now we have proven Equation~\eqref{eq:pointwisechaining} for a fixed choice of vectors $u,v$, it remains to conclude by a covering net argument. For that, consider $\cN$ a $\frac{1}{4}$-net of $\bS_{p_1-1} \times \bS_{p_2-1}$, with cardinality smaller than $9^{p_1+p_2}$. From Equation~\eqref{eq:pointwisechaining}, a simple union bound then yields that if $n\geq C\ln(9^{p_1+p_2}/\delta)$ with probability at least $1-\delta$,
\begin{equation*}
\sup_{(u,v)\in \cN}\sup_{f\in\cF_{u,v}} \left| \frac{1}{n} \sum_{i=1}^n f(X_i) - \bE[f(X)]\right| \leq \frac{1}{2}\varepsilon + C\sqrt{\frac{\vcdim(\cA)+1}{n}} + C \ln(1/\varepsilon) \sqrt{\frac{\ln(9^{p_1+p_2}/\delta)}{n}}.
\end{equation*}
In particular, if $n\geq C \varepsilon^{-2}\ln(1/\varepsilon)^2 (\vcdim(\cA) + p_1+p_2+ \ln(1/\delta))$ for a large enough universal constant $C$, then with probability at least $1-\delta$, 
\begin{equation}\label{eq:boundnet}
\sup_{(u,v)\in \cN}\sup_{f\in\cF_{u,v}} \left| \frac{1}{n} \sum_{i=1}^n f(X_i) - \bE[f(X)]\right| \leq \varepsilon.
\end{equation}

Moreover, writing for shortness $M_A = \frac{1}{n} \sum_{i=1}^n \iind{X_i\in A} h_1(X_i) h_2(X_i)^\top - \bE[\iind{X\in A} h_1(X) h_2(X)^\top]$ for any $A\in\cA$,
\begin{align*}
\sup_{f\in\cF} \left| \frac{1}{n} \sum_{i=1}^n f(X_i) - \bE[f(X)]\right| & = \sup_{A\in\cA} \sup_{\substack{u\in\bS_{p_1-1}\\v\in\bS_{p_2-1}}}  u^\top M_A v.
\end{align*}
Fix $A\in\cA$ and let $(u_A,v_A)\in \bS_{p_1-1}\times \bS_{p_2-1}$ be such that $\sup_{\substack{u\in\bS_{p_1-1}\\v\in\bS_{p_2-1}}} u^\top M_A v =u_A^\top M_A v_A$. Letting $(u',v')$ be such that $\|(u_A,v_A)-(u',v')\|_2\leq \frac{1}{4}$, it then comes that
\begin{align*}
u'^\top M_A v' & =  u_A^\top M_A v_A - (u_A-u')^\top M_A v_A - u'^\top M_A (v_A-v')\\
& \geq \left(1-\|u_A-u'\|-\|v_A-v'\|\right)\sup_{\substack{u\in\bS_{p_1-1}\\v\in\bS_{p_2-1}}} u^\top M_A v\\
& \geq \frac{1}{2}\sup_{\substack{u\in\bS_{p_1-1}\\v\in\bS_{p_2-1}}} u^\top M_A v.
\end{align*}
In particular, this implies that
\begin{align*}
\sup_{f\in\cF} \left| \frac{1}{n} \sum_{i=1}^n f(X_i) - \bE[f(X)]\right| & \leq  2\sup_{A\in\cA} \sup_{(u,v)\in \cN}  u^\top M_A v,
\end{align*}
which concludes the proof of Lemma~\ref{lemma:chainingVC}, thanks to Equation~\eqref{eq:boundnet}.

\paragraph{Case $\cA=\{\R^d\}$.} This particular case is much simpler, as $\cF_{u,v}$ is now a single function. In consequence, instead of Equation~\eqref{eq:pointwisechaining}, a simple Bernstein inequality yields that when $n\geq C \ln(1/\delta)$, with probability at least $1-\delta$
\begin{equation*}
\sup_{f\in\cF_{u,v}} \left| \frac{1}{n} \sum_{i=1}^n f(X_i) - \bE[f(X)]\right| \leq C \sqrt{\frac{\ln(1/\delta)}{n}}.
\end{equation*}
The covering argument then allows to conclude similarly, requiring only a sample complexity $n\geq C \varepsilon^{-2}(p_1+p_2+\ln(1/\delta))$.
\end{proof}
Of note, we could probably improve---or even get rid of---the term $\ln(1/\varepsilon)^2$ in the required sample complexity of Lemma~\ref{lemma:chainingVC}, since Theorem 2.14.25 of \citet{vaart1996weak} actually yields a probability bounded as $\exp\left(-Cn \min\left(\frac{t^2}{\sigma_{\cF_M}^2}, \frac{t}{M}\right)\right)$ for some variance term $\sigma_{\cF_M}^2$. We use the obvious bound $\sigma_{\cF_M}^2\leq M^2$ here, but we believe a tighter bound is possible. However, it is sufficient for our purpose, since we will use Lemma~\ref{lemma:chainingVC} with $\varepsilon$ of constant order.

\begin{lemma}[adapted from \citealp{soltanolkotabi2017learning}, Lemma 5.5]\label{lemma:solta}
Let $w^\star_{0}$ be a fixed vector in $\bS_{d-1}$ and $\varepsilon_1, \varepsilon_2, \delta\in(0,1/2)$. There exists a positive universal constant $C$, such that if $n\geq \frac{C}{\varepsilon_1} \left(d+\ln(1/\delta)\right)$, then with probability at least $1-\delta$,
\begin{align*}
\frac{1}{n}\sum_{i=1}^n \left( x_i^\top P_{w_0^{\star\,\perp}} w\right)^2\iind{\frac{1}{2}|x_i^\top w_0^\star|\leq |x_i^\top  P_{w_0^{\star\,\perp}} w|} \leq C\left( \varepsilon_2 + \varepsilon_1 \right) \|P_{w_0^{\star\,\perp}} w\|^2,
\end{align*}
uniformly over all $w\in\R^d$ such that $\|P_{w_0^{\star\,\perp}}w \| \leq \varepsilon_2$.
\end{lemma}
\begin{proof}
This is actually a direct application\footnote{Note that there appear to be two typos in the statement of Lemma 5.5 in \citet{soltanolkotabi2017learning}: the term $\sqrt{\frac{21}{20}}\varepsilon$ should read $\sqrt{\frac{21}{20}\varepsilon}$, and the sample-size requirement should be $n\geq c\delta^{-2}\omega^2(\cC\cap\bS_{d-1})$. The first typo appears to be inherited from a similar typo in the proof on page 30 of \citet{soltanolkotabi2017structured}. Moreover, the sample complexity required in \citet{soltanolkotabi2017structured} is indeed $n\geq c\delta^{-2}\omega^2(\cC\cap\bS_{d-1})$.} of \citet[][Lemma 5.5]{soltanolkotabi2017learning} with, following their notations,
\begin{itemize}
\item $\mathcal{C} =  \{w_0^\star\}^\perp$;
\item $\varepsilon = \frac{\varepsilon_2}{\frac{1}{2}+\varepsilon_2}$;
\item $\delta = \sqrt{\varepsilon_1}$;
\item $w^* = \left(\frac{1}{2}+\varepsilon_2\right) w_0^\star$.
\end{itemize}
\end{proof}
\end{document}